\documentclass{article}

\usepackage{microtype}
\usepackage{graphicx}
\usepackage{subfigure}
\usepackage{booktabs} 

\usepackage{hyperref}

\usepackage[accepted]{icml2020}

\icmltitlerunning{On the Global Convergence Rates of Softmax Policy Gradient Methods}

\usepackage{amsmath,amsfonts,bm}
\usepackage{amssymb,amsthm}
\usepackage{mathtools}
\usepackage{algorithm,algorithmic}
\usepackage{natbib}
\usepackage{hyperref}
\usepackage{url}
\usepackage{cancel}
\usepackage[capitalize]{cleveref}

\renewcommand{\algorithmicrequire}{\textbf{Input:}}
\renewcommand{\algorithmicensure}{\textbf{Output:}}

\crefname{proposition}{Proposition}{Propositions}
\crefname{theorem}{Theorem}{Theorems}
\crefname{lemma}{Lemma}{Lemmas}
\crefname{update_rule}{Update}{Updates}
\crefname{algorithm}{Algorithm}{Algorithms}

\newcommand\numberthis{\addtocounter{equation}{1}\tag{\theequation}}
\makeatletter
\newcommand{\pushright}[1]{\ifmeasuring@#1\else\omit\hfill$\displaystyle#1$\fi\ignorespaces}
\newcommand{\pushleft}[1]{\ifmeasuring@#1\else\omit$\displaystyle#1$\hfill\fi\ignorespaces}
\makeatother

\def\eqref#1{equation~\ref{#1}}

\def\1{\bm{1}}

\DeclareMathAlphabet{\mathsfit}{\encodingdefault}{\sfdefault}{m}{sl}
\SetMathAlphabet{\mathsfit}{bold}{\encodingdefault}{\sfdefault}{bx}{n}

\def\gA{{\mathcal{A}}}

\def\gM{{\mathcal{M}}}
\def\gN{{\mathcal{N}}}

\def\gP{{\mathcal{P}}}

\def\gR{{\mathcal{R}}}
\def\gS{{\mathcal{S}}}

\def\gX{{\mathcal{X}}}

\def\sH{{\mathbb{H}}}

\def\sR{{\mathbb{R}}}

\newcommand{\R}{\mathbb{R}}

\newcommand{\softmax}{\mathrm{softmax}}

\newcommand{\KL}{D_{\mathrm{KL}}}
\newcommand{\Var}{\mathrm{Var}}

\DeclareMathOperator*{\argmax}{arg\,max}
\DeclareMathOperator*{\argmin}{arg\,min}

\newtheorem{theorem}{Theorem}
\newtheorem{lemma}{Lemma}
\newtheorem{claim}{Claim}
\newtheorem{definition}{Definition}
\newtheorem{proposition}{Proposition}
\newtheorem{remark}{Remark}

\newtheorem{assumption}{Assumption}

\newtheorem{update_rule}{Update}

\DeclareMathOperator*{\expectation}{\mathbb{E}}

\def\rvone{{\mathbf{1}}}
\def\rvzero{{\mathbf{0}}}

\def\identitymatrix{\mathbf{Id}}
\def\diagonalmatrix{\text{diag}}

\DeclareMathOperator*{\probability}{Pr}

\newcommand{\cS}{\mathcal{S}}

\newcommand{\norm}[1]{\|#1\|}

\begin{document}

\twocolumn[
\icmltitle{On the Global Convergence Rates of Softmax Policy Gradient Methods}



\icmlsetsymbol{equal}{*}
\icmlsetsymbol{uofa}{$\clubsuit$}
\icmlsetsymbol{dm_alberta}{$\heartsuit$}
\icmlsetsymbol{google_brain}{$\spadesuit$}

\begin{icmlauthorlist}
\icmlauthor{Jincheng Mei}{uofa,google_brain,equal}
\icmlauthor{Chenjun Xiao}{uofa}
\icmlauthor{Csaba Szepesv{\'a}ri}{dm_alberta,uofa}
\icmlauthor{Dale Schuurmans}{google_brain,uofa}
~\\~\\
\icmlauthor{$^\clubsuit$\normalfont{University of Alberta}}{}
\icmlauthor{$^\heartsuit$\normalfont{DeepMind}}{}
\icmlauthor{$^\spadesuit$\normalfont{Google Research, Brain Team}}{}
\end{icmlauthorlist}

\icmlcorrespondingauthor{Jincheng Mei}{jmei2@ualberta.ca}

\icmlkeywords{Machine Learning, ICML}

\vskip 0.3in
]



\printAffiliationsAndNotice{Work done as an intern at Google Research, Brain Team. This version (v3) generalizes \cref{lem:lojasiewicz_softmax_general} to multiple optimal actions, i.e., \cref{eq:lojasiewicz_softmax_general_result_3}. Jincheng Mei would like to thank Ziheng Wang at University of Oxford, Mathematical Institute for asking related questions} 

\begin{abstract}
We make three contributions toward better understanding policy gradient methods in the tabular setting.
First, we show that with the true gradient, policy gradient with a softmax parametrization converges at a $O(1/t)$ rate, with constants depending on the problem and initialization.
This result significantly expands the recent asymptotic convergence results. 
The analysis relies on two findings:
that the softmax policy gradient satisfies a \L{}ojasiewicz inequality, and the minimum probability of an optimal action during optimization can be bounded in terms of its initial value.
Second, we analyze entropy regularized policy gradient and show that it enjoys a significantly faster linear convergence rate $O(e^{-c \cdot t})$ toward softmax optimal policy $(c > 0)$.
This result resolves an open question in the recent literature. 
Finally, combining the above two results and additional new $\Omega(1/t)$ lower bound results, we explain how entropy regularization improves policy optimization, even with the true gradient, from the perspective of convergence rate. The separation of rates is further explained using the notion of  non-uniform \L{}ojasiewicz degree. 
These results provide a theoretical understanding of the impact of entropy and corroborate existing empirical studies.
\end{abstract}

\section{Introduction}
\label{sec:introduction}

The \emph{policy gradient} is one of the most foundational concepts in
Reinforcement Learning (RL),
lying at the core of policy-search and actor-critic methods. This paper is concerned with the analysis of the convergence rate of 
\emph{policy gradient methods} \citep{sutton2000policy}.
As an approach to RL, the appeal of policy gradient methods is that they are conceptually straightforward 
and under some regularity conditions they guarantee monotonic improvement of the value.
A secondary appeal is that policy gradient methods were shown to achieve
effective empirical performance \citep[e.g.,][]{schulman2015trust,schulman2017proximal}.

Despite the prevalence and importance of policy optimization in RL,
the theoretical understanding of policy gradient method has,
until recently, been severely limited.
A key barrier to understanding is the inherent non-convexity
of the value landscape with respect to standard
policy parametrizations.
As a result,
little has been known about the global convergence behavior of policy gradient
method.
Recently, important new progress in understanding the convergence behavior
of policy gradient has been achieved.
As in this paper we will restrict ourselves to the tabular setting, 
we analyze the part of the literature that also deals with this setting.
While the tabular setting is clearly limiting, this is the setting where so far the cleanest
results have been achieved and understanding this setting is a necessary first step towards
the bigger problem of understanding RL algorithms.
Returning to the discussion of recent work,
\citet{BhRu19} showed that,
without parametrization, projected gradient ascent on the simplex
does not suffer from spurious local optima.
In concurrent work, 
\citet{AgKaLeMa19} showed that
{\em (i)} without parametrization, projected gradient ascent converges at rate
$O(1/\sqrt{t})$ to a global optimum;
and
{\em (ii)} with softmax parametrization, policy gradient converges asymptotically.
\citeauthor{AgKaLeMa19} also analyze other variants of policy gradient,
and show that
policy gradient with relative entropy regularization
converges at rate $O(1/\sqrt{t})$,
natural policy gradient (mirror descent) converges at rate $O(1/ t)$,
and given a ``compatible'' function approximation (thus, going beyond the tabular case)
natural policy gradient converges at rate $O(1/\sqrt{t})$.
\citet{ShEfMa20} obtains the slower rate $O(1/\sqrt{t})$ for mirror descent.
They also proposed a variant that adds entropy regularization and prove a
rate of $O(1/t)$ for this modified problem.

Despite these advances,
many open questions remain
in understanding the behavior of policy gradient methods,
even in the tabular setting and even when the true gradient is available in the updates.
In this paper, we provide answers to the following three questions
left open by previous work in this area:
{\em (i)} What is the convergence rate of policy gradient methods with softmax
parametrization?
The best previous result, due to \citet{AgKaLeMa19},
established asymptotic convergence but gave no rates. 
{\em (ii)} What is the convergence rate of entropy regularized softmax policy gradient?
Figuring out the answer to this question was 
explicitly stated as an open problem by \citet{AgKaLeMa19}. 
{\em (iii)} Empirical results suggest that entropy helps optimization
\citep{ahmed2019understanding}.
Can this empirical observation be turned into a rigorous theoretical result?%
\footnote{While \citet{ShEfMa20} suggest that entropy regularization speeds up mirror descent to achieve 
the rate of $O(1/t)$,
in light of the corresponding result of
\citet{AgKaLeMa19} who established the same rate for the unregularized version of mirror descent,
their conclusion needs further support (e.g., lower bounds).
}

\textit{First}, we prove that with the true gradient,
policy gradient methods with a softmax parametrization converge to the
optimal policy at a $O(1/t)$ rate,
with constants depending on the problem and initialization.
This result significantly strengthens the recent asymptotic convergence
results of \citet{AgKaLeMa19}.
Our analysis relies on two novel findings:
{\em (i)} that softmax policy gradient satisfies what we call a non-uniform
\L{}ojasiewicz-type inequality with the constant in the inequality depending on the optimal action
probability under the current policy;
{\em (ii)} the minimum probability of an optimal action during optimization
can be bounded in terms of its initial value.
Combining these two findings, with a few other properties we describe,
it can be shown that softmax policy gradient method achieves 
a $O(1/t)$ convergence rate.

\textit{Second}, we analyze entropy regularized policy gradient
and show that it enjoys a
linear convergence rate of $O(e^{-t})$ toward the softmax optimal policy, which is significantly faster than that of the unregularized version.
This result resolves an open question in \citet{AgKaLeMa19},
where the authors analyzed a more aggressive relative entropy regularization
rather than the more common entropy regularization.
A novel insight is that entropy regularized gradient updates behave
similarly to the contraction operator in value learning,
with a contraction factor that depends on the current policy.

\textit{Third}, we provide a theoretical understanding of entropy
regularization in policy gradient methods.
{\em (i)} We prove a new lower bound of $\Omega(1/t)$ for softmax policy gradient,
implying that 
the upper bound of $O(1/t)$ that we established, apart from constant factors,
is unimprovable.
This result also provides a theoretical explanation of the optimization
advantage of entropy regularization:
even with access to the true gradient,
entropy helps policy gradient
\emph{converge faster than any achievable rate of softmax policy gradient method
without regularization}.
{\em (ii)} We study the concept of non-uniform \L{}ojasiewicz degree
and show that, without regularization, the \L{}ojasiewicz degree of
expected reward cannot be positive,
which allows $O(1/t)$ rates to be established.
We then show that with entropy regularization,
the \L{}ojasiewicz degree of maximum entropy reward becomes $1/2$,
which is sufficient to obtain linear $O(e^{-t})$ rates.
This change of the
relationship between gradient norm and sub-optimality
reveals a deeper reason for the improvement in convergence rates.
The theoretical study we provide corroborates existing empirical studies
on the impact of entropy in policy optimization \citep{ahmed2019understanding}.

The remainder of the paper is organized as follows.
After introducing notation and defining the setting in
\cref{sec:notations_settings},
we present the three main contributions in
\cref{sec:policy_gradient,sec:entropy_policy_gradient,sec:theoretical_understanding_entropy} as aforementioned.
\cref{sec:conclusions_future_work} gives our conclusions.

\section{Notations and Settings}
\label{sec:notations_settings}

For a finite set $\gX$, we use $\Delta(\gX)$ to denote the set of probability distributions over $\gX$. A finite Markov decision process (MDP) $\gM = (\gS, \gA, \gP, r, \gamma)$ is determined by a finite state space $\gS$, a finite action space $\gA$, transition function $\gP : \gS \times \gA \to \Delta(\gS)$, reward function $r : \gS \times \gA \to \sR$, and discount factor $\gamma \in [0, 1)$. Given a policy $\pi : \gS \to \Delta(\gA)$, the value of state $s$ under $\pi$ is defined as
\begin{align}
\label{eq:state_value_function}
    V^\pi(s) \coloneqq \expectation_{\substack{s_0 = s, a_t \sim \pi(\cdot | s_t), \\ s_{t+1} \sim \gP( \cdot | s_t, a_t)}}{\left[ \sum_{t=0}^{\infty}{\gamma^t r(s_t, a_t)} \right]}.
\end{align}
We also let $V^\pi(\rho) \coloneqq \expectation_{s \sim \rho}{ \left[ V^\pi(s) \right]}$, where $\rho \in \Delta(\gS)$ is an initial state distribution. The state-action value of $\pi$ at $(s,a) \in \gS \times \gA$ is defined as
\begin{align}
\label{eq:state_action_value_function}
    Q^\pi(s, a) \coloneqq r(s, a) + \gamma \sum_{s^\prime}{ \gP( s^\prime | s, a) V^\pi(s^\prime) }.
\end{align}
We let $A^\pi(s,a) \coloneqq Q^\pi(s, a) - V^\pi(s)$ be the so-called advantage function of $\pi$. The (discounted) state distribution of $\pi$ is defined as
\begin{align}
     d_{s_0}^{\pi}(s) \coloneqq (1 - \gamma) \sum_{t=0}^{\infty}{ \gamma^t \probability(s_t = s | s_0, \pi, \gP) },
\end{align}
and we let $d_{\rho}^{\pi}(s) \coloneqq \expectation_{s_0 \sim \rho}{\left[ d_{s_0}^{\pi}(s) \right]}$. Given $\rho$, there exists an optimal policy $\pi^*$ such that
\begin{align}
    V^{\pi^*}(\rho) = \max_{\pi : \gS \to \Delta(\gA)}{ V^{\pi}(\rho) }. 
\end{align}
We denote $V^*(\rho) \coloneqq V^{\pi^*}(\rho)$ for conciseness. Since $\gS \times \gA$ is finite, for convenience, 
without loss of generality,
we assume that the one step reward lies in the $[0,1]$ interval: 
\begin{assumption}[Bounded reward]
\label{asmp:bounded_reward}
$r(s,a) \in [0,1]$, $\forall (s,a)$.
\end{assumption}
The softmax transform of a vector exponentiates the components of the vector and normalizes it so that the result lies in the simplex. This can be used to transform vectors assigned to state-action pairs into policies:
\paragraph{Softmax transform.} Given the function $\theta : \gS \times \gA \to \sR$, the softmax transform of $\theta$ is defined as 
$\pi_\theta(\cdot | s) \coloneqq \softmax(\theta(s, \cdot))$, where for all $a \in \gA$,
\begin{align}
\label{eq:softmax_transform}
	\pi_\theta(a | s) = \frac{\exp\{\theta(s,a)\}}{\sum_{a^\prime}{\exp\{\theta(s,a^\prime)\}}}.
\end{align}
Due to its origin in logistic regression, we call the values $\theta(s,a)$ the \emph{logit} values and the function $\theta$ itself a logit function.
We also extend this notation to the case when there are no states:
For $\theta : [K] \to \sR$, we define $\pi_\theta \coloneqq \softmax(\theta)$ using $\pi_\theta(a) = \exp\{\theta(a)\} / \sum_{a^\prime}{ \exp\{\theta(a^\prime)\} }$ ($a\in [K]$).

\paragraph{H matrix.} 
Given any distribution $\pi$ over $[K]$, let $H(\pi) \coloneqq \diagonalmatrix(\pi) - \pi \pi^\top \in \sR^{K \times K}$, where $\diagonalmatrix(x) \in \sR^{K \times K}$ is the diagonal matrix that has $x\in \sR^K$ at its diagonal.
The $H$ matrix will play a central role in our analysis because 
$H(\pi_\theta)$ is the Jacobian of the $\theta \mapsto \pi_\theta:=\softmax(\theta)$ map that maps $\R^{[K]}$ to the $(K-1)$-simplex:
\begin{align}
\label{eq:H_matrix}
    \left( \frac{d \pi_\theta}{d \theta} \right)^\top = H(\pi_\theta).
\end{align}
Here, we are using the standard convention that derivatives give row-vectors.
Finally, we recall the definition of smoothness from convex analysis:
\paragraph{Smoothness.} A function $f: \Theta \to \sR$ 
with $\Theta\subset \sR^d$
is $\beta$-smooth (w.r.t. $\ell_2$ norm, $\beta >0$) if for all $\theta$, $\theta^\prime \in \Theta$,
\begin{align}
    \left| f(\theta^\prime) - f(\theta) - \Big\langle \frac{d f(\theta)}{d \theta}, \theta^\prime - \theta \Big\rangle \right| \le \frac{\beta}{2} \cdot \| \theta^\prime - \theta \|_2^2.
\end{align}

\section{Policy Gradient}
\label{sec:policy_gradient}

Policy gradient is a special policy search method. In policy search, one considers a family of policies parametrized by finite-dimensional parameter vectors, reducing the search for a good policy to searching in the space of parameters.
This search is usually accomplished by making incremental changes (additive updates) to the parameters.
Representative policy-based RL methods include REINFORCE \citep{williams1992simple}, natural policy gradient \citep{kakade2002natural}, deterministic policy gradient \citep{silver2014deterministic}, and trust region policy optimization \citep{schulman2015trust}. 
In policy gradient methods, the parameters are updated by following the gradient of the map that maps policy parameters to values. Under mild conditions, the gradient can be reexpressed in a convenient form in terms of the policy's action-value function and the gradients of the policy parametrization:
\begin{theorem}[Policy gradient theorem \citep{sutton2000policy}]
\label{thm:policy_gradient_theorem_general}
Fix a map
 $\theta \mapsto \pi_\theta(a | s)$ that for any $(s,a)$ is differentiable and fix
 an initial distribution $\mu \in \Delta(\gS)$.
Then,
\begin{align*}
    \frac{\partial V^{\pi_\theta}(\mu)}{\partial \theta} = \frac{1}{1-\gamma} \expectation_{s \sim d_{\mu}^{\pi_\theta} } { \left[ \sum_{a}  \frac{\partial \pi_\theta(a | s)}{\partial \theta} \cdot Q^{\pi_\theta}(s,a) \right] }\,.
\end{align*}
\end{theorem}

\subsection{Vanilla Softmax Policy Gradient}
We focus on the policy gradient method that uses the softmax parametrization. 
Since we consider the tabular case, the policy is then parametrized using the logit $\theta : \gS \times \gA \to \sR$ function and  $\pi_\theta(\cdot | s) = \softmax(\theta(s, \cdot))$.
The vanilla form of policy gradient for this case is shown in \cref{alg:policy_gradient_softmax}. 
\begin{algorithm}[ht]
   \caption{Policy Gradient Method}
   \label{alg:policy_gradient_softmax}
\begin{algorithmic}
   \STATE {\bfseries Input:} Learning rate $\eta > 0$.
   \STATE Initialize logit $\theta_1(s,a)$ for all $(s,a)$.
   \FOR{$t=1$ {\bfseries to} $T$}
   \STATE $\theta_{t+1} \gets \theta_{t} + \eta \cdot \frac{\partial V^{\pi_{\theta_t}}(\mu)}{\partial \theta_t}$.
   \ENDFOR
\end{algorithmic}
\end{algorithm} 

With some calculation, 
\cref{thm:policy_gradient_theorem_general} can be used to show that the gradient takes the following special form in this case:
\begin{lemma}
\label{lem:policy_gradient_softmax}
Softmax policy gradient w.r.t. $\theta$ is
\begin{align}
    \frac{\partial V^{\pi_\theta}(\mu)}{\partial \theta(s,a)} = \frac{1}{1-\gamma} \cdot d_{\mu}^{\pi_\theta}(s) \cdot \pi_\theta(a|s) \cdot A^{\pi_\theta}(s,a).
\end{align}
\end{lemma}
Due to space constraints, the proof of this, as well as of all the remaining results are given in the appendix.
While this lemma was known \citep{AgKaLeMa19}, we included a proof for the sake of completeness.


Recently, \citet{AgKaLeMa19} showed that softmax policy gradient asymptotically converges to $\pi^*$, i.e., $V^{\pi_{\theta_t}}(\rho) \to V^*(\rho)$ as $t \to \infty$ provided that $\mu(s) > 0$ holds for all states $s\in \gS$. We strengthen this result to show that the rate of convergence (in terms of value sub-optimality) is $O(1/t)$. The next section is devoted to this result. For better accessibility, we start with the result for the bandit case which presents an opportunity to explaining the main ideas underlying our result in a clean fashion.

\subsection{Convergence Rates}

\subsubsection{The Instructive Case of Bandits}
\label{sec:policy_gradient_one_state_case}
As promised, in this section we consider  ``bandit case'': In particular, assume that the MDP has a single state and the discount factor $\gamma$ is zero: $\gamma=0$.
In this case, \cref{eq:state_value_function} reduces to maximizing the expected reward,
\begin{align}
\label{eq:expected_reward_objective}
    \max_{\theta : \gA \to \sR}{ \expectation_{a \sim \pi_{\theta}}{ [ r(a) ]  } }.
\end{align}
With $\pi_\theta = \softmax(\theta)$, even in this simple setting, the objective is non-concave in $\theta$, as shown by a simple example:
\begin{proposition}
\label{prop:non_concave_softmax_expected_reward}
On some problems, $\theta\mapsto \expectation_{a \sim \pi_{\theta}}{ [ r(a) ] }$ is a non-concave function over $\R^K$.
\end{proposition}
As $\gamma=0$ and there is a single state,
\cref{lem:policy_gradient_softmax} simplifies to
\begin{align}
    \frac{d \pi_{\theta}^\top r}{d \theta(a)} = \pi_{\theta}(a) \cdot (r(a) - \pi_{\theta}^\top r )\,.
\end{align}
Putting things together, we see that in this case the update
in \cref{alg:policy_gradient_softmax} takes the following form:
\begin{update_rule}[Softmax policy gradient, expected reward]
\label{update_rule:softmax_special}
$\theta_{t+1}(a) \gets \theta_t(a) + \eta \cdot \pi_{\theta_t}(a) \cdot (r(a) - \pi_{\theta_t}^\top r ),$ $\forall a \in [K]$.
\end{update_rule}
As is well known, if a function is smooth, then a small gradient update will be guaranteed to improve the objective value. 
As it turns out, for the softmax parametrization, the expected reward objective  is $\beta$-smooth with $\beta \le 5/2$:
\begin{lemma}[Smoothness]
\label{lem:smoothness_softmax_special}
$\forall r \in \left[ 0, 1\right]^K$, $\theta \mapsto \pi_\theta^\top r$ is $5/2$-smooth.
\end{lemma}
Smoothness alone (as is also well known) is not sufficient to guarantee that gradient updates converge to a global optimum. For non-concave objectives, the next best thing to guarantee convergence to global maxima is to establish that the gradient of the objective at any parameter dominates the sub-optimality of the parameter.
Inequalities of this form are known as a \L{}ojasiewicz inequality \citep{lojasiewicz1963propriete}. 
The reason gradient dominance helps is because it prevents the gradient vanishing before reaching a maximum.
The objective function of our problem also satisfies such an inequality, although of a weaker, ``non-uniform'' form. For the following result, for simplicity, we assume that the optimal action is unique. 
This assumption can be lifted with a little extra work, which is discussed at the end of this section.
\begin{lemma}[Non-uniform \L{}ojasiewicz]
\label{lem:lojasiewicz_softmax_special}
Assume $r$ has one unique maximizing action $a^*$. Let $\pi^* = \argmax_{\pi \in \Delta}{ \pi^\top r}$. Then, 
\begin{align}
    \left\| \frac{d \pi_\theta^\top r}{d \theta} \right\|_2 \ge \pi_\theta(a^*) \cdot ( \pi^* - \pi_\theta )^\top r\,.
\end{align}
\end{lemma}
The weakness of this inequality is that the right-hand side scales with $\pi_\theta(a^*)$ 
-- hence we call it non-uniform.
As a result, \cref{lem:lojasiewicz_softmax_special} is not very useful if $\pi_{\theta_t}(a^*)$, the optimal action's probability, becomes very small during the updates. 

Nevertheless, the inequality still suffices to get an following  intermediate result.
The proof of this result combines smoothness and the \L{}ojasiewicz inequality we derived. 
\begin{lemma}[Pseudo-rate]
\label{lem:pseudo_rates_softmax_special}
Let $c_t = \min_{1 \le s \le t}{ \pi_{\theta_s}(a^*) }$.
Using \cref{update_rule:softmax_special} with $\eta = 2/5$, for all $t \ge 1$,
\begin{align*}
    ( \pi^* - \pi_{\theta_t} )^\top r &\le 5 / ( t \cdot c_t^2), \qquad \text{and} \\
    \sum_{t=1}^{T}{ ( \pi^* - \pi_{\theta_t} )^\top r} &\le \min\left\{ \sqrt{5 T} / c_T , \ ( 5 \log{T} ) / c_T^2 + 1\right\}.
\end{align*}
\end{lemma}
In the remainder of this section we assume that $\eta=2/5$.
\begin{remark}
\label{rmk:necessary_dependence_optimal_action_prob}
The value of $\pi_{\theta_t}(a^*)$, while it is nonzero (and so is $c_t$) can be small (e.g., because of the choice of $\theta_1$). 
Consequently, its minimum $c_t$ can be quite small and the upper bound in \cref{lem:pseudo_rates_softmax_special} can be large, or even vacuous. 
The dependence of the previous result on $\pi_{\theta_t}(a^*)$ comes from \cref{lem:lojasiewicz_softmax_special}. 
As it turns out, it is not possible to eliminate or improve the dependence on $\pi_\theta(a^*)$ in
\cref{lem:lojasiewicz_softmax_special}.
To see this consider $r = (5, 4, 4)^\top$, $\pi_\theta = (2 \epsilon, 1/2 - 2 \epsilon, 1/2)$ where $\epsilon > 0$ is small number. By algebra, $( \pi^* - \pi_\theta )^\top r = 1 - 2 \epsilon > 1/2$, $\frac{d \pi_\theta^\top r}{d \theta} = (2 \epsilon - 4 \epsilon^2, - \epsilon + 4 \epsilon^2, - \epsilon )^\top$, $\left\| \frac{d \pi_\theta^\top r}{d \theta} \right\|_2 = \epsilon \cdot \sqrt{6 - 24 \epsilon + 32 \epsilon^2} \le 3 \epsilon$. Hence, for any constant $C > 0$,
\begin{align}
    C \cdot ( \pi^* - \pi_\theta )^\top r > C/2 > 3 \epsilon \ge \left\| \frac{d \pi_\theta^\top r}{d \theta} \right\|_2,
\end{align}
which means for any \L{}ojasiewicz-type inequality, $C$ necessarily depends on $\epsilon$ and hence on $\pi_\theta(a^*)= 2\epsilon$.
\end{remark}
The necessary dependence on $\pi_{\theta_t}(a^*)$ makes it clear 
that  \cref{lem:pseudo_rates_softmax_special} is insufficient to conclude a $O(1/t)$ rate.
since $c_t$ may vanish faster than $O(1/t)$ as $t$ increases. 
Our next result eliminates this possibility.
In particular, the result follows from the asymptotic convergence result of 
\citet{AgKaLeMa19} which states that $\pi_{\theta_t}(a^*) \to 1$ as $t \to \infty$. 
From this and because $\pi_{\theta}(a)>0$ for any $\theta\in \R^K$ and action $a$, we conclude that $\pi_{\theta_t}(a^*)$ remains bounded away from zero during the course of the updates:
\begin{lemma}
\label{lem:optimal_action_prob_positive}
\label{lem:lower_bound_cT_softmax_special}
We have $\inf_{t\ge 1} \pi_{\theta_t}(a^*) > 0$. 
\end{lemma}

\begin{figure*}[ht]
\centering
\includegraphics[width=1\linewidth]{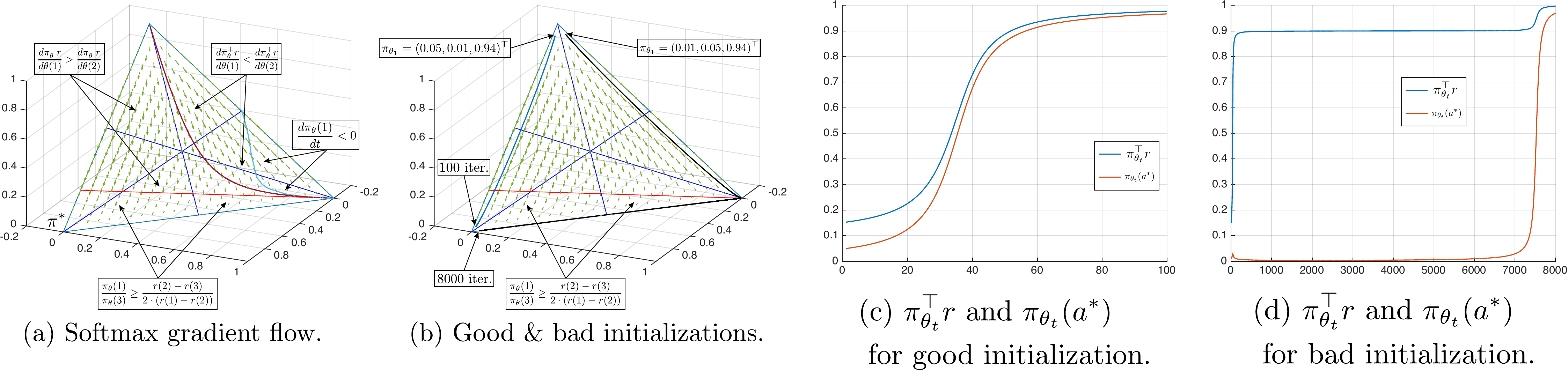}
\vskip -0.1in
\caption{Visualization of proof idea for \cref{lem:lower_bound_cT_softmax_special}.
} 
\label{fig:fig1}
\vskip -0.2in
\end{figure*}

With some extra work, one can also show that eventually $\theta_t$ enters a region where $\pi_{\theta_t}(a^*)$ can only increase:
\begin{proposition}
\label{prop:t_zero_softmax_special}
For any initialization there exist $t_0 \ge 1$ such that for any $t\ge t_0$, $t \mapsto \pi_{\theta_t}(a^*)$ is increasing. In particular, when $\pi_{\theta_1}$ is the uniform distribution, $t_0=1$.
\end{proposition}
With \cref{lem:pseudo_rates_softmax_special,lem:lower_bound_cT_softmax_special}, we can now obtain an $O(1/t)$ convergence rate for softmax policy gradient method\footnote{For a continuous version of \cref{update_rule:softmax_special}, \citet{walton2020short} proves a $O(1/t)$ rate, using a Lyapunov function argument.}:
\begin{theorem}[Arbitrary initialization]
\label{thm:final_rates_softmax_special}
Using \cref{update_rule:softmax_special} with $\eta = 2/5$, for all $t \ge 1$,
\begin{align}
    ( \pi^* - \pi_{\theta_t} )^\top r \le 5 / (c^2 \cdot t),
\end{align}
where $c = \inf_{t\ge 1} \pi_{\theta_t}(a^*) > 0$ is a constant that depends on $r$ and $\theta_1$, but it does not depend on the time $t$.
\end{theorem}
\cref{prop:t_zero_softmax_special}
suggests that one should set $\theta_1$ so that $\pi_{\theta_1}$ is uniform.
Using this initialization, we can show that $\inf_{t\ge 1} \pi_{\theta_t}(a^*) \ge 1/K$, strengthening  \cref{thm:final_rates_softmax_special}:
\begin{theorem}[Uniform initialization]
\label{thm:rates_uniform_softmax_special}
Using \cref{update_rule:softmax_special} with $\eta = 2/5$ and $\theta_1$ such that $\pi_{\theta_1}(a) = 1/K$, $\forall a$, for all $t \ge 1$,
\begin{align*}
    ( \pi^* - \pi_{\theta_t} )^\top r &\le 5 K^2 / t, \qquad \text{and} \\
    \sum_{t=1}^{T}{ ( \pi^* - \pi_{\theta_t} )^\top r} &\le \min\left\{ K \sqrt{5 T}, \ 5 K^2 \log{T} + 1\right\}.
\end{align*}
\end{theorem}
\begin{remark}
In \cref{sec:theoretical_understanding_entropy}, we prove a lower bound $\Omega(1/t)$ for the same update rule, showing that the upper bound $O(1/t)$ of \cref{thm:final_rates_softmax_special}, apart from constant factors, is unimprovable.
\end{remark}
In general it is difficult to characterize how the constant $C$ in \cref{thm:final_rates_softmax_special} depends on the problem and initialization. For the simple $3$-armed case, this dependence is relatively clear:
\begin{lemma}
\label{lem:t0_softmax_special}
Let $r(1) > r(2) > r(3)$. Then, $a^*=1$ and $\inf_{t\ge 1} \pi_{\theta_t}(a^*) = \min_{1 \le t \le t_0} \pi_{\theta_t}(1)$, where 
\begin{align}
    t_0 = \min \left\{ t \ge 1 \,:\, \frac{ \pi_{\theta_t}(1) }{ \pi_{\theta_t}(3) } \ge \frac{r(2) - r(3)}{2 \cdot ( r(1) - r(2) ) } \right\}.
\end{align}
\end{lemma}
Note that the smaller $r(1) - r(2)$ and $\pi_{\theta_1}(1)$ are, the larger $t_0$ is, which potentially means $c$ in \cref{thm:final_rates_softmax_special} can be smaller.

\paragraph{Visualization.} Let $r = (1.0, 0.9, 0.1)^\top$. In \cref{fig:fig1}(a), 
 the region below the red line corresponds to
 $\gR=\left\{ \pi_\theta :  \pi_{\theta}(1) / \pi_{\theta}(3)  \ge (r(2) - r(3))/(2 \cdot (r(1) - r(2))) \right\}$.
Any globally convergent iteration will enter $\gR$ within finite time (the closure of $\gR$ contains $\pi^*$) and never leaves $\gR$ (this is the main idea in  \cref{lem:lower_bound_cT_softmax_special}). Subfigure (b) shows the behavior of the gradient updates with ``good'' ($\pi_{\theta_1} = (0.05, 0.01, 0.94)^\top$) and ``bad'' ($\pi_{\theta_1} = (0.01, 0.05, 0.94)^\top$) initial policies. While these are close to each other, the iterates behave quite differently (in both cases $\eta = 2/5$). From the good initialization, the iterates converge quickly: after $100$ iterations the distance to the optimal policy is already quite small. At the same time, starting from a ``bad'' initial value, the iterates are first  attracted toward a sub-optimal action. It takes more than $7000$ iterations for the algorithm to escape this sub-optimal corner! 
In subfigure (c), we see that $\pi_{\theta_t}(a^*)$ increases for the good initialization, while in subfigure (d), for the bad initialization, we see that it initially decreases. 
These experiments confirm that the dependence of the error bound in
\cref{thm:final_rates_softmax_special} on the initial values cannot be removed.

\paragraph{Non-unique optimal actions.} When the optimal action is non-unique, the arguments need to be slightly modified. Instead of using a single $\pi_{\theta}(a^*)$, we need to consider $\sum_{a^* \in \gA^*}{ \pi_\theta(a^*)}$, i.e., the sum of probabilities of all optimal actions. Details are given in the appendix.

\subsubsection{General MDPs}
\label{sec:policy_gradient_general_mdps}

For general MDPs, the optimization problem takes the form
\begin{align*}
\label{eq:value_function_objective}
    \max_{\theta : \gS \times \gA \to \sR}{V^{\pi_\theta}}(\rho) = \max_{\theta : \gS \times \gA \to \sR}{ \expectation_{s \sim \rho}{\sum_{a}{ \pi_\theta(a | s) \cdot Q^{\pi_\theta}(s, a) } } } .
\end{align*}
Here, as before, $\pi_\theta(\cdot|s) = \softmax(\theta(s,\cdot))$, $s\in \gS$.
Following \citet{AgKaLeMa19}, the values here are defined with respect to an initial state distribution $\rho$ which may not be the same as the initial state distribution $\mu$ used in the gradient updates (cf. \cref{alg:policy_gradient_softmax}), allowing for greater flexibility in our analysis.
While the initial state distributions do not play any role in the bandit case, here, in the multi-state case, they have a strong influence.
In particular, for the rest of this section, we will assume that the initial state distribution $\mu$ used in the gradient updates is bounded away from zero:
\begin{assumption}[Sufficient exploration]\label{ass:posinit} 
The initial state distribution satisfies $\min_s \mu(s)>0$.
\end{assumption}
\cref{ass:posinit} was also adapted by \citet{AgKaLeMa19}, which ensures ``sufficient exploration'' in the sense that the 
occupancy measure $d^\pi_\mu$ of any policy $\pi$ when started from $\mu$ will be guaranteed to be positive over the whole state space.
\citet{AgKaLeMa19} asked whether this assumption is necessary for convergence to global optimality. 
\begin{proposition}
There exists an MDP and $\mu$ with $\min_s \mu(s)=0$ such that 
there exists $\theta^*:\gS\times \gA \to [0,\infty]$ 
such that $\theta^*$ is the stationary point of
$\theta \mapsto V^{\pi_\theta}(\mu)$ while $\pi_{\theta^*}$ is not an optimal policy.
Furthermore, this stationary point is an attractor, hence, starting gradient ascent in a small enough vicinity of $\theta^*$ will make it converge to $\theta^*$.
\end{proposition}
The MDP of this proposition is $S$ bandit problems: Each state $s\in \gS$ under each action deterministically gives itself as the next state. The reward is selected so that in each $s$ there is a unique optimal action. If $\mu$ leaves out state $s$ (i.e., $\mu(s)=0$), clearly, the gradient of $\theta\mapsto V^{\pi_\theta}(\mu)$ w.r.t.  $\theta$ is zero regardless of the choice of $\theta$. Hence, any $\theta$ such that $\theta(s,a)=+\infty$ for $a$ optimal in state $s$ with $\mu(s)>0$ and $\theta(s,a)$ finite otherwise will satisfy the properties of the proposition. It remains open whether the sufficient exploration condition is necessary for unichain MDPs.

According to \cref{asmp:bounded_reward}, $r(s,a) \in [0,1]$, $Q(s, a) \in [ 0,  1/(1 - \gamma)] $, and hence the objective function is still smooth, as was also shown by \citet{AgKaLeMa19}:
\begin{lemma}[Smoothness]
\label{lem:smoothness_softmax_general}
$V^{\pi_\theta}(\rho)$ is $8 / (1 - \gamma)^3$-smooth.
\end{lemma}
As mentioned in \cref{sec:policy_gradient_one_state_case}, smoothness and (uniform) \L{}ojasiewicz inequality are sufficient to prove a convergence rate. 
As noted by \citet{AgKaLeMa19}, the main difficulty is to establish a (uniform) \L{}ojasiewicz inequality for softmax parametrization. 
As it turns out, the results from the bandit case carry over to multi-state MDPs.

For stating this and the remaining results, we fix 
a \emph{deterministic} optimal policy $\pi^*$ 
and denote by $a^*(s)$ the action that $\pi^*$ selects in state $s$.
With this, the promised result on the non-uniform \L{}ojasiewicz inequality is as follows:
\begin{lemma}[Non-uniform \L{}ojasiewicz] We have,
\label{lem:lojasiewicz_softmax_general}
\begin{align*}
    \left\| \frac{\partial V^{\pi_\theta}(\mu)}{\partial \theta }\right\|_2 \ge \frac{ \min_s{ \pi_\theta(a^*(s)|s) } }{ \sqrt{S} \cdot  \left\| d_{\rho}^{\pi^*} / d_{\mu}^{\pi_\theta} \right\|_\infty } \cdot \left[ V^*(\rho) - V^{\pi_\theta}(\rho) \right]\,.
\end{align*}
\end{lemma}
By \cref{ass:posinit}, $d_{\mu}^{\pi_\theta}$ is also bounded away from zero on the whole state space and thus the multiplier of the sub-optimality in the above inequality is positive.

Generalizing \cref{lem:lower_bound_cT_softmax_special},
we show that $\min_s{ \pi_{\theta_t}(a^*(s)|s) }$ is uniformly bounded away from zero:
\begin{lemma}
\label{lem:lower_bound_cT_softmax_general}
Let \cref{ass:posinit} hold. Using \cref{alg:policy_gradient_softmax}, we have,
$c:=\inf_{s\in \cS,t\ge 1} \pi_{\theta_t}(a^*(s)|s) > 0$.
\end{lemma}
Using \cref{lem:smoothness_softmax_general,lem:lojasiewicz_softmax_general,lem:lower_bound_cT_softmax_general}, we prove that softmax policy gradient converges to an optimal policy at a $O(1/t)$ rate in MDPs, just like what we have seen in the bandit case:
\begin{theorem}
\label{thm:final_rates_softmax_general}
Let \cref{ass:posinit} hold and
let $\{\theta_t\}_{t\ge 1}$ be generated using \cref{alg:policy_gradient_softmax} with $\eta = (1 - \gamma)^3/8$,
$c$ the positive constant from \cref{lem:lower_bound_cT_softmax_general}.
Then, for all $t\ge 1$,
\begin{align*}
    V^*(\rho) - V^{\pi_{\theta_t}}(\rho) \le \frac{16 S }{c^2(1-\gamma)^6 t} \cdot \bigg\| \frac{d_{\mu}^{\pi^*}}{\mu} \bigg\|_\infty^2 \cdot \bigg\| \frac{1}{\mu} \bigg\|_\infty\,.
\end{align*}
\end{theorem}
As far as we know, 
this is the first convergence-rate result for softmax policy gradient for MDPs.
\begin{remark}
\cref{thm:final_rates_softmax_general} implies that
the iteration complexity of \cref{alg:policy_gradient_softmax} to achieve $O(\epsilon)$ sub-optimality is
$O\Big( \frac{S  }{c^2(1 - \gamma)^6 \epsilon} \cdot \Big\| \frac{d_{\mu}^{\pi^*}}{\mu} \Big\|_\infty^2 \cdot \Big\| \frac{ 1 }{ \mu } \Big\|_\infty \Big)$,
which, as a function of $\epsilon$, is better than the results of \citet{AgKaLeMa19} for 
{\em (i)} projected gradient ascent on the simplex
($O\Big( \frac{S A }{(1 - \gamma)^6 \epsilon^2} \cdot \Big\| \frac{ d_{\rho}^{\pi^*} }{ \mu } \Big\|_\infty^2 \Big)$) or for
{\em (ii)}
softmax policy gradient with relative-entropy regularization ($O\Big( \frac{S^2 A^2 }{(1 - \gamma)^6 \epsilon^2} \cdot \Big\| \frac{ d_{\rho}^{\pi^*} }{ \mu } \Big\|_\infty^2 \Big)$). 
The improved  dependence on $\epsilon$ (or $t$) in our result follows 
from \cref{lem:lojasiewicz_softmax_general,lem:lower_bound_cT_softmax_general} and a different proof technique utilized to prove \cref{thm:final_rates_softmax_general}, while we pay a price because our bound depends on $c$, which adds an extra dependence on the MDP as well as on the initialization of the algorithm.
\end{remark}

\section{Entropy Regularized Policy Gradient}
\label{sec:entropy_policy_gradient}

\citet{AgKaLeMa19} considered relative-entropy regularization in policy gradient to get an $O(1/\sqrt{t})$ convergence rate. As they note, relative-entropy is more ``agressive'' in penalizing small probabilities than the more ``common'' entropy regularizer (cf. Remark 5.5 in their paper)
 and it remains unclear whether this latter regularizer leads to an algorithm with the same rate.
In this section, we answer this positively and in fact prove a much better rate.
In particular, we show that 
entropy regularized policy gradient with the softmax parametrization enjoys a linear rate of $O(e^{-t})$.
In retrospect, perhaps this is unsurprising as entropy regularization bears a strong similarity to 
introducing a strongly convex regularizer in convex optimization, where this change is known 
to significantly improve the rate of convergence of first-order methods \citep[e.g.,][Chapter~2]{nesterov2018lectures}.

\subsection{Maximum Entropy RL}
In entropy regularized RL, or sometimes called maximum entropy RL, 
near-deterministic policies are penalized
\citep{WiPe91,mnih2016asynchronous,nachum2017bridging,haarnoja2018soft,mei2019principled},
which is achieved by modifying the value of a policy $\pi$ to
\begin{equation}
\label{eq:soft_value_function_objective}
	\tilde{V}^{\pi}(\rho) \coloneqq V^\pi(\rho) + \tau \cdot \sH(\rho, \pi)\,,
\end{equation}
where  
$\sH(\rho, \pi)$ is the ``discounted entropy'', defined as 
\begin{align}
\label{eq:discounted_entropy}
	\sH(\rho, \pi) \coloneqq \expectation_{\substack{s_0 \sim \rho, a_t \sim \pi(\cdot | s_t), \\ s_{t+1} \sim \gP( \cdot | s_t, a_t)}}{\left[ \sum_{t=0}^{\infty}{- \gamma^t \log{\pi(a_t | s_t)}} \right]}.
\end{align}
and
$\tau \ge 0$, the ``temperature'', determines the strength of the penalty.%
\footnote{
To better align with naming conventions in information-theory,
discounted entropy should be rather called the discounted action-entropy rate
as entropy itself in the literature on Markov chain information theory would normally refer
to the entropy of the stationary distribution of the chain, while entropy rate refers to what is being used here.
}
Clearly, the value of any policy can be obtained by adding an entropy penalty to the rewards (as proposed originally by \citet{WiPe91}). Hence,
similarly to \cref{lem:policy_gradient_softmax}, one can obtain the following expression for the gradient of the entropy regularized objective under the softmax policy parametrization:
\begin{lemma}
\label{lem:policy_gradient_entropy}
It holds that for all $(s,a)$,
\begin{align}
\label{eq:policy_gradient_entropy}
    \frac{\partial \tilde{V}^{\pi_\theta}(\mu)}{\partial \theta(s,a)} = \frac{1}{1-\gamma} \cdot d_{\mu}^{\pi_\theta}(s) \cdot \pi_\theta(a|s) \cdot \tilde{A}^{\pi_\theta}(s,a),
\end{align}
where $\tilde{A}^{\pi_\theta}(s, a)$ is the ``soft'' advantage function defined as
\begin{align}
     \tilde{A}^{\pi_\theta}(s, a) &\coloneqq \tilde{Q}^{\pi_\theta}(s, a) - \tau \log{\pi_\theta(a | s)} - \tilde{V}^{\pi_\theta}(s), \\
     \tilde{Q}^{\pi_\theta}(s, a) &\coloneqq r(s,a) + \gamma \sum_{s^\prime}{ \gP( s^\prime | s, a)  \tilde{V}^{{\pi_\theta}}(s^\prime) }.
\end{align}
\end{lemma}
\subsection{Convergence Rates}
As in the non-regularized case,  to gain insight, we first consider MDPs with a single state and $\gamma=0$.

\subsubsection{Bandit Case}
\label{sec:entropy_policy_gradient_one_state_case}

In the one-state case with $\gamma=0$,
\cref{eq:soft_value_function_objective} reduces to maximizing the entropy-regularized reward,
\begin{align}
\label{eq:maximum_entropy_reward}
    \max_{\theta : \gA \to \sR}{ \expectation_{a \sim \pi_{\theta}}{ \left[ r(a) - \tau \log{\pi_\theta(a)} \right] } }.
\end{align}
Again, \cref{eq:maximum_entropy_reward} is a non-concave function of $\theta$. In this case, regularized policy gradient reduces to
\begin{align}
    \frac{ d \{ \pi_\theta^\top ( r - \tau \log{\pi_\theta}) \} }{d \theta} = H(\pi_\theta) (r - \tau \log{\pi_\theta}),
\end{align}
where $H(\pi_\theta)$ is the same as in \cref{eq:H_matrix}. Using the above gradient in \cref{alg:policy_gradient_softmax} we have the following update rule:
\begin{update_rule}[Softmax policy gradient, maximum entropy reward]
\label{update_rule:entropy_special}
$\theta_{t+1} \gets \theta_t + \eta \cdot H(\pi_{\theta_t}) (r - \tau \log{\pi_{\theta_t}})$.
\end{update_rule}
Due to the presence of regularization, the optimal solution will be biased with the bias disappearing as $\tau \to 0$:
\paragraph{Softmax optimal policy.} $\pi_\tau^* \coloneqq \softmax(r/ \tau)$ is the optimal solution of \cref{eq:maximum_entropy_reward}. 
\begin{remark}
\label{rmk:biased_softmax_optimal_policy}
At this stage, we could use arguments similar to those of  \cref{sec:policy_gradient} to show the $O(1/t)$ convergence of $\pi_{\theta_t}$ to $\pi_\tau^*$. However, we can use an alternative idea to show that entropy-regularized policy gradient converges significantly faster. The issue of bias will be discussed later. 
\end{remark}
Our alternative idea is to show that \cref{update_rule:entropy_special} defines a contraction but with a contraction coefficient that depends on the parameter that the update is applied to:
\begin{lemma}[Non-uniform contraction]
\label{lem:contraction_entropy_special}
Using \cref{update_rule:entropy_special} with $\tau \eta \le 1$, $\forall t > 0$,
\begin{align}
    \| \zeta_{t+1} \|_2 \le \left(1 - \tau \eta \cdot \min_{a}  \pi_{\theta_t}(a)  \right) \cdot \| \zeta_{t} \|_2,
\end{align}
where $\zeta_t \coloneqq \tau \theta_{t} - r - \frac{(\tau \theta_{t} - r)^\top \rvone}{K} \cdot \rvone$.
\end{lemma}
This lemma immediately implies the following bound:
\begin{lemma}
\label{lem:matching_entropy_special}
Using \cref{update_rule:entropy_special} with $\tau \eta \le 1$, $\forall t > 0$,
\begin{align}
    \| \zeta_{t} \|_2 \le \frac{ 2 ( \tau \norm{\theta_1}_\infty +1 ) \sqrt{K} }{\exp\left\{ \tau \eta \sum_{s=1}^{t-1}{ \min_{a}{ \pi_{\theta_s}(a) } }  \right\}}\,.
\end{align}
\end{lemma}
Similarly to \cref{lem:lower_bound_cT_softmax_special}, we can show that the minimum action probability can be lower bounded by its initial value.
\begin{lemma}
\label{lem:lower_bound_min_prob_entropy_special}
There exists $c=c(\tau,K,\norm{\theta_1}_\infty)>0$,  
such that for all $t \ge 1$, $\min_{a}{ \pi_{\theta_t}(a) } \ge c$.
Thus, $\sum_{s=1}^{t-1}{ \min_{a}{ \pi_{\theta_s}(a) } } \ge c \cdot (t-1) $.
\end{lemma}
A closed-form expression for $c$ is given in the appendix.
Note that when $\tau=0$ (no regularization), the result would no longer hold true.
The key here 
is that $\min_a \pi_{\theta_t}(a) \to \min_a \pi^*_\tau(a)>0$ as $t\to\infty$
and the latter inequality holds thanks to $\tau>0$. 
From \cref{lem:matching_entropy_special,lem:lower_bound_min_prob_entropy_special}, it follows that entropy regularized softmax policy gradient enjoys a linear convergence rate:
\begin{theorem}
\label{thm:rates_entropy_special}
Using \cref{update_rule:entropy_special} with $\eta \le 1 / \tau$, for all $t \ge 1$,
\begin{align}
    \tilde{\delta}_t &\le \frac{ 2 ( \tau \norm{\theta_1}_\infty +1 )^2 K / \tau }{\exp\left\{ 2 \tau \eta \cdot c \cdot (t-1) \right\}},
\end{align}
where $\tilde{\delta}_t \coloneqq { \pi_\tau^* }^\top \left( r - \tau \log{ \pi_\tau^* } \right)  - \pi_{\theta_t}^\top \left( r - \tau \log{ \pi_{\theta_t} } \right)$ and $c>0$ is from \cref{lem:lower_bound_min_prob_entropy_special}.
\end{theorem}

\subsubsection{General MDPs}
\label{sec:entropy_policy_gradient_general_mdps}

For general MDPs, the problem is to maximize $\tilde{V}^{\pi_\theta}(\rho)$ in \cref{eq:soft_value_function_objective}. The softmax optimal policy $\pi_\tau^*$ is known to satisfy the following consistency conditions \citep{nachum2017bridging}:
\begin{align}
\label{eq:path_consistency_conditions_1}
    \pi_\tau^*(a | s) &= \exp\left\{ ( \tilde{Q}^{\pi_\tau^*}(s, a) - \tilde{V}^{\pi_\tau^*}(s) ) / \tau \right\}, \\
\label{eq:path_consistency_conditions_2}
    \tilde{V}^{\pi_\tau^*}(s) &= \tau \log{ \sum_{a} \exp\left\{ \tilde{Q}^{\pi_\tau^*}(s, a) / \tau \right\}}.
\end{align}
Using a somewhat lengthy calculation, we show that the discounted entropy in \cref{eq:discounted_entropy} is smooth:
\begin{lemma}[Smoothness]
\label{lem:smoothness_entropy_general}
$\sH(\rho, \pi_\theta)$ is $(4 + 8 \log{A}) /(1-\gamma)^3$-smooth, where $A \coloneqq | \gA |$ is the total number of actions.
\end{lemma}
Our next key result shows that the augmented value function $\tilde{V}^{\pi_\theta}(\rho)$ satisfies a ``better type'' of \L{}ojasiewicz inequality:
\begin{lemma}[Non-uniform \L{}ojasiewicz] 
\label{lem:lojasiewicz_entropy_general}
Suppose $\mu(s) > 0$ for all state $s \in \gS$. Then,
\begin{align}
    \bigg\| \frac{\partial \tilde{V}^{{\pi_\theta}}(\mu) }{\partial \theta} \bigg\|_2 \ge C(\theta) \cdot \left[ \tilde{V}^{\pi_\tau^*}(\rho) - \tilde{V}^{{\pi_\theta}}(\rho) \right]^{\frac{1}{2}},
\end{align}
where 
\begin{align*}
    C(\theta) \coloneqq \frac{\sqrt{2 \tau}}{\sqrt{S}} \cdot \min_{s}{\sqrt{ \mu(s) } } \cdot \min_{s,a}{ \pi_\theta(a | s)  } \cdot \bigg\| \frac{d_{\rho}^{\pi_\tau^*} }{ d_{\mu}^{\pi_\theta}} \bigg\|_\infty^{-\frac{1}{2}}.
\end{align*}
\end{lemma}
The main difference to the previous versions of the non-uniform
\L{}ojasiewicz inequality
  is that the sub-optimality gap appears under the square root.
For small sub-optimality gaps this means that the gradient must be larger -- a stronger ``signal''.
Next, we show that action probabilities are still uniformly bounded away from zero:
\begin{lemma}
\label{lem:lower_bound_min_prob_entropy_general}
Using \cref{alg:policy_gradient_softmax} with 
the entropy regularized objective, 
we have $c:=\inf_{t \ge 1} \min_{s,a} { \pi_{\theta_t}(a | s) } > 0$.
\end{lemma}
With \cref{lem:smoothness_entropy_general,lem:lojasiewicz_entropy_general,lem:lower_bound_min_prob_entropy_general}, we show a $O(e^{-t})$ rate for entropy regularized policy gradient in general MDPs:
\begin{theorem}
\label{thm:final_rates_entropy_general}
Suppose $\mu(s) > 0$ for all state $s$. Using \cref{alg:policy_gradient_softmax} with the entropy regularized objective and softmax parametrization and $\eta = (1 - \gamma)^3/(8 + \tau ( 4 + 8 \log{A}))$,
there exists a constant $C>0$ such that
 for all $t \ge 1$,
\begin{align*}
    \tilde{V}^{\pi_\tau^*}(\rho) - \tilde{V}^{\pi_{\theta_t}}(\rho) \le \bigg\| \frac1\mu \bigg\|_\infty  \cdot \frac{1 + \tau \log{A}}{(1 - \gamma)^2}
    \, \cdot \, e^{-C(t-1)}
    \,.
\end{align*}
\end{theorem}
The value of the constant $C$ in this theorem appears in the proof of the result in the appendix in a closed form.

\subsubsection{Controlling the Bias} 
As noted in \cref{rmk:biased_softmax_optimal_policy}, $\pi_\tau^*$ is biased, i.e., $\pi_\tau^* \not= \pi^*$ for fixed $\tau > 0$. We discuss two possible approaches to deal with the bias, but much remains to be done to properly address the bias.
For simplicity, we consider the bandit case.
\paragraph{A two-stage approach.}
Note that for any fixed $\tau>0$, $\pi_\tau^*(a^*) \ge \pi_\tau^*(a)$ for all $a \not= a^*$. Therefore, using policy gradient with $\pi_{\theta_1} = \pi_\tau^*$, we have $\pi_{\theta_t}(a^*) \ge c_t \ge 1/K$. This suggests a two-stage method: first, 
to ensure $\pi_{\theta_t}(a^*) \ge \max_a \pi_{\theta_t}(a)$, use entropy-regularized policy gradient some iterations and then turn off regularization.
\begin{theorem}
\label{thm:rates_two_stage_special}
Denote $\Delta = r(a^*) - \max_{a \not= a^*}{ r(a) } > 0$. Using \cref{update_rule:entropy_special} for $t_1 \in O( e^{ 1/ \tau }  \cdot \log{( \frac{ \tau + 1}{\Delta } } ) )$ iterations and then \cref{update_rule:softmax_special} for $t_2 \ge 1$ iterations, we have,
\begin{align}
    ( \pi^* - \pi_{\theta_t} )^\top r \le 5 / ( {C}^2 \cdot t_2),
\end{align}
where $t = t_1 + t_2$, and $C \in [1/K, 1)$.
\end{theorem}
This approach removes the nasty dependence on the choice of the initial parameters.
While this dependence is also removed if we initialize with the uniform policy, uniform initialization is insufficient
if only noisy estimates of the gradients are available. However, we leave the study of this case for future work.
An obvious problem with this approach is that $\Delta$ is unknown. 
This can be helped by exiting the first phase when we detect ``convergence'' e.g. by detecting that the relative change of the policy is small.
\paragraph{Decreasing the penalty.} 
Another simple idea is to decrease the strength of regularization, e.g., set $\tau_t \in O( 1/ \log{t} )$. 
Consider the following update, which is a slight variation of the previous one:
\begin{update_rule}
\label{update_rule:decaying_entropy_special}
$\theta_{t+1} \gets \frac{\tau_t }{ \tau_{t+1}} \cdot ( \theta_t + \eta_t \cdot H(\pi_{\theta_t}) (r - \tau_t \log{\pi_{\theta_t}} ) )$.
\end{update_rule}
The rationale for the scaling factor is that it allows one to prove a variant of
 \cref{lem:contraction_entropy_special}.
While this is promising, the proof cannot be finished as before.
The difficulty is that $\pi_{\theta_t}\to \pi^*$ (which is what we want to achieve) implies that 
$\min_{a}{ \pi_{\theta_t}(a) } \to 0$, which prevents the use of our previous proof technique. We show the following partial results.
\begin{theorem}
\label{thm:rates_decaying_entropy_special}
Using \cref{update_rule:decaying_entropy_special} with $\tau_t = \frac{\alpha \cdot \Delta}{\log{t}}$ for $t \ge 2$, where $\alpha > 0$, and $\eta_t = 1/\tau_t$, we have, for all $t \ge 1$,
\begin{align*}
    ( \pi^* - \pi_{\theta_t} )^\top r \le \frac{K}{t^{1/\alpha}} + \frac{ C \cdot  \log{t}  }{\exp\{ \sum_{s=1}^{t-1}{ \min_{a}{ \pi_{\theta_s}(a) } }  \}},
\end{align*}
where $C \coloneqq \frac{2 ( \tau_1 \norm{\theta_1}_\infty +1 ) \sqrt{K}}{\alpha \cdot \Delta}$.
\end{theorem}
The final rates then depend on how fast $\min_{a}{ \pi_{\theta_t}(a) }$ diminishes as function of $t$. We conjecture that the rate in some cases degenerates to $O( \frac{\log{t}}{t^{1/\alpha}})$, which is strictly faster than $O(1/t)$ in non-regularized case when $\alpha \in (0,1)$ and is observed in simulations in the appendix. We leave it as an open problem to study decaying entropy in general MDPs.


\section{Does Entropy Regularization Really Help?}
\label{sec:theoretical_understanding_entropy}
The previous section indicated that entropy regularization may speed up convergence.
In addition, ample empirical evidence suggest that this may be the case
\citep[e.g.,][]{WiPe91,mnih2016asynchronous,nachum2017bridging,haarnoja2018soft,mei2019principled}.
In this section, we aim to provide new insights into why entropy may help policy optimization, taking an optimization perspective. 

We start by establishing a lower bound that shows that the $O(1/t)$ rate we established earlier for softmax policy gradient without entropy regularization cannot be improved. Next, we introduce the notion of \L{}ojasiewicz degree, which we show to increase in the presence of entropy regularization.
We then connect a higher degree to faster convergence rates.
Note that our proposal to view entropy regularization as an optimization aid is somewhat conflicting with the more common explanation that entropy regularization helps by encouraging exploration. While it is definitely true that entropy regularization encourages exploration, the form of exploration it encourages is not sensitive to epistemic uncertainty and as such it fails to provide a satisfactory solution to the exploration problem \citep[e.g.,][]{o2020making}.

\subsection{Lower Bounds}
The purpose of this section is to establish that the $O(1/t)$ rates established earlier for unpenalized policy gradient is tight.
To get lower bounds, we need to show that progress in every iteration cannot be too large. 
This holds when we can reverse the inequality in the \L{}ojasiewicz inequality.
To this regard, in bandit problems we have the following result:
\begin{lemma}[Reversed \L{}ojasiewicz]
\label{lem:reverse_lojasiewicz_softmax_special}
Take any $r\in [0,1]^K$.
Denote $\Delta = r(a^*) - \max_{a \not= a^*}{ r(a) } > 0$. Then,
\begin{align}
    \bigg\| \frac{d \pi_\theta^\top r}{d \theta} \bigg\|_2 \le (\sqrt{2} / \Delta) \cdot (\pi^* - \pi_\theta)^\top r.
\end{align}
\end{lemma}
Using this result gives the desired lower bound:
\begin{theorem}[Lower bound]
\label{thm:lower_bound_softmax_special}
Take any $r\in [0,1]^K$.
For large enough $t \ge 1$, using \cref{update_rule:softmax_special} with learning rate $\eta_t \in ( 0 , 1]$,
\begin{align}
    (\pi^* - \pi_{\theta_t})^\top r \ge \Delta^2 / (6 \cdot t ).
\end{align}
\end{theorem}
Note that \cref{thm:lower_bound_softmax_special} is a special case of general MDPs. Next, we strengthen this result and show that the $\Omega(1/t)$ lower bound also holds for \emph{any} MDP:
\begin{theorem}[Lower bound]
\label{thm:lower_bound_softmax_general}
Take any MDP.
For large enough $t \ge 1$, using  \cref{alg:policy_gradient_softmax} with $\eta_t \in ( 0, 1] $,
\begin{align}
    V^*(\mu) - V^{\pi_{\theta_t}}(\mu) \ge \frac{ (1- \gamma)^5 \cdot (\Delta^*)^2}{12 \cdot t},
\end{align}
where $\Delta^* \coloneqq \min_{s \in \gS, a \not= a^*(s)}\{ Q^*(s, a^*(s)) - Q^*(s, a) \} > 0$ is the optimal value gap of the MDP.
\end{theorem}
\begin{remark}
Our convergence rates in \cref{sec:policy_gradient} match the lower bounds up to constant. However, the constant gap is large, e.g., $K^2$ in \cref{thm:rates_uniform_softmax_special}, and $\Delta^2$ in \cref{thm:lower_bound_softmax_special}.
The gap is because the reversed \L{}ojasiewicz inequality of 
\cref{lem:reverse_lojasiewicz_softmax_special} uses $\Delta$, which is unavoidable when $\pi_\theta$ is close to $\pi^*$. 
We leave it as an open problem to close this gap.
\end{remark}
With the lower bounds established, we confirm that entropy regularization helps policy optimization by speeding up convergence, though the question remains as to the mechanism by which the improved convergence rate manifests itself.


\subsection{Non-uniform \L{}ojasiewicz Degree}
To gain further insight into how entropy regularization helps, we introduce the non-uniform \L{}ojasiewicz degree:
\begin{definition}[Non-uniform \L{}ojasiewicz degree]
\label{def:non_uniform_lojasiewicz_degree}
A function $f : \gX \to \sR$ has \L{}ojasiewicz degree $\xi \in [0, 1]$ if\footnote{In literature \citep{lojasiewicz1963propriete}, $C$ cannot depend on $x$.
Based on the examples we have seen, we relax this requirement.
}
\begin{align}
    \left\| \nabla_x f(x) \right\|_2 \ge C(x) \cdot \left| f(x) - f(x^*) \right|^{1-\xi},
\end{align}
$\forall x \in \gX$, where $C(x) > 0$ holds for all $x\in \gX$.
\end{definition}
The uniform degree, where $C(x)$ is a positive constant, has previously been connected to convergence speed in the optimization literature.
\citet{barta2017rate} studied this effect for first-, while 
\citet{nesterov2006cubic,zhou2018convergence} studied this for second-order methods.
As noted beforehand, a larger degree (smaller exponent of the sub-optimality) is expected to improve the convergence speed of algorithms that rely on gradient information.
Intuitively, we expect this to continue to hold for the non-uniform \L{}ojasiewicz degree as well.
With this, we now study what \L{}ojasiewicz degrees can one obtain with and without entropy regularization.

Our first result shows that the \L{}ojasiewicz degree of the expected reward objective (in bandits) cannot be positive:
\begin{proposition}
\label{prop:lojasiewicz_degree_softmax}
Let $r\in [0,1]^K$ be arbitrary and consider $\theta\mapsto \expectation_{a \sim \pi_{\theta}}{ [ r(a) ] }$.
The non-uniform \L{}ojasiewicz degree of this map with constant $C(\theta) = \pi_\theta(a^*)$ is zero.
\end{proposition}
Note that according to \cref{rmk:necessary_dependence_optimal_action_prob}, it is necessary that $C(\theta)$ depends on $\pi_\theta(a^*)$. The difference between \cref{prop:lojasiewicz_degree_softmax} and the reversed \L{}ojasiewicz inequality of  \cref{lem:reverse_lojasiewicz_softmax_special} is subtle. \cref{lem:reverse_lojasiewicz_softmax_special} is a condition that implies impossibility to get rates faster than $O(1/t)$, while \cref{prop:lojasiewicz_degree_softmax} says it is not  sufficient to get rates faster than $O(1/t)$ \textit{using the same technique as in \cref{lem:pseudo_rates_softmax_special}}. However, this does not preclude that other techniques could give faster rates.

Next, we show that the \L{}ojasiewicz degree of the entropy-regularized expected reward objective is at least $1/2$:
\begin{proposition}
\label{prop:lojasiewicz_degree_entropy}
Fix $\tau>0$.
With $C(\theta) = \sqrt{2 \tau} \cdot  \min_{a}{ \pi_\theta(a) }$, the \L{}ojasiewicz degree of $\theta \mapsto \expectation_{a \sim \pi_{\theta}}{ \left[ r(a) - \tau \log{\pi_\theta(a)} \right] }$ is at least $1/2$.
\end{proposition}

\section{Conclusions and Future Work}
\label{sec:conclusions_future_work}

We set out to study the convergence speed of 
softmax policy gradient methods with and without entropy regularization in the tabular setting.
Here, the error is measured in terms of the sub-optimality of the policy obtained after some number of updates.
Our main findings is that without entropy regularization, the rate is $\Theta(1/t)$, which is faster than rates previously obtained. Our analysis also uncovered an unpleasant dependence on the initial parameter values. 
With entropy regularization, the rate becomes linear, where now the constant in the exponent is influenced by the initial choice of parameters. Thus, our analysis shows that entropy regularization substantially changes the rate at which gradient methods converge.
Our main technical innovation is the introduction of a non-uniform variant of the \L{}ojasiewicz inequality.
Our work leaves open a number of interesting questions:
While we have some lower bounds, there remains some gaps to be filled between the lower and upper bounds.
Other interesting directions are extending the results for alternative (e.g., restricted) policy parametrizations or studying policy gradient when the gradient must be estimated from data.
One also expects that non-uniform \L{}ojasiewicz  inequalities and the \L{}ojasiewicz degree could also be put to good use in other areas of non-convex optimization.

\section*{Acknowledgements}
Jincheng Mei would like to thank Bo Dai and Lihong Li for
helpful discussions and for providing feedback on a draft of
this manuscript. Jincheng Mei would like to thank Ruitong Huang for enlightening early discussions.
Csaba Szepesv\'ari gratefully acknowledges funding  from 
the Canada CIFAR AI Chairs Program, Amii and NSERC.


\bibliography{all}
\bibliographystyle{icml2020}


\appendix
\onecolumn

The appendix is organized as follows.
\begin{itemize}
    \item \cref{sec:proofs}: proofs for the technical results in the main paper.
    \begin{itemize}
        \item \cref{sec:proofs_policy_gradient}: proofs for the results of softmax policy gradient in \cref{sec:policy_gradient}.
        \begin{itemize}
            \item \cref{sec:proofs_policy_gradient_priliminaries}: Preliminaries.
            \item \cref{sec:proofs_policy_gradient_bandits}: One-state MDPs (bandits).
            \item \cref{sec:proofs_policy_gradient_general_mdps}: General MDPs.
        \end{itemize}
        \item \cref{sec:proofs_entropy_policy_gradient}: proofs for the results of entropy regularized softmax policy gradient in \cref{sec:entropy_policy_gradient}.
        \begin{itemize}
            \item \cref{sec:proofs_entropy_policy_gradient_priliminaries}: Preliminaries.
            \item \cref{sec:proofs_entropy_policy_gradient_bandits}: One-state MDPs (bandits).
            \item \cref{sec:proofs_entropy_policy_gradient_general_mdps}: General MDPs.
            \item \cref{sec:proofs_entropy_policy_gradient_decaying_entropy_bandits}: Two-stage and decaying entropy regularization.
        \end{itemize}
        \item \cref{sec:proofs_theoretical_understanding_entropy}: proofs for \cref{sec:theoretical_understanding_entropy} (does entropy regularization really help?).
        \begin{itemize}
            \item \cref{sec:proofs_theoretical_understanding_entropy_bandits}: One-state MDPs (bandits).
            \item \cref{sec:proofs_theoretical_understanding_entropy_general_mdps}: General MDPs.
            \item \cref{sec:proofs_theoretical_understanding_entropy_lojasiewicz_degree}: Non-uniform \L{}ojasiewicz degree.
        \end{itemize}
    \end{itemize}
    \item \cref{sec:supporting_lemmas}: miscellaneous extra supporting results that are not mentioned in the main paper.
    \item \cref{sec:remark_sub_optimality_guarantees}: 
    further remarks on sub-optimality guarantees for other entropy-based RL methods beyond those presented in the main paper.
    \item \cref{sec:simulation_results}: simulation results to verify the convergence rates, which are not presented in the main paper.
\end{itemize}

\section{Proofs}
\label{sec:proofs}

\subsection{Proofs for \cref{sec:policy_gradient}: softmax parametrization}
\label{sec:proofs_policy_gradient}

\subsubsection{Preliminaries}
\label{sec:proofs_policy_gradient_priliminaries}

\textbf{\cref{lem:policy_gradient_softmax}.}
Consider the map $\theta \mapsto V^{\pi_\theta}(\mu)$ where 
$\theta\in \R^{\gS\times \gA}$ and
$\pi_\theta(\cdot|s) = \softmax(\theta(s,\cdot))$.
The derivative of this map satisfies
\begin{align}
    \frac{\partial V^{\pi_\theta}(\mu)}{\partial \theta(s,a)} = \frac{1}{1-\gamma} \cdot d_{\mu}^{\pi_\theta}(s) \cdot \pi_\theta(a|s) \cdot A^{\pi_\theta}(s,a).
\end{align}
Note that this is given as \citet[Lemma C.1]{AgKaLeMa19}; we include a proof for completeness.
\begin{proof}
According to the policy gradient theorem (\cref{thm:policy_gradient_theorem_general}),
\begin{align}
    \frac{\partial V^{\pi_\theta}(\mu)}{\partial \theta} = \frac{1}{1-\gamma} \expectation_{s^\prime \sim d_{\mu}^{\pi_\theta} } { \left[ \sum_{a}  \frac{\partial \pi_\theta(a | s^\prime)}{\partial \theta} \cdot Q^{\pi_\theta}(s^\prime,a) \right] }.
\end{align}
For $s^\prime \not= s$, $\frac{\partial \pi_\theta(a | s^\prime)}{\partial \theta(s, \cdot)} = \rvzero$ since $\pi_\theta(a | s^\prime)$ does not depend on $\theta(s, \cdot)$. Therefore,
\begin{align}
    \frac{\partial V^{\pi_\theta}(\mu)}{\partial \theta(s, \cdot)} &= \frac{1}{1-\gamma} \cdot d_{\mu}^{\pi_\theta}(s) \cdot { \left[ \sum_{a} \frac{\partial \pi_\theta(a | s)}{\partial \theta(s, \cdot)} \cdot Q^{\pi_\theta}(s,a) \right] } \\
    &= \frac{1}{1-\gamma} \cdot d_{\mu}^{\pi_\theta}(s) \cdot \left( \frac{d \pi(\cdot | s)}{d \theta(s, \cdot)} \right)^\top Q^{\pi_\theta}(s,\cdot) \\
    &= \frac{1}{1-\gamma} \cdot d_{\mu}^{\pi_\theta}(s) \cdot H(\pi_\theta(\cdot | s)) Q^{\pi_\theta}(s,\cdot). \qquad \left( \text{using \cref{eq:H_matrix}} \right)
\end{align}
Since $H(\pi_\theta(\cdot | s)) = \diagonalmatrix( \pi_\theta(\cdot | s) ) - \pi_\theta(\cdot | s) \pi_\theta(\cdot | s)^\top$, for each component $a$, we have
\begin{align}
    \frac{\partial V^{\pi_\theta}(\mu)}{\partial \theta(s, a)} &= \frac{1}{1-\gamma} \cdot d_{\mu}^{\pi_\theta}(s) \cdot \pi_\theta(a | s) \cdot \left[ Q^{\pi_\theta}(s, a) - \sum_{a}{ \pi_\theta(a | s) \cdot Q^{\pi_\theta}(s, a) } \right] \\
    &= \frac{1}{1-\gamma} \cdot d_{\mu}^{\pi_\theta}(s) \cdot \pi_\theta(a | s) \cdot \left[ Q^{\pi_\theta}(s, a) - V^{\pi_\theta}(s) \right] \qquad \left( \text{using } V^{\pi_\theta}(s) = \sum_{a}{ \pi_\theta(a | s) \cdot Q^{\pi_\theta}(s, a) } \right) \\
    &= \frac{1}{1-\gamma} \cdot d_{\mu}^{\pi_\theta}(s) \cdot \pi_\theta(a | s) \cdot A^{\pi_\theta}(s, a). \qedhere
\end{align}
\end{proof}

\subsubsection{Proofs for softmax parametrization in bandits}
\label{sec:proofs_policy_gradient_bandits}

\textbf{\cref{prop:non_concave_softmax_expected_reward}.}
On some problems, $\theta\mapsto \expectation_{a \sim \pi_{\theta}}{ [ r(a) ] }$ is a non-concave function over $\R^K$.
\begin{proof}
Consider the following example: $r = (1, 9/10, 1/10)^\top$, $\theta_1 = (0, 0, 0)^\top$, $\pi_{\theta_1} = \softmax(\theta_1) = (1/3, 1/3, 1/3)^\top$, $\theta_2 = (\ln{9}, \ln{16}, \ln{25})^\top$, and $\pi_{\theta_2} = \softmax(\theta_2) = (9/50, 16/50, 25/50)^\top$. We have,
\begin{align}
    \frac{1}{2} \cdot \left( \pi_{\theta_1}^\top r + \pi_{\theta_2}^\top r \right) = \frac{1}{2} \cdot  \left(\frac{2}{3} + \frac{259}{500} \right) = \frac{1777}{3000} = \frac{14216}{24000}.
\end{align}
On the other hand, defining  $\bar{\theta} = \frac{1}{2} \cdot \left( \theta_1 + \theta_2 \right) = \left( \ln{3}, \ln{4}, \ln{5} \right)^\top$ we have $\pi_{\bar{\theta}} = \softmax(\bar{\theta}) = \left( 3/12, 4/12, 5/12 \right)^\top$ and
\begin{align}
    \pi_{\bar{\theta}}^\top r = \frac{71}{120} = \frac{14200}{24000}.
\end{align}
Since $\frac{1}{2} \cdot \left( \pi_{\theta_1}^\top r + \pi_{\theta_2}^\top r \right) > \pi_{\bar{\theta}}^\top r$, $\theta \mapsto \expectation_{a \sim \pi_{\theta}(\cdot)}{ [ r(a) ] }$ is a non-concave function of $\theta$.
\end{proof}

\textbf{\cref{lem:smoothness_softmax_special}} (Smoothness)\textbf{.} Let $\pi_\theta = \softmax(\theta)$ and $\pi_{\theta^\prime} = \softmax(\theta^\prime)$. For any $r \in \left[ 0, 1\right]^K$, $\theta \mapsto \pi_\theta^\top r$ is $5/2$-smooth, i.e.,
\begin{align}
    \left| ( \pi_{\theta^\prime} - \pi_\theta)^\top r - \Big\langle \frac{d \pi_\theta^\top r}{d \theta}, \theta^\prime - \theta \Big\rangle \right| \le \frac{5}{4} \cdot \| \theta^\prime - \theta \|_2^2.
\end{align}
\begin{proof}
Let  $S:=S(r,\theta)\in \R^{K\times K}$ be 
the second derivative of the value map $\theta \mapsto \pi_\theta^\top r$.
By Taylor's theorem, it suffices to show that the spectral radius of $S$ (regardless of $r$ and $\theta$) is bounded by $5/2$.
Now, by its definition we have
\begin{align}
    S &= \frac{d }{d \theta } \left\{ \frac{d \pi_\theta^\top r}{d \theta} \right\} \\
    &= \frac{d }{d \theta } \left\{ H(\pi_\theta) r \right\} 
    \qquad\left( \text{using \cref{eq:H_matrix}} \right) \\
    &= \frac{d }{d \theta } \left\{ ( \diagonalmatrix(\pi_\theta) - \pi_\theta \pi_\theta^\top) r \right\}.
\end{align}
Continuing with our calculation fix $i, j \in [K]$. Then, 
\begin{align}
    S_{i, j} &= \frac{d \{ \pi_\theta(i) \cdot  ( r(i) - \pi_\theta^\top r ) \} }{d \theta(j)} \\
    &= \frac{d \pi_\theta(i) }{d \theta(j)} \cdot ( r(i) - \pi_\theta^\top r ) + \pi_\theta(i) \cdot \frac{d \{ r(i) - \pi_\theta^\top r \} }{d \theta(j)} \\
    &= (\delta_{ij} \pi_\theta(j) -  \pi_\theta(i) \pi_\theta(j) ) \cdot ( r(i) - \pi_\theta^\top r ) - \pi_\theta(i) \cdot ( \pi_\theta(j) r(j) - \pi_\theta(j) \pi_\theta^\top r ) \\
    &= \delta_{ij} \pi_\theta(j) \cdot ( r(i) - \pi_\theta^\top r ) -  \pi_\theta(i) \pi_\theta(j) \cdot ( r(i) - \pi_\theta^\top r ) - \pi_\theta(i) \pi_\theta(j) \cdot ( r(j) -  \pi_\theta^\top r ),
\end{align}
where
\begin{align}
\label{eq:delta_ij_notation}
    \delta_{ij} = \begin{cases}
		1, & \text{if } i = j, \\
		0, & \text{otherwise}
	\end{cases}
\end{align}
is Kronecker's $\delta$-function.
To show the bound on 
the spectral radius of $S$, pick $y \in \sR^K$. Then,
\begin{align}
    \left| y^\top S y \right| &= \left| \sum\limits_{i=1}^{K}{ \sum\limits_{j=1}^{K}{ S_{i,j} y(i) y(j)} } \right| \\
    &= \left| \sum_{i}{ \pi_\theta(i) ( r(i) - \pi_\theta^\top r ) y(i)^2 } - 2 \sum_{i} \pi_\theta(i) ( r(i) - \pi_\theta^\top r ) y(i) \sum_{j} \pi_\theta(j) y(j) \right| \\
    &= \left| \left( H(\pi_\theta) r \right)^\top \left( y \odot y \right) - 2 \cdot \left( H(\pi_\theta) r \right)^\top y \cdot \left( \pi_\theta^\top y \right) \right| \\
    &\le \left\| H(\pi_\theta) r \right\|_\infty \cdot \left\| y \odot y \right\|_1 + 2 \cdot \left\| H(\pi_\theta) r \right\|_1 \cdot \left\| y \right\|_\infty \cdot \left\| \pi_\theta \right\|_1 \cdot \left\| y \right\|_\infty,
\end{align}
where $\odot$ is Hadamard (component-wise) product, and the last inequality uses H{\" o}lder's inequality together with the triangle inequality. Note that $\| y \odot y \|_1 = \| y \|_2^2$, $\| \pi_\theta \|_1 = 1$, and $\| y \|_\infty \le \| y \|_2$. For $i\in [K]$, denote by $H_{i,:}(\pi_\theta)$ the $i$-th row of $H(\pi_\theta)$ as a row vector. Then,
\begin{align}
    \left\| H_{i,:}(\pi_\theta) \right\|_1 &=  \pi_\theta(i) - \pi_\theta(i)^2 + \pi_\theta(i) \cdot \sum_{j \not= i}{ \pi_\theta(j) } \\
    &= \pi_\theta(i) - \pi_\theta(i)^2 +  \pi_\theta(i) \cdot ( 1 - \pi_\theta(i) ) \\
    &= 2 \cdot \pi_\theta(i) \cdot ( 1 - \pi_\theta(i) ) \\
    &\le 1/2. \qquad\left( \text{using that } x \cdot (1 - x ) \le 1/4 \text{ holds for } x \in [0, 1] \right)
\end{align}
On the other hand,
\begin{align}
\label{eq:H_matrix_r_1_norm_upper_bound_special}
    \left\| H(\pi_\theta) r \right\|_1 &= \sum_{i}{ \pi_\theta(i) \cdot \left| r(i) - \pi_\theta^\top r \right| } 
    \\
    &\le \max_{i}{ \left| r(i) - \pi_\theta^\top r \right| } \\
    &\le 1. \qquad\left( \text{using } r \in \left[ 0, 1\right]^K \right)
\end{align}
Therefore we have,
\begin{align}
\label{eq:H_matrix_maximum_eigenvalue}
    \left| y^\top S(r, \theta) y \right| &\le \left\| H(\pi_\theta) r \right\|_\infty \cdot \left\| y \right\|_2^2 + 2 \cdot \left\| H(\pi_\theta) r \right\|_1 \cdot \left\| y \right\|_2^2 \\
    &= \max_{i} \left| \left( H_{i,:}(\pi_\theta) \right)^\top r \right| \cdot \left\| y \right\|_2^2 + 2 \cdot \left\| H(\pi_\theta) r \right\|_1 \cdot \left\| y \right\|_2^2 \\
    &\le \max_{i} \left\| H_{i,:}(\pi_\theta) \right\|_1 \cdot \left\| r \right\|_\infty \cdot \left\| y \right\|_2^2 + 2 \cdot 1 \cdot \left\| y \right\|_2^2 \\
    &\le (1/2 + 2) \cdot \left\| y \right\|_2^2 = 5/2 \cdot \left\| y \right\|_2^2,
\end{align}
finishing the proof.
\end{proof}

\textbf{\cref{lem:lojasiewicz_softmax_special}} (Non-uniform \L{}ojasiewicz)\textbf{.}
Assume $r$ has a single maximizing action $a^*$. Let $\pi^* \coloneqq \argmax_{\pi \in \Delta}{ \pi^\top r}$, and $ \pi_\theta =\softmax(\theta)$. Then, for any $\theta$,
\begin{align}
    \left\| \frac{d \pi_\theta^\top r}{d \theta} \right\|_2 \ge \pi_\theta(a^*) \cdot ( \pi^* - \pi_\theta )^\top r\,.
\end{align}
When there are multiple optimal actions,  we have
\begin{align}
    \left\| \frac{d \pi_\theta^\top r}{d \theta} \right\|_2 \ge \frac{1}{ \sqrt{| \gA^* |} } \cdot \left[ \sum_{a^* \in \gA^*}{ \pi_\theta(a^*) } \right] \cdot ( \pi^* - \pi_\theta )^\top r,
\end{align}
where $\gA^* = \{ a^* : r(a^*) = \max_{a}{ r(a) } \}$ is the set of optimal actions.
\begin{proof}
We give the proof for the general case, as the case of a single maximizing action is a corollary to this case. Using the expression we got for the gradient earlier, 
\begin{align}
    \left\| \frac{d \pi_\theta^\top r}{d \theta} \right\|_2 &\ge \left( \sum_{a^* \in \gA^* }{\left[ \pi_\theta(a^*) \cdot (r(a^*) - \pi_\theta^\top r) \right]^2} \right)^\frac{1}{2} \\
    &\ge \frac{1}{ \sqrt{| \gA^* |} } \sum_{a^* \in \gA^* }{ \pi_\theta(a^*) \cdot (r(a^*) - \pi_\theta^\top r) } \qquad\left(\text{by Cauchy-Schwarz}\right) \\
    &= \frac{1}{ \sqrt{| \gA^* |} } \cdot \left[ \sum_{a^* \in \gA^*}{ \pi_\theta(a^*) } \right] \cdot ( \pi^* - \pi_\theta )^\top r. \qedhere
\end{align}
\end{proof}

For the remaining results in this section, for simplicity, we assume that $\gA^* = \{a^*\}$, i.e., there is a unique optimal action $a^*$.

\textbf{\cref{lem:pseudo_rates_softmax_special}} (Pseudo-rate)\textbf{.}
Let $\pi_{\theta_t} = \softmax(\theta_t)$, and $c_t = \min_{1 \le s \le t}{ \pi_{\theta_s}(a^*) }$.
Using \cref{update_rule:softmax_special} with $\eta = 2/5$, for all $t \ge 1$,
\begin{align}
\label{eq:pseudo_regret}
    ( \pi^* - \pi_{\theta_t} )^\top r &\le 5 / ( t \cdot c_t^2), \qquad \text{and} \\
    \sum_{t=1}^{T}{ ( \pi^* - \pi_{\theta_t} )^\top r} &\le \min\left\{ \sqrt{5 T} / c_T , \ ( 5 \log{T} ) / c_T^2 + 1\right\}.
\end{align}
\begin{proof}
According to \cref{lem:smoothness_softmax_special},
\begin{align}
    \left| ( \pi_{\theta_{t+1}} - \pi_{\theta_t})^\top r - \Big\langle \frac{d \pi_{\theta_t}^\top r}{d \theta_t}, \theta_{t+1} - \theta_t \Big\rangle \right| \le \frac{5}{4} \cdot \| \theta_{t+1} - \theta_t \|_2^2,
\end{align}
which implies
\begin{align}
\label{eq:smoothness_progress_special}
    \pi_{\theta_t}^\top r - \pi_{\theta_{t+1}}^\top r &\le - \Big\langle \frac{d \pi_{\theta_t}^\top r}{d \theta_t}, \theta_{t+1} - \theta_{t} \Big\rangle + \frac{5}{4} \cdot \| \theta_{t+1} - \theta_{t} \|_2^2 \\
    &= - \eta \cdot \left\| \frac{d \pi_{\theta_t}^\top r}{d \theta_t} \right\|_2^2 + \frac{5}{4} \cdot \eta^2 \cdot \left\| \frac{d \pi_{\theta_t}^\top r}{d \theta_t} \right\|_2^2 
    \qquad\left(\text{using } \theta_{t+1} = \theta_t + \eta \cdot \frac{d \pi_{\theta_t}^\top r}{d \theta_t} \right) \\
    &= - \frac{1}{5} \cdot \left\| \frac{d \pi_{\theta_t}^\top r}{d \theta_t} \right\|_2^2 
    \qquad\qquad \left(\text{using } \eta = 2/5 \right) \\
    &\le - \frac{1}{5} \cdot \left[ \pi_{\theta_t}(a^*) \cdot ( \pi^* - \pi_{\theta_t} )^\top r \right]^2 
    \qquad\left(\text{by \cref{lem:lojasiewicz_softmax_special}} \right) \\
    &\le - \frac{c_t^2}{5} \cdot \left[ ( \pi^* - \pi_{\theta_t} )^\top r \right]^2, 
    \qquad\left(\text{by the definition of } c_t \right)
\end{align}
which is equivalent to
\begin{align}
\label{eq:delta_iteration}
    ( \pi^* - \pi_{\theta_{t+1}} )^\top r - ( \pi^* - \pi_{\theta_t} )^\top r \le - \frac{c_t^2}{5} \cdot \left[ ( \pi^* - \pi_{\theta_t} )^\top r \right]^2.
\end{align}
Let $\delta_t = ( \pi^* - \pi_{\theta_t} )^\top r$. 
To prove the first part, we need to show 
 that $\delta_t \le \frac{5}{c_t^2} \cdot \frac{1}{t}$ holds for any $t\ge 1$.
We prove this by induction on $t$.

\noindent Base case:
Since $\delta_t \le 1$ and $c_t \in (0,1)$, the result trivially holds up to $t\le 5$. 

\noindent Inductive step:
Now, let $t\ge 2$ and suppose that $\delta_t \le \frac{5}{c_t^2} \cdot \frac{1}{t}$. 
Consider $f_t : \sR \to \sR$
defined using $f_t(x) = x - \frac{c_t^2}{5} \cdot x^2$. 
We have that $f_t$ is monotonically increasing in $\big[0, \frac{5}{2 \cdot c_t^2} \big]$.
Hence,
\begin{align}
\label{eq:one_over_t_induction}
    \delta_{t+1} &\le f_t(\delta_{t})
    \qquad\qquad \left(\text{by \cref{eq:delta_iteration}} \right)
    \\
    &\le f_t\left( \frac{5}{c_t^2} \cdot \frac{1}{t} \right) 
    \qquad\left(\text{using } \delta_t \le \frac{5}{c_t^2} \cdot \frac{1}{t} \le \frac{5}{2 \cdot c_t^2}, \ t \ge 2 \right) \\
    &= \frac{5}{c_t^2} \cdot \left( \frac{1}{t} - \frac{1}{t^2} \right) \\
    &\le \frac{5}{c_t^2} \cdot \frac{1}{t+1} \\
    &\le \frac{5}{c_{t+1}^2} \cdot \frac{1}{t+1},
    \qquad\left(\text{using } c_t \ge c_{t+1} > 0 \right)
\end{align}
which completes the induction and the proof of the first part of the lemma.

For the second part, summing up $\delta_t \le \frac{5}{c_t^2} \cdot \frac{1}{t} \le \frac{5}{c_T^2} \cdot \frac{1}{t}$, we have
\begin{align}
     \sum_{t=1}^{T}{ ( \pi^* - \pi_{\theta_t} )^\top r} \le \frac{5 \log{T}}{c_T^2} + 1.
\end{align}
On the other hand, rearranging \cref{eq:delta_iteration} and summing up $\delta_t^2 \le \frac{5}{c_t^2} \cdot (\delta_t - \delta_{t+1})  \le \frac{5}{c_T^2} \cdot (\delta_t - \delta_{t+1}) $ from $t = 1$ to $T$,
\begin{align}
    \sum_{t=1}^{T}{ \delta_t^2 } &\le \frac{5}{c_T^2} \sum_{t=1}^{T}{ \left( \delta_t - \delta_{t+1} \right) } \\
    &= \frac{5}{c_T^2} \cdot \left( \delta_1 - \delta_{T+1} \right) \\
    &\le \frac{5}{c_T^2}. \qquad\left(\text{since } \delta_{T+1} \ge 0, \ \delta_1 \le 1 \right)
\end{align}
Therefore, by Cauchy-Schwarz,
\begin{equation*}
    \sum_{t=1}^{T}{ ( \pi^* - \pi_{\theta_t} )^\top r } = \sum_{t=1}^{T}{ \delta_t } \le \sqrt{T} \cdot \sqrt{\sum_{t=1}^{T}{ \delta_t^2 }} \le \frac{ \sqrt{5 T}}{ c_T }. \qedhere
\end{equation*}
\end{proof}

\textbf{\cref{lem:lower_bound_cT_softmax_special}.}
For $\eta=2/5$, we have 
$\inf_{t\ge 1} \pi_{\theta_t}(a^*) > 0$.
\begin{proof}
Let 
\begin{align}
    c = \frac{K}{2 \Delta} \cdot \left(1 - \frac{\Delta}{K} \right)
\end{align}
and 
\begin{align}
    \Delta = r(a^*) - \max_{a \not= a^*}{ r(a) } > 0
\end{align}
denote the reward gap of $r$.
We will prove that $\inf_{t\ge 1} \pi_{\theta_t}(a^*) = \min_{1 \le t \le t_0}{ \pi_{\theta_t}(a^*) }$, where $t_0 =\min\{ t: \pi_{\theta_t}(a^*) \ge \frac{c}{c+1} \}$. 
Note that $t_0$ depends only on $\theta_1$ and $c$, and $c$ depends only on the problem.
Define the following regions,
\begin{align}
    \gR_1 &= \left\{ \theta : \frac{d \pi_\theta^\top r}{d \theta(a^*)} \ge \frac{d \pi_\theta^\top r}{d \theta(a)}, \ \forall a \not= a^* \right\}, \\
    \gR_2 & = \left\{ \theta : \pi_\theta(a^*) \ge \pi_\theta(a), \ \forall a \not= a^* \right\}\,,\\
    \gN_c & = \left\{ \theta : \pi_\theta(a^*) \ge \frac{c}{c+1} \right\}\,.
\end{align}

We make the following three-part claim.
\begin{claim}\label{cl:regions}
The following hold:
\begin{description}
    \item[a)] \label{cl:regions:a}
    $\gR_1$ is a ``nice" region, in the sense that if $\theta_{t} \in \gR_1$ then, with any $\eta>0$,
     following a gradient update
     {\em (i)}  $\theta_{t+1} \in \gR_1$ 
     and 
     {\em (ii)} $\pi_{\theta_{t+1}}(a^*) \ge \pi_{\theta_{t}}(a^*)$.
    \item[b)]  We have $\gR_2\subset \gR_1$ and $\gN_c \subset \gR_1$.
    \item[c)] 
    For $\eta=2/5$,
    there exists a finite time $t_0 \ge 1$, such that $\theta_{t_0} \in \gN_c$, and thus $\theta_{t_0} \in \gR_1$, which implies that $\inf_{t\ge 1} \pi_{\theta_t}(a^*) = \min_{1 \le t \le t_0}{ \pi_{\theta_t}(a^*) }$.
\end{description}
\end{claim}
\paragraph{Claim a)}
Part~(i): We want to show that 
if $\theta_{t} \in \gR_1$, then $\theta_{t+1} \in \gR_1$. 
Let 
\begin{align}
\gR_1(a) &= \left\{ \theta : \frac{d \pi_\theta^\top r}{d \theta(a^*)} \ge \frac{d \pi_\theta^\top r}{d \theta(a)}\right\}\,.
\end{align}
Note that $\gR_1 = \cap_{a\ne a^*} \gR_1(a)$.
Pick $a\ne a^*$. Clearly, it suffices to show that if $\theta_t\in \gR_1(a)$ then $\theta_{t+1}\in \gR_1(a)$.
Hence, suppose that $\theta_t\in \gR_1(a)$.
We consider two cases.

\noindent Case (a): $\pi_{\theta_t}(a^*) \ge \pi_{\theta_t}(a)$.
Since $\pi_{\theta_t}(a^*) \ge \pi_{\theta_t}(a)$, we also have $\theta_t(a^*) \ge \theta_t(a)$. 
After an update of the parameters,
\begin{align}
    \theta_{t+1}(a^*) &= \theta_{t}(a^*) + \eta \cdot \frac{d \pi_{\theta_t}^\top r}{d \theta_t(a^*)} \\
    &\ge \theta_t(a) + \eta \cdot \frac{d \pi_{\theta_t}^\top r}{d \theta_t(a)} \\
    &= \theta_{t+1}(a),
\end{align}
which implies that $\pi_{\theta_{t+1}}(a^*) \ge \pi_{\theta_{t+1}}(a)$. Since $r(a^*) - \pi_{\theta_{t+1}}^\top r > 0$ and $r(a^*) > r(a)$,
\begin{align}
    \pi_{\theta_{t+1}}(a^*) \cdot \left( r(a^*) - \pi_{\theta_{t+1}}^\top r \right) \ge \pi_{\theta_{t+1}}(a) \cdot \left( r(a) - \pi_{\theta_{t+1}}^\top r \right), 
\end{align}
which is equivalent to $\frac{d \pi_{\theta_{t+1}}^\top r}{d \theta_{t+1}(a^*)} \ge \frac{d \pi_{\theta_{t+1}}^\top r}{d \theta_{t+1}(a)}$, i.e., $\theta_{t+1} \in \gR_1(a)$.

\noindent Case (b): Suppose now that $\pi_{\theta_t}(a^*) < \pi_{\theta_t}(a)$.
First note that for any $\theta$ and $a\ne a^*$, $\theta \in \gR_1(a)$ holds if and only if
\begin{align}
\label{eq:r1acond}
    r(a^*) - r(a) &\ge \left( 1 - \frac{\pi_{\theta}(a^*)}{\pi_{\theta}(a)} \right) \cdot \left( r(a^*) - \pi_{\theta}^\top r \right)\,.
\end{align}
Indeed, from the condition $\frac{d \pi_{\theta}^\top r}{d \theta(a^*)} \ge \frac{d \pi_{\theta}^\top r}{d \theta(a)}$, we get
\begin{align}
    \pi_{\theta}(a^*) \cdot \left( r(a^*) - \pi_{\theta}^\top r \right) &\ge \pi_{\theta}(a) \cdot \left( r(a) - \pi_{\theta}^\top r \right) \\
    &= \pi_{\theta}(a) \cdot \left( r(a^*) - \pi_{\theta}^\top r \right) - \pi_{\theta}(a) \cdot \left( r(a^*) - r(a) \right),
\end{align}
which, after rearranging, is equivalent to \cref{eq:r1acond}.
Hence, it suffices to show that  \cref{eq:r1acond} holds for $\theta_{t+1}$ provided it holds for $\theta_t$.

From the latter condition, we get
\begin{align}
    r(a^*) - r(a) \ge 
     \left( 1 - \exp\left\{ \theta_{t}(a^*) - \theta_{t}(a) \right\} \right) \cdot \left( r(a^*) - \pi_{\theta_{t}}^\top r \right).
\end{align}
After an update of the parameters, according to the ascent lemma for smooth function (\cref{lem:ascent_lemma_smooth_function}), 
$\pi_{\theta_{t+1}}^\top r \ge \pi_{\theta_{t}}^\top r$, i.e.,
\begin{align}
    0 < r(a^*) - \pi_{\theta_{t+1}}^\top r \le r(a^*) - \pi_{\theta_{t}}^\top r\,.
\end{align}
On the other hand,
\begin{align}
    \theta_{t+1}(a^*) - \theta_{t+1}(a) &= \theta_{t}(a^*) + \eta \cdot \frac{d \pi_{\theta_t}^\top r}{d \theta_t(a^*)} - \theta_{t}(a) - \eta \cdot \frac{d \pi_{\theta_t}^\top r}{d \theta_t(a)} \\
    &\ge \theta_{t}(a^*) - \theta_{t}(a),
\end{align}
which implies that
\begin{align}
    1 - \exp\left\{ \theta_{t+1}(a^*) - \theta_{t+1}(a) \right\} \le 1 - \exp\left\{ \theta_{t}(a^*) - \theta_{t}(a) \right\}.
\end{align}
Furthermore, 
by our assumption that $\pi_{\theta_t}(a^*) < \pi_{\theta_t}(a)$,
we have
$1 - \exp\left\{ \theta_{t}(a^*) - \theta_{t}(a) \right\} = 1 - \frac{\pi_{\theta_{t}}(a^*)}{\pi_{\theta_{t}}(a)} > 0$. Putting things together, we get
\begin{align}
    \left( 1 - \exp\left\{ \theta_{t+1}(a^*) - \theta_{t+1}(a) \right\} \right) \cdot \left( r(a^*) - \pi_{\theta_{t+1}}^\top r \right) &\le \left( 1 - \exp\left\{ \theta_{t}(a^*) - \theta_{t}(a) \right\} \right) \cdot \left( r(a^*) - \pi_{\theta_{t}}^\top r \right) \\
    &\le r(a^*) - r(a),
\end{align}
which is equivalent to
\begin{align}
     \left( 1 - \frac{\pi_{\theta_{t+1}}(a^*)}{\pi_{\theta_{t+1}}(a)} \right) \cdot \left( r(a^*) - \pi_{\theta_{t+1}}^\top r \right) \le r(a^*) - r(a),
\end{align}
and thus by our previous remark, $\theta_{t+1}\in \gR_1(a)$, thus, finishing the proof of part~(i).

Part~(ii): 
Assume again that $\theta_t \in \gR_1$. We want to show that
$\pi_{\theta_{t+1}}(a^*) \ge \pi_{\theta_{t}}(a^*)$. 
Since $\theta_t \in \gR_1$, we have $\frac{d \pi_{\theta_t}^\top r}{d \theta_t(a^*)} \ge \frac{d \pi_{\theta_t}^\top r}{d \theta_t(a)}, \ \forall a \not= a^*$. Hence,
\begin{align}
    \pi_{\theta_{t+1}}(a^*) &= \frac{\exp\left\{ \theta_{t+1}(a^*) \right\}}{ \sum_{a}{ \exp\left\{ \theta_{t+1}(a) \right\}} } \\
    &= \frac{\exp\left\{ \theta_{t}(a^*) + \eta \cdot \frac{d \pi_{\theta_t}^\top r}{d \theta_t(a^*)} \right\}}{ \sum_{a}{ \exp\left\{ \theta_{t}(a) + \eta \cdot \frac{d \pi_{\theta_t}^\top r}{d \theta_t(a)} \right\}} } \\
    &\ge \frac{\exp\left\{ \theta_{t}(a^*) + \eta \cdot \frac{d \pi_{\theta_t}^\top r}{d \theta_t(a^*)} \right\}}{ \sum_{a}{ \exp\left\{ \theta_{t}(a) + \eta \cdot \frac{d \pi_{\theta_t}^\top r}{d \theta_t(a^*)} \right\}} } \qquad\left(\text{using } \frac{d \pi_{\theta_t}^\top r}{d \theta_t(a^*)} \ge \frac{d \pi_{\theta_t}^\top r}{d \theta_t(a)} \right) \\
    &= \frac{\exp\left\{ \theta_{t}(a^*) \right\}}{ \sum_{a}{ \exp\left\{ \theta_{t}(a) \right\}} } = \pi_{\theta_t}(a^*).
\end{align}

\paragraph{Claim b)} 
We start by showing that $\gR_2\subset \gR_1$.
For this, let $\theta\in \gR_2$, i.e., $\pi_\theta(a^*) \ge \pi_\theta(a)$. Then,
\begin{align}
    \frac{d \pi_\theta^\top r}{d \theta(a^*)} &= \pi_\theta(a^*) \cdot \left( r(a^*) - \pi_\theta^\top r \right) \\
    &> \pi_\theta(a) \cdot \left( r(a) - \pi_\theta^\top r \right)  \qquad\left(\text{using } r(a^*) - \pi_\theta^\top r > 0 \text{ and } r(a^*) > r(a) \right) \\
    &= \frac{d \pi_\theta^\top r}{d \theta(a)}.
\end{align}
Hence, $\theta\in \gR_1$ and thus $\gR_2\subset \gR_1$ as desired.

Now, let us prove that $\gN_c \subset \gR_1$.
Take $\theta\in \gN_c$. We want to show that $\theta\in \gR_1$.
If $\theta\in \gR_2$, by $\gR_2\subset\gR_1$, we also have that $\theta\in \gR_1$.
Hence, it remains to show that $\theta\in \gR_1$ holds when $\theta \in \gN_c$ and $\theta\not\in \gR_2$.

Thus, take any $\theta$ that satisfies these two conditions. Pick $a\ne a^*$.
It suffices to show that $\theta\in \gR_1(a)$.
Without loss of generality, assume that $a^*=1$ and $a=2$.
Then, we have,
\begin{align}
\label{eq:r1i}
    \frac{d \pi_\theta^\top r}{d \theta(a^*)} - \frac{d \pi_\theta^\top r}{d \theta(a)} &= \frac{d \pi_\theta^\top r}{d \theta(1)} - \frac{d \pi_\theta^\top r}{d \theta(2)} \\
    &= \pi_\theta(1) \cdot \left( r(1) - \pi_\theta^\top r \right) - \pi_\theta(2) \cdot \left( r(2) - \pi_\theta^\top r \right) 
    \\
    &= 2 \cdot \pi_\theta(1) \cdot \left( r(1) - \pi_\theta^\top r \right) + \sum_{i=3}^{K}{ \pi_\theta(i) \cdot \left( r(i) - \pi_\theta^\top r \right) } 
    \qquad\left(\text{see below}\right)
    \\
    &= \left( 2 \cdot \pi_\theta(1) + \sum_{i=3}^{K}{ \pi_\theta(i)} \right) \cdot \left( r(1) - \pi_\theta^\top r \right) - \sum_{i=3}^{K}{ \pi_\theta(i) \cdot \left( r(1) - r(i) \right) } \\
    &\ge \left( 2 \cdot \pi_\theta(1) + \sum_{i=3}^{K}{ \pi_\theta(i)} \right) \cdot \left( r(1) - \pi_\theta^\top r \right) - \sum_{i=3}^{K}{ \pi_\theta(i) } \\
    &\ge \left( 2 \cdot \pi_\theta(1) + \sum_{i=3}^{K}{ \pi_\theta(i)} \right) \cdot \frac{\Delta}{K} - \sum_{i=3}^{K}{ \pi_\theta(i) },
\end{align}
where the second equation is because
\begin{align}
    \pi_\theta(2) \cdot \left( r(2) - \pi_\theta^\top r \right) + \sum_{i \not= 2}{ \pi_\theta(i) \cdot \left( r(i) - \pi_\theta^\top r \right) } = 0,
\end{align}
the first inequality is by $0 < r(1) - r(i) \le 1$ and the second inequality is because of
\begin{align}
    r(1) - \pi_\theta^\top r &= \sum_{i=1}^{K}{ \pi_\theta(i) \cdot r(1) }  - \sum_{i=1}^{K}{ \pi_\theta(i) \cdot r(i)} \\
    &= \sum_{i=2}^{K}{ \pi_\theta(i) \cdot \left( r(1) - r(i) \right) } \\
    &\ge \sum_{i=2}^{K}{ \pi_\theta(i) \cdot \Delta } \ge \max_{a \not= a^*}\{ \pi_\theta(a) \} \cdot \Delta \\
    &\ge \frac{\Delta}{K}. 
    \qquad\left(\text{using } \pi_\theta(a^*) < \max_{a \not= a^*}\{ \pi_\theta(a) \}, \ \max_{a \not= a^*}\{ \pi_\theta(a) \} = \max_{a}\{ \pi_\theta(a) \} \ge \frac{1}{K} \right)
\end{align}
Plugging $\sum_{i=3}^{K}{ \pi_\theta(i) } = 1 - \pi_\theta(1) - \pi_\theta(2)$
into \cref{eq:r1i} and rearranging the resulting expression we get
\begin{align}
    \frac{d \pi_\theta^\top r}{d \theta(a^*)} - \frac{d \pi_\theta^\top r}{d \theta(a)} 
    &\ge 
    \pi_\theta(1) \cdot  \left(1+\frac{\Delta}{K}\right) - \left(1-\frac{\Delta}{K}\right) + \pi_\theta(2)  \cdot  \left(1-\frac{\Delta}{K}\right) \\
    &\ge 
    \pi_\theta(2) \cdot  \left(1-\frac{\Delta}{K}\right) \ge 0\,,
    \qquad\left(\text{using } \theta\in \gN_c \text{, i.e., } \ \pi_\theta(1)\ge c/(c+1) \right)
\end{align}
which implies that $\theta \in \gR_1(a)$, thus, finishing the proof.

\paragraph{Claim c)}
We claim that  $\pi_{\theta_t}(a^*) \to 1$ as $t \to \infty$. 
For this, we wish to use the asymptotic convergence results of \citet[Theorem 5.1]{AgKaLeMa19},
which states this, but the stepsize there is $\eta\le 1/5$ while here we have $\eta=2/5$. 
We claim that their asymptotic result still hold with the larger $\eta$. In fact, the restriction on $\eta$ comes from that they can only prove the ascent lemma (\cref{lem:ascent_lemma_smooth_function}) for $\eta \le 1/5$. Other than this, their proof does not rely on the choice of $\eta$. Since we can prove the ascent lemma with $\eta\le 2/5$ (and in particular with $\eta=2/5$), their result continues to hold even with $\eta=2/5$.

Thus, $\pi_{\theta_t}(a^*) \to 1$ as $t \to \infty$. Hence, there exists $t_0 \ge 1$, such that $\pi_{\theta_{t_0}}(a^*) \ge \frac{c}{c+1}$ , which means $\theta_{t_0} \in \gN_c \subset \gR_1$. According to the first part in our proof, i.e., once $\theta_t$ is in $\gR_1$, following gradient update $\theta_{t+1}$ will be in $\gR_1$, and $\pi_{\theta_t}(a^*)$ is increasing in $\gR_1$, we have $\inf_{t}{ \pi_{\theta_t}(a^*)} = \min_{1 \le t \le t_0}{ \pi_{\theta_t}(a^*)}$. $t_0$ depends on initialization and $c$, which only depends on the problem.
\end{proof}

\textbf{\cref{prop:t_zero_softmax_special}.}
For any initialization there exist $t_0 \ge 1$ such that for any $t\ge t_0$, $t \mapsto \pi_{\theta_t}(a^*)$ is increasing. In particular, when $\pi_{\theta_1}$ is the uniform distribution, $t_0=1$.
\begin{proof}
We have
$t_0 =\min\{ t\ge 1: \pi_{\theta_t}(a^*) \ge \frac{c}{c+1} \}$, where $c = \frac{K}{2 \Delta} \cdot \left(1 - \frac{\Delta}{K} \right)$ in the proof for \cref{lem:lower_bound_cT_softmax_special} satisfies for any $t\ge t_0$, $t \mapsto \pi_{\theta_t}(a^*)$ is increasing. 

Now, let $\theta_1$ be so that $\pi_{\theta_1}$ is the uniform distribution.
We show that $t_0=1$. 
Recall from \cref{cl:regions} 
that $\gR_2$ is the region where the probability of the optimal action exceeds that of the suboptimal ones
and $\gR_1$ is the region where the gradient of the optimal action exceeds those of the suboptimal ones
and that $\gR_2\subset \gR_1$.
Clearly, $\theta_1\in \gR_2$ and hence also $\theta_1\in \gR_1$.
Now, by Part~a) of \cref{cl:regions}, $\gR_1$ is invariant under the updates, showing that $t_0=1$ holds as required.
\end{proof}

\textbf{\cref{thm:final_rates_softmax_special}} (Arbitrary initialization)\textbf{.}
Using \cref{update_rule:softmax_special} with $\eta = 2/5$, for all $t \ge 1$,
\begin{align}
    ( \pi^* - \pi_{\theta_t} )^\top r \le 5 / (c^2 \cdot t),
\end{align}
where $c = \inf_{t\ge 1} \pi_{\theta_t}(a^*) > 0$ is a constant that depends on $r$ and $\theta_1$, but it does not depend on the time $t$.
\begin{proof}
According to \cref{lem:pseudo_rates_softmax_special,lem:lower_bound_cT_softmax_special}, the claim immediately holds, with $c =  \inf_{t \ge 1}{ \pi_{\theta_t}(a^*)} > 0$.
\end{proof}

\textbf{\cref{thm:rates_uniform_softmax_special}} (Uniform initialization)\textbf{.}
Using \cref{update_rule:softmax_special} with $\eta = 2/5$ and $\pi_{\theta_1}(a) = 1/K$, $\forall a$, for all $t \ge 1$,
\begin{align}
    ( \pi^* - \pi_{\theta_t} )^\top r &\le 5 K^2 / t, \qquad \text{and} \\
    \sum_{t=1}^{T}{ ( \pi^* - \pi_{\theta_t} )^\top r} &\le \min\left\{ K \sqrt{5 T}, \ 5 K^2 \log{T} + 1 \right\}.
\end{align}
\begin{proof}
Since the initial policy is uniform policy, $\pi_{\theta_1}(a^*) \ge 1/K$. 
According to \cref{prop:t_zero_softmax_special}, for all $t \ge t_0 = 1$, 
$t \mapsto \pi_{\theta_t}(a^*)$ is increasing.
Hence, we have $\pi_{\theta_t}(a^*) \ge 1/K$, $\forall t \ge 1$, and $c_t = \min_{1 \le s \le t}{ \pi_{\theta_s}(a^*) } \ge 1/K$. 
According to \cref{lem:pseudo_rates_softmax_special},
\begin{align}
    ( \pi^* - \pi_{\theta_t} )^\top r \le \frac{5}{c_t^2} \cdot \frac{1}{t},
\end{align}
we have $( \pi^* - \pi_{\theta_t} )^\top r \le 5 K^2 / t$, $\forall t \ge 1$. The remaining results follow from \cref{eq:pseudo_regret} and $c_T \ge 1/K$.
\end{proof}

\textbf{\cref{lem:t0_softmax_special}.}
Let $r(1) > r(2) > r(3)$. Then, $a^*=1$ and $\inf_{t\ge 1} \pi_{\theta_t}(1) = \min_{1 \le t \le t_0} \pi_{\theta_t}(1)$, where 
\begin{align}
    t_0 = \min \left\{ t \ge 1 \,:\,  \frac{ \pi_{\theta_t}(1) }{ \pi_{\theta_t}(3) } \ge \frac{r(2) - r(3)}{2 \cdot ( r(1) - r(2) ) } \right\}\,.
\end{align}
In general, for $K$-action bandit cases, let $r(1) > r(2) > \cdots > r(K)$, we have,
\begin{align}
    t_0 = \min \left\{ t \ge 1 \,:\,  \pi_{\theta}(1) \ge \frac{ \sum_{j \not= 1, j \not= i}{ \pi_\theta(j) \cdot \left( r(i) - r(j) \right)} }{2 \cdot ( r(1) - r(i) ) }, \text{ for all } i \in \{ 2, 3, \dots K-1 \} \right\}\,.
\end{align}
\begin{proof}
\textbf{$3$-action case.}
Recall the definition of $\gR_1$ from the proof for \cref{lem:lower_bound_cT_softmax_special}:
\begin{align}
    \gR_1 = \left\{ \theta : \frac{d \pi_\theta^\top r}{d \theta(a^*)} \ge \frac{d \pi_\theta^\top r}{d \theta(a)}, \ \forall a \not= a^* \right\}.
\end{align}
By Part~a) of \cref{cl:regions}, it suffices to prove that $\theta\in \gR_1$.
Thus, our goal is to show that any $\theta$ such that $\frac{ \pi_{\theta}(1) }{ \pi_{\theta}(3) } \ge \frac{r(2) - r(3)}{2 \cdot ( r(1) - r(2) ) }$ is in fact an element of $\gR_1$. Suppose $\frac{ \pi_{\theta}(1) }{ \pi_{\theta}(3) } \ge \frac{r(2) - r(3)}{2 \cdot ( r(1) - r(2) ) }$. There are two cases.

Case (a): If $\frac{ \pi_{\theta}(1) }{ \pi_{\theta}(3) } \ge \frac{r(2) - r(3)}{ r(1) - r(2) }$, then we have,
\begin{align}
    r(2) - \pi_\theta^\top r &= - \pi_\theta(1) \cdot \left( r(1) - r(2) \right) + \pi_\theta(3) \cdot \left( r(2) - r(3) \right) \\
    &= \pi_\theta(3) \cdot ( r(1) - r(2) ) \cdot \left[ - \frac{ \pi_{\theta}(1) }{ \pi_{\theta}(3) } + \frac{r(2) - r(3)}{ r(1) - r(2) } \right] \\
    &\le 0, \qquad\left(\frac{ \pi_{\theta}(1) }{ \pi_{\theta}(3) } \ge \frac{r(2) - r(3)}{ r(1) - r(2) }\right)
\end{align}
which implies,
\begin{align}
    \frac{d \pi_\theta^\top r}{d \theta(1)} - \frac{d \pi_\theta^\top r}{d \theta(2)} &= \pi_\theta(1) \cdot \left( r(1) - \pi_\theta^\top r \right) - \pi_\theta(2) \cdot \left( r(2) - \pi_\theta^\top r \right) \\
    &\ge 0 - 0 = 0. \qquad\left(r(1) - \pi_\theta^\top r > 0 \right)
\end{align}
Note that since $r(1) > \pi_\theta^\top r$, and $r(3) < \pi_\theta^\top r$, we have
\begin{align}
\label{eq:t0_softmax_special_dtheta1_ge_dtheta3}
    \frac{d \pi_\theta^\top r}{d \theta(1)} - \frac{d \pi_\theta^\top r}{d \theta(3)} &= \pi_\theta(1) \cdot \left( r(1) - \pi_\theta^\top r \right) - \pi_\theta(3) \cdot \left( r(3) - \pi_\theta^\top r \right) \\
    &\ge 0 - 0 = 0.
\end{align}
Therefore we have $\frac{d \pi_\theta^\top r}{d \theta(1)} \ge \frac{d \pi_\theta^\top r}{d \theta(2)}$ and $\frac{d \pi_\theta^\top r}{d \theta(1)} \ge \frac{d \pi_\theta^\top r}{d \theta(3)}$, i.e., $\theta \in \gR_1$.

Case (b): If $\frac{r(2) - r(3)}{2 \cdot ( r(1) - r(2) ) } \le \frac{ \pi_{\theta}(1) }{ \pi_{\theta}(3) } < \frac{r(2) - r(3)}{ r(1) - r(2) }$, then we have,
\begin{align}
    \frac{d \pi_\theta^\top r}{d \theta(1)} - \frac{d \pi_\theta^\top r}{d \theta(2)} &= \pi_\theta(1) \cdot \left( r(1) - \pi_\theta^\top r \right) - \pi_\theta(2) \cdot \left( r(2) - \pi_\theta^\top r \right) \\
    &= 2 \cdot \pi_\theta(1) \cdot \left( r(1) - \pi_\theta^\top r \right) + \pi_\theta(3) \cdot \left( r(3) - \pi_\theta^\top r \right) \\
    &\ge \pi_\theta(3) \cdot \left[ \frac{r(2) - r(3)}{ r(1) - r(2) } \cdot \left( r(1) - \pi_\theta^\top r \right) + \left( r(3) - \pi_\theta^\top r \right)  \right] \qquad\left( \frac{ \pi_{\theta}(1) }{ \pi_{\theta}(3) } \ge \frac{r(2) - r(3)}{2 \cdot ( r(1) - r(2) ) } \right) \\
    &\ge \pi_\theta(3) \cdot \left[ \frac{r(2) - r(3)}{ r(1) - r(2) } \cdot \left( r(1) - r(2) \right) + \left( r(3) - \pi_\theta^\top r \right)  \right] \\
    &= \pi_\theta(3) \cdot \left( r(2) -  \pi_\theta^\top r \right) \ge 0,
\end{align}
where the second equation is according to
\begin{align}
    \pi_\theta(1) \cdot \left( r(1) - \pi_\theta^\top r \right) + \pi_\theta(2) \cdot \left( r(2) - \pi_\theta^\top r \right) + \pi_\theta(3) \cdot \left( r(3) - \pi_\theta^\top r \right) = \pi_\theta^\top r - \pi_\theta^\top r = 0,
\end{align}
and the second inequality is because of
\begin{align}
    r(1) - \pi_\theta^\top r &= \left(1 - \pi_\theta(1) \right) \cdot r(1) - \left( \pi_\theta(2) \cdot r(2) + \pi_\theta(3) \cdot r(3) \right) \\
    &= \pi_\theta(2) \cdot \left( r(1) - r(2) \right) + \pi_\theta(3) \cdot \left( r(1) - r(3) \right) \\
    &= \left( \pi_\theta(2) + \pi_\theta(3) \right) \cdot \left( r(1) - r(2) \right) + \pi_\theta(3) \cdot \left( r(2) - r(3) \right) \\
    &> \left( \pi_\theta(2) + \pi_\theta(3) \right) \cdot \left( r(1) - r(2) \right) + \pi_\theta(1) \cdot \left( r(1) - r(2) \right) \qquad\left( \frac{ \pi_{\theta}(1) }{ \pi_{\theta}(3) } < \frac{r(2) - r(3)}{ r(1) - r(2) } \right) \\
    &= r(1) - r(2),
\end{align}
and the last inequality is from
\begin{align}
    r(2) - \pi_\theta^\top r &= \pi_\theta(3) \cdot ( r(1) - r(2) ) \cdot \left[ - \frac{ \pi_{\theta}(1) }{ \pi_{\theta}(3) } + \frac{r(2) - r(3)}{ r(1) - r(2) } \right] \\
    &> 0. \qquad \left(\frac{ \pi_{\theta}(1) }{ \pi_{\theta}(3) } < \frac{r(2) - r(3)}{ r(1) - r(2) } \right)
\end{align}
Now we have $\frac{d \pi_\theta^\top r}{d \theta(1)} \ge \frac{d \pi_\theta^\top r}{d \theta(2)}$. According to \cref{eq:t0_softmax_special_dtheta1_ge_dtheta3}, we have $\frac{d \pi_\theta^\top r}{d \theta(1)} \ge \frac{d \pi_\theta^\top r}{d \theta(3)}$. Therefore we have $\theta \in \gR_1$.

\textbf{$K$-action case.} Suppose for each action $i \in \{ 2, 3, \dots K-1 \}$, $\pi_{\theta}(1) \ge \frac{ \sum_{j \not= 1, j \not= i}{ \pi_\theta(j) \cdot \left( r(i) - r(j) \right)} }{2 \cdot ( r(1) - r(i) ) }$. There are two cases.

Case (a): If $\pi_{\theta}(1) \ge \frac{ \sum_{j \not= 1, j \not= i}{ \pi_\theta(j) \cdot \left( r(i) - r(j) \right)} }{ r(1) - r(i) }$, then we have,
\begin{align}
    r(i) - \pi_\theta^\top r &= - \pi_\theta(1) \cdot \left( r(1) - r(i) \right) + \sum_{j \not= 1, j \not= i} \pi_\theta(j) \cdot \left( r(i) - r(j) \right) \\
    &\le 0, \qquad \left( \pi_{\theta}(1) \ge \frac{ \sum_{j \not= 1, j \not= i}{ \pi_\theta(j) \cdot \left( r(i) - r(j) \right)} }{ r(1) - r(i) } \right)
\end{align}
which implies, for all $i \in \{ 2, 3, \dots K-1 \}$,
\begin{align}
    \frac{d \pi_\theta^\top r}{d \theta(1)} - \frac{d \pi_\theta^\top r}{d \theta(i)} &= \pi_\theta(1) \cdot \left( r(1) - \pi_\theta^\top r \right) - \pi_\theta(i) \cdot \left( r(i) - \pi_\theta^\top r \right) \\
    &\ge 0 - 0 = 0. \qquad\left( r(1) - \pi_\theta^\top r > 0 \right)
\end{align}
Similar with \cref{eq:t0_softmax_special_dtheta1_ge_dtheta3}, since $r(1) > \pi_\theta^\top r$, and $r(K) < \pi_\theta^\top r$, we have
\begin{align}
\label{eq:t0_softmax_special_dtheta1_ge_dthetaK}
    \frac{d \pi_\theta^\top r}{d \theta(1)} - \frac{d \pi_\theta^\top r}{d \theta(K)} &= \pi_\theta(1) \cdot \left( r(1) - \pi_\theta^\top r \right) - \pi_\theta(K) \cdot \left( r(K) - \pi_\theta^\top r \right) \\
    &\ge 0 - 0 = 0.
\end{align}
Therefore we have $\frac{d \pi_\theta^\top r}{d \theta(1)} \ge \frac{d \pi_\theta^\top r}{d \theta(i)}$, for all $i \in \{ 2, 3, \dots K \}$, i.e., $\theta \in \gR_1$.

Case (b): If $\frac{ \sum_{j \not= 1, j \not= i}{ \pi_\theta(j) \cdot \left( r(i) - r(j) \right)} }{2 \cdot ( r(1) - r(i) ) } \le  \pi_{\theta}(1) < \frac{ \sum_{j \not= 1, j \not= i}{ \pi_\theta(j) \cdot \left( r(i) - r(j) \right)} }{ r(1) - r(i) }$, then we have, for all $i \in \{ 2, 3, \dots K-1 \}$,
\begin{align}
    \frac{d \pi_\theta^\top r}{d \theta(1)} - \frac{d \pi_\theta^\top r}{d \theta(i)} &= \pi_\theta(1) \cdot \left( r(1) - \pi_\theta^\top r \right) - \pi_\theta(i) \cdot \left( r(i) - \pi_\theta^\top r \right) \\
    &= 2 \cdot \pi_\theta(1) \cdot \left( r(1) - \pi_\theta^\top r \right) + \sum_{j \not= 1, j \not= i}{ \pi_\theta(j) \cdot \left( r(j) - \pi_\theta^\top r \right)} \\
    &\ge \frac{ \sum_{j \not= 1, j \not= i}{ \pi_\theta(j) \cdot \left( r(i) - r(j) \right)} }{ r(1) - r(i) } \cdot \left( r(1) - \pi_\theta^\top r \right) + \sum_{j \not= 1, j \not= i}{ \pi_\theta(j) \cdot \left( r(j) - \pi_\theta^\top r \right)} \\
    &\ge \frac{ \sum_{j \not= 1, j \not= i}{ \pi_\theta(j) \cdot \left( r(i) - r(j) \right)} }{ r(1) - r(i) } \cdot \left( r(1) - r(i) \right) + \sum_{j \not= 1, j \not= i}{ \pi_\theta(j) \cdot \left( r(j) - \pi_\theta^\top r \right)} \\
    &= \sum_{j \not= 1, j \not= i}{ \pi_\theta(j) \cdot \left( r(i) - \pi_\theta^\top r \right)} \ge 0,
\end{align}
where the second equation is according to
\begin{align}
    \pi_\theta(1) \cdot \left( r(1) - \pi_\theta^\top r \right) + \pi_\theta(i) \cdot \left( r(i) - \pi_\theta^\top r \right) + \sum_{j \not= 1, j \not= i}{ \pi_\theta(j) \cdot \left( r(j) - \pi_\theta^\top r \right)} = \pi_\theta^\top r - \pi_\theta^\top r = 0,
\end{align}
and the first inequality is by $r(1) - \pi_\theta^\top r > 0$ and,
\begin{align}
    \pi_{\theta}(1) \ge \frac{ \sum_{j \not= 1, j \not= i}{ \pi_\theta(j) \cdot \left( r(i) - r(j) \right)} }{2 \cdot ( r(1) - r(i) ) },
\end{align}
and the second inequality is because of
\begin{align}
\MoveEqLeft
    r(1) - \pi_\theta^\top r = \pi_\theta(i) \cdot \left( r(1) - r(i) \right) + \sum_{j \not= 1, j \not= i}{ \pi_\theta(j) \cdot \left( r(1) - r(j) \right)} \\
    &= \sum_{j \not= 1}{ \pi_\theta(j) \cdot \left( r(1) - r(i) \right)} + \sum_{j \not= 1, j \not= i}{ \pi_\theta(j) \cdot \left( r(i) - r(j) \right)} \\
    &> \sum_{j \not= 1}{ \pi_\theta(j) \cdot \left( r(1) - r(i) \right)} + \pi_\theta(1) \cdot \left( r(1) - r(i) \right) \qquad \left( \frac{ \sum_{j \not= 1, j \not= i}{ \pi_\theta(j) \cdot \left( r(i) - r(j) \right)} }{ r(1) - r(i) } > \pi_{\theta}(1) \right) \\
    &= r(1) - r(i),
\end{align}
and the last inequality is from $ \frac{ \sum_{j \not= 1, j \not= i}{ \pi_\theta(j) \cdot \left( r(i) - r(j) \right)} }{ r(1) - r(i) } > \pi_{\theta}(1) > 0$ and,
\begin{align}
    r(i) - \pi_\theta^\top r &= - \pi_\theta(1) \cdot \left( r(1) - r(i) \right) + \sum_{j \not= 1, j \not= i} \pi_\theta(j) \cdot \left( r(i) - r(j) \right) \\
    &> 0. \qquad \left( \pi_{\theta}(1) < \frac{ \sum_{j \not= 1, j \not= i}{ \pi_\theta(j) \cdot \left( r(i) - r(j) \right)} }{ r(1) - r(i) } \right)
\end{align}
Now we have $\frac{d \pi_\theta^\top r}{d \theta(1)} \ge \frac{d \pi_\theta^\top r}{d \theta(i)}$, for all $i \in \{ 2, 3, \dots K-1 \}$. According to \cref{eq:t0_softmax_special_dtheta1_ge_dthetaK}, we have $\frac{d \pi_\theta^\top r}{d \theta(1)} \ge \frac{d \pi_\theta^\top r}{d \theta(K)}$. Therefore we have $\theta \in \gR_1$.
\end{proof}

\subsubsection{Proofs for softmax parametrization in MDPs}
\label{sec:proofs_policy_gradient_general_mdps}

\textbf{\cref{lem:smoothness_softmax_general}} (Smoothness)\textbf{.}
$V^{\pi_\theta}(\rho)$ is $8 / (1 - \gamma)^3$-smooth.
\begin{proof}
See \citet[Lemma E.4]{AgKaLeMa19}. Our proof is for completeness. Denote $\theta_\alpha = \theta + \alpha u$, where $\alpha \in \sR$ and $u \in \sR^{SA}$. For any $s \in \gS$,
\begin{align}
    \sum_{a}{\left| \frac{\partial \pi_{\theta_\alpha}(a | s)}{\partial \alpha} \Big|_{\alpha=0} \right|} &= \sum_{a}{\left| \Big\langle \frac{\partial \pi_{\theta_\alpha}(a | s)}{\partial \theta_\alpha} \Big|_{\alpha=0}, \frac{\partial \theta_\alpha}{\partial \alpha} \Big\rangle \right|} \\
    &= \sum_{a}{\left| \Big\langle \frac{\partial \pi_{\theta}(a | s)}{\partial \theta}, u \Big\rangle \right|}.
\end{align}
Since $\frac{\partial \pi_\theta(a | s)}{\partial \theta(s^\prime, \cdot)} = 0$, for $s^\prime \not= s$,
\begin{align}
\label{eq:smoothness_softmax_general_intermediate_pi_first_derivative_upper_bound}
    \sum_{a}{\left| \frac{\partial \pi_{\theta_\alpha}(a | s)}{\partial \alpha} \Big|_{\alpha=0} \right|} &= \sum_{a}{\left| \Big\langle \frac{\partial \pi_{\theta}(a | s)}{\partial \theta(s, \cdot)}, u(s, \cdot) \Big\rangle \right|} \\
    &= \sum_{a}{\pi_{\theta}(a | s) \cdot \left| u(s,a) - \pi_{\theta}(\cdot | s)^\top u(s, \cdot) \right|} \\
    &\le \max_{a}{ | u(s, a) | + | \pi_{\theta}(\cdot | s)^\top u(s, \cdot) | } \le 2 \cdot \| u \|_2.
\end{align}
Similarly,
\begin{align}
    \sum_{a}{\left| \frac{\partial^2 \pi_{\theta_\alpha}(a | s)}{\partial \alpha^2} \Big|_{\alpha=0} \right|} &= \sum_{a}{\left| \Big\langle \frac{\partial}{\partial \theta_\alpha} \left\{ \frac{\partial \pi_{\theta_\alpha}(a|s)}{\partial \alpha} \right\} \Big|_{\alpha=0}, \frac{\partial \theta_\alpha}{\partial \alpha} \Big\rangle \right|} \\
    &= \sum_{a}{\left| \Big\langle \frac{\partial^2 \pi_{\theta_\alpha}(a|s)}{\partial \theta_\alpha^2} \Big|_{\alpha=0} \frac{\partial \theta_\alpha}{\partial \alpha}, \frac{\partial \theta_\alpha}{\partial \alpha} \Big\rangle \right|} \\
    &= \sum_{a}{\left| \Big\langle \frac{\partial^2 \pi_{\theta}(a | s)}{\partial \theta^2(s, \cdot)} u(s, \cdot), u(s, \cdot) \Big\rangle \right|}.
\end{align}
Let $S(a, \theta) = \frac{\partial^2 \pi_{\theta}(a | s)}{\partial \theta^2(s, \cdot)} \in \sR^{A \times A}$. $\forall i, j \in [A]$, the value of $S(a, \theta)$ is,
\begin{align}
    S_{i,j} &= \frac{\partial \{ \delta_{ia} \pi_{\theta}(a | s) - \pi_{\theta}(a | s) \pi_{\theta}(i | s) \}}{\partial \theta(s,j)} \\
    &= \delta_{ia} \cdot \left[ \delta_{ja} \pi_{\theta}(a | s) -  \pi_{\theta}(a | s) \pi_{\theta}(j | s) \right] - \pi_{\theta}(a | s) \cdot \left[\delta_{ij} \pi_{\theta}(j | s) - \pi_{\theta}(i | s) \pi_{\theta}(j | s) \right] - \pi_{\theta}(i | s) \cdot \left[ \delta_{ja} \pi_{\theta}(a | s) - \pi_{\theta}(a | s) \pi_{\theta}(j | s) \right],
\end{align}
where the $\delta$ notation is as defined in \cref{eq:delta_ij_notation}. Then we have,
\begin{align}
\MoveEqLeft
    \left| \Big\langle \frac{\partial^2 \pi_{\theta}(a | s)}{\partial \theta^2(s, \cdot)} u(s, \cdot), u(s, \cdot) \Big\rangle \right| = \left| \sum_{i=1}^{A} \sum_{j=1}^{A} { S_{i,j} u(s,i) u(s,j) } \right| \\
    &= \pi_{\theta}(a | s) \cdot \left| u(s,a)^2 - 2 \cdot u(s,a) \cdot \pi_{\theta}(\cdot | s)^\top u(s, \cdot) - \pi_{\theta}(\cdot | s)^\top \left( u(s, \cdot) \odot u(s, \cdot) \right) + 2 \cdot \left( \pi_{\theta}(\cdot | s)^\top u(s, \cdot) \right)^2 \right|.
\end{align}
Therefore we have,
\begin{align}
\label{eq:smoothness_softmax_general_intermediate_pi_second_derivative_upper_bound}
    \sum_{a}{\left| \frac{\partial^2 \pi_{\theta_\alpha}(a | s)}{\partial \alpha^2} \Big|_{\alpha=0} \right|} &\le \max_{a}{ \left\{ u(s,a)^2 + 2 \cdot \left| u(s,a) \cdot \pi_{\theta}(\cdot | s)^\top u(s, \cdot) \right| \right\} +  \pi_{\theta}(\cdot | s)^\top \left( u(s, \cdot) \odot u(s, \cdot) \right) + 2 \cdot \left( \pi_{\theta}(\cdot | s)^\top u(s, \cdot) \right)^2 } \\
    &\le \| u(s, \cdot) \|_2^2 + 2 \cdot \| u(s, \cdot) \|_2^2 + \| u(s, \cdot) \|_2^2 + 2 \cdot \| u(s, \cdot) \|_2^2 \le 6 \cdot \| u \|_2^2.
\end{align}
Define $P(\alpha) \in \sR^{S \times S}$, where $\forall ( s, s^\prime )$,
\begin{align}
\label{eq:smoothness_softmax_general_intermediate_P_PI_def}
   \left[ P(\alpha) \right]_{( s, s^\prime )} = \sum_{a}{ \pi_{\theta_\alpha}(a | s) \cdot \gP(s^\prime | s, a) }.
\end{align}
The derivative w.r.t. $\alpha$ is
\begin{align}
    \left[ \frac{\partial P(\alpha)}{\partial \alpha} \Big|_{\alpha=0} \right]_{( s, s^\prime)} = \sum_{a} \left[ \frac{ \partial \pi_{\theta_\alpha}(a | s) }{\partial \alpha} \Big|_{\alpha=0} \right] \cdot \gP(s^\prime | s, a).
\end{align}
For any vector $x \in \sR^{S}$, we have
\begin{align}
    \left[ \frac{\partial P(\alpha)}{\partial \alpha} \Big|_{\alpha=0} x \right]_{(s)} &= \sum_{s^\prime} \sum_{a}{ \left[ \frac{ \partial \pi_{\theta_\alpha}(a | s) }{\partial \alpha} \Big|_{\alpha=0} \right] \cdot \gP(s^\prime | s, a) \cdot x(s^\prime) }.
\end{align}
The $\ell_\infty$ norm is upper bounded as
\begin{align}
\label{eq:smoothness_softmax_general_intermediate_P_PI_first}
    \left\|  \frac{\partial P(\alpha)}{\partial \alpha} \Big|_{\alpha=0} x \right\|_\infty &= \max_{s}{ \left| \sum_{s^\prime} \sum_{a}{ \left[ \frac{ \partial \pi_{\theta_\alpha}(a | s) }{\partial \alpha} \Big|_{\alpha=0} \right] \cdot \gP(s^\prime | s, a) \cdot x(s^\prime) } \right| } \\
    &\le \max_{s}{ \sum_{a}{ \sum_{s^\prime} { \gP(s^\prime | s, a) \cdot \left| \frac{ \partial \pi_{\theta_\alpha}(a | s) }{\partial \alpha} \Big|_{\alpha=0} \right| } } \cdot \| x \|_\infty } \\
    &= \max_{s} \sum_{a}{ \left| \frac{ \partial \pi_{\theta_\alpha}(a | s) }{\partial \alpha} \Big|_{\alpha=0} \right|  } \cdot \| x \|_\infty \\
    &\le 2 \cdot \| u \|_2 \cdot \| x \|_\infty. \qquad
    \left( \text{by \cref{eq:smoothness_softmax_general_intermediate_pi_first_derivative_upper_bound}} \right)
\end{align}
Similarly, taking second derivative w.r.t. $\alpha$,
\begin{align}
    \left[ \frac{\partial^2 P(\alpha)}{\partial \alpha^2} \Big|_{\alpha=0} \right]_{( s, s^\prime )} = \sum_{a} \left[ \frac{ \partial^2 \pi_{\theta_\alpha}(a | s) }{\partial \alpha^2} \Big|_{\alpha=0} \right] \cdot \gP(s^\prime | s, a).
\end{align}
The $\ell_\infty$ norm is upper bounded as
\begin{align}
\label{eq:smoothness_softmax_general_intermediate_P_PI_second}
    \left\|  \frac{\partial^2 P(\alpha)}{\partial \alpha^2} \Big|_{\alpha=0} x \right\|_\infty &= \max_{s}{ \left| \sum_{s^\prime} \sum_{a}{ \left[ \frac{ \partial^2 \pi_{\theta_\alpha}(a | s) }{\partial \alpha^2} \Big|_{\alpha=0} \right] \cdot \gP(s^\prime | s, a) \cdot x(s^\prime) } \right| } \\
    &\le \max_{s}{ \sum_{a}{ \sum_{s^\prime} { \gP(s^\prime | s, a) \cdot \left| \frac{ \partial^2 \pi_{\theta_\alpha}(a | s) }{\partial \alpha^2} \Big|_{\alpha=0} \right| } } \cdot \| x \|_\infty } \\
    &= \max_{s} \sum_{a}{ \left| \frac{ \partial^2 \pi_{\theta_\alpha}(a | s) }{\partial \alpha^2} \Big|_{\alpha=0} \right|  } \cdot \| x \|_\infty \\
    &\le 6 \cdot \| u \|_2^2 \cdot \| x \|_\infty. \qquad \left(
    \text{by \cref{eq:smoothness_softmax_general_intermediate_pi_second_derivative_upper_bound}} \right)
\end{align}
Next, consider the state value function of $\pi_{\theta_\alpha}$,
\begin{align}
    V^{\pi_{\theta_\alpha}}(s) &= \sum_{a}{\pi_{\theta_\alpha}(a | s) \cdot r(s, a) } + \gamma \sum_{a}{ \pi_{\theta_\alpha}(a | s) \sum_{s^\prime}{ \gP(s^\prime | s, a) \cdot V^{\pi_{\theta_\alpha}}(s^\prime) }  },
\end{align}
which implies,
\begin{align}
\label{eq:state_value_bellman_equation}
    V^{\pi_{\theta_\alpha}}(s) &= e_{s}^\top M(\alpha) r_{\theta_\alpha},
\end{align}
where
\begin{align}
\label{eq:smoothness_softmax_general_intermediate_M_matrix_def}
    M(\alpha) = \left( \identitymatrix - \gamma P(\alpha) \right)^{-1},
\end{align}
and $r_{\theta_\alpha} \in \sR^{S}$ for $s\in \gS$ is given by
\begin{align}
    r_{\theta_\alpha}(s) = \sum_{a}{\pi_{\theta_\alpha}(a | s) \cdot r(s, a)}.
\end{align}
Since $\left[ P(\alpha) \right]_{(s,s^\prime)} \ge 0$, $\forall (s,s^\prime)$, and
\begin{align}
    M(\alpha) = \left( \identitymatrix - \gamma P(\alpha) \right)^{-1} = \sum_{t=0}^\infty{\gamma^t \left[ P(\alpha) \right]^t},
\end{align}
we have $\left[ M(\alpha) \right]_{(s,s^\prime)} \ge 0$, $\forall (s,s^\prime)$. Denote $\left[ M(\alpha) \right]_{i,:}$ as the $i$-th row vector of $M(\alpha)$. We have
\begin{align}
    \rvone = \frac{1}{1 - \gamma} \cdot \left( \identitymatrix - \gamma P(\alpha) \right) \rvone \Longrightarrow{} M(\alpha) \rvone = \frac{1}{1 - \gamma} \cdot \rvone,
\end{align}
which implies, $\forall i$,
\begin{align}
    \left\| \left[ M(\alpha) \right]_{i,:} \right\|_1 = \sum_{j}{ \left[ M(\alpha) \right]_{(i,j)} } = \frac{1}{1 - \gamma}.
\end{align}
Therefore, for any vector $x \in \sR^{S}$,
\begin{align}
\label{eq:smoothness_softmax_general_intermediate_M_matrix_norm}
    \left\| M(\alpha) x \right\|_\infty &= \max_{i}{ \left| \left[ M(\alpha) \right]_{i,:}^\top x \right| } \\
    &\le \max_{i}{ \left\| \left[ M(\alpha) \right]_{i,:} \right\|_1 \cdot \| x \|_\infty} \\
    &= \frac{1}{1 - \gamma} \cdot \| x \|_\infty.
\end{align}
According to \cref{asmp:bounded_reward}, $r(s,a) \in [0,1]$, $\forall (s,a)$. We have,
\begin{align}
\label{eq:smoothness_softmax_general_intermediate_rpi_upper_bound}
    \left\| r_{\theta_\alpha} \right\|_\infty = \max_{s}{ \left| r_{\theta_\alpha}(s) \right| } = \max_{s}{ \left| \sum_{a}{\pi_{\theta_\alpha}(a | s) \cdot r(s, a)} \right| } \le 1.
\end{align}
Since $\frac{\partial \pi_\theta(a | s)}{\partial \theta(s^\prime, \cdot)} = 0$, for $s^\prime \not= s$,
\begin{align}
    \left| \frac{\partial r_{\theta_\alpha}(s)}{\partial \alpha} \right| &= \left| \left( \frac{\partial r_{\theta_\alpha}(s)}{\partial \theta_\alpha} \right)^\top \frac{\partial \theta_\alpha}{\partial \alpha} \right| \\
    &= \left| \left( \frac{\partial \{ \pi_{\theta_\alpha}(\cdot | s)^\top r(s, \cdot )\} }{\partial \theta_\alpha(s, \cdot)} \right)^\top u(s, \cdot) \right| \\
    &= \left| \left( H\left( \pi_{\theta_\alpha}(\cdot | s) \right) r(s, \cdot ) \right)^\top u(s, \cdot) \right| \\
    &\le \left\| H\left( \pi_{\theta_\alpha}(\cdot | s) \right) r(s, \cdot ) \right\|_1 \cdot \left\| u(s, \cdot) \right\|_\infty.
\end{align}
Similarly to \cref{eq:H_matrix_r_1_norm_upper_bound_special}, the $\ell_1$ norm is upper bounded as
\begin{align}
    \left\| H\left( \pi_{\theta_\alpha}(\cdot | s) \right) r(s, \cdot ) \right\|_1 &= \sum_{a}{ \pi_{\theta_\alpha}(a | s) \cdot \left| r(s, a) - \pi_{\theta_\alpha}(\cdot | s)^\top r(s, \cdot ) \right| } \\
    &\le \max_{a}{ \left| r(s, a) - \pi_{\theta_\alpha}(\cdot | s)^\top r(s, \cdot ) \right| } \\
    &\le 1. \qquad\left(\text{since } r(s,a) \in [0, 1] \right)
\end{align}
Therefore we have,
\begin{align}
\label{eq:smoothness_softmax_general_intermediate_rpi_first_derivative_upper_bound}
    \left\| \frac{\partial r_{\theta_\alpha}}{\partial \alpha} \right\|_\infty &= \max_{s}{ \left| \frac{\partial r_{\theta_\alpha}(s)}{\partial \alpha} \right| } \\
    &\le \max_{s}{ \left\| H\left( \pi_{\theta_\alpha}(\cdot | s) \right) r(s, \cdot ) \right\|_1 \cdot \left\| u(s, \cdot) \right\|_\infty } \\
    &\le \| u \|_2.
\end{align}
Similarly,
\begin{align}
\label{eq:smoothness_softmax_general_intermediate_rpi_second_derivative_upper_bound}
    \left\| \frac{\partial^2 r_{\theta_\alpha}}{\partial \alpha^2} \right\|_\infty &= \max_{s}{ \left| \frac{\partial^2 r_{\theta_\alpha}(s)}{\partial \alpha^2} \right| } \\
    &= \max_{s}{ \left| \left( \frac{\partial }{\partial \theta_\alpha}\left\{ \frac{\partial r_{\theta_\alpha}(s)}{\partial \alpha} \right\} \right)^\top \frac{\partial \theta_\alpha}{\partial \alpha} \right| } \\
    &= \max_{s}{ \left| \left( \frac{\partial^2 r_{\theta_\alpha}(s) }{\partial \theta_\alpha^2} \frac{\partial \theta_\alpha}{\partial \alpha} \right)^\top \frac{\partial \theta_\alpha}{\partial \alpha} \right| } \\
    &= \max_{s}\left| u(s, \cdot)^\top \frac{\partial^2 \{ \pi_{\theta_\alpha}(\cdot | s)^\top r(s, \cdot )\} }{\partial \theta_\alpha(s, \cdot)^2} u(s, \cdot) \right| \\
    &\le 5/2 \cdot \| u(s, \cdot ) \|_2^2 \le 3 \cdot \| u \|_2^2. 
	\qquad \left( \text{by \cref{eq:H_matrix_maximum_eigenvalue}} \right)
\end{align}
Taking derivative w.r.t. $\alpha$ in \cref{eq:state_value_bellman_equation},
\begin{align}
    \frac{\partial V^{\pi_{\theta_\alpha}}(s)}{\partial \alpha} = \gamma \cdot e_{s}^\top M(\alpha) \frac{\partial P(\alpha)}{\partial \alpha} M(\alpha) r_{\theta_\alpha} + e_{s}^\top M(\alpha) \frac{\partial r_{\theta_\alpha}}{\partial \alpha}.
\end{align}
Taking second derivative w.r.t. $\alpha$,
\begin{align}
\label{eq:smoothness_softmax_general_intermediate_V_second_derivative_def}
    \frac{\partial^2 V^{\pi_{\theta_\alpha}}(s)}{\partial \alpha^2} &= 2 \gamma^2 \cdot e_{s}^\top M(\alpha) \frac{\partial P(\alpha)}{\partial \alpha} M(\alpha) \frac{\partial P(\alpha)}{\partial \alpha} M(\alpha) r_{\theta_\alpha} + \gamma \cdot e_{s}^\top M(\alpha) \frac{\partial^2 P(\alpha)}{\partial \alpha^2} M(\alpha) r_{\theta_\alpha} \\
    &\qquad + 2 \gamma \cdot e_{s}^\top M(\alpha) \frac{\partial P(\alpha)}{\partial \alpha} M(\alpha) \frac{\partial r_{\theta_\alpha}}{\partial \alpha} + e_{s}^\top M(\alpha) \frac{\partial^2 r_{\theta_\alpha}}{\partial \alpha^2}.
\end{align}
For the last term,
\begin{align}
\label{eq:smoothness_softmax_general_intermediate_1}
    \left| e_{s}^\top M(\alpha) \frac{\partial^2 r_{\theta_\alpha}}{\partial \alpha^2} \Big|_{\alpha=0} \right| &\le \left\| e_{s} \right\|_1 \cdot \left\| M(\alpha) \frac{\partial^2 r_{\theta_\alpha}}{\partial \alpha^2} \Big|_{\alpha=0} \right\|_\infty \\
    &\le \frac{1}{1 - \gamma} \cdot \left\| \frac{\partial^2 r_{\theta_\alpha}}{\partial \alpha^2} \Big|_{\alpha=0} \right\|_\infty 
    \qquad \left(\text{by \cref{eq:smoothness_softmax_general_intermediate_M_matrix_norm}} \right) \\
    &\le \frac{3}{1 - \gamma} \cdot \| u \|_2^2. 
    \qquad \left( \text{by \cref{eq:smoothness_softmax_general_intermediate_rpi_second_derivative_upper_bound}} \right)
\end{align}
For the second last term,
\begin{align}
\label{eq:smoothness_softmax_general_intermediate_2}
    \left| e_{s}^\top M(\alpha) \frac{\partial P(\alpha)}{\partial \alpha} M(\alpha) \frac{\partial r_{\theta_\alpha}}{\partial \alpha} \Big|_{\alpha=0} \right| &\le \left\| M(\alpha) \frac{\partial P(\alpha)}{\partial \alpha} M(\alpha) \frac{\partial r_{\theta_\alpha}}{\partial \alpha} \Big|_{\alpha=0} \right\|_\infty \\
    &\le \frac{1}{1 - \gamma} \cdot \left\| \frac{\partial P(\alpha)}{\partial \alpha} M(\alpha) \frac{\partial r_{\theta_\alpha}}{\partial \alpha} \Big|_{\alpha=0} \right\|_\infty \qquad \left( \text{by \cref{eq:smoothness_softmax_general_intermediate_M_matrix_norm}} \right) \\
    &\le \frac{2 \cdot \| u \|_2 }{1 - \gamma} \cdot \left\| M(\alpha) \frac{\partial r_{\theta_\alpha}}{\partial \alpha} \Big|_{\alpha=0} \right\|_\infty 
    \qquad \left( \text{by \cref{eq:smoothness_softmax_general_intermediate_P_PI_first}} \right) \\
    &\le \frac{2 \cdot \| u \|_2 }{(1 - \gamma)^2} \cdot \left\|  \frac{\partial r_{\theta_\alpha}}{\partial \alpha} \Big|_{\alpha=0} \right\|_\infty \qquad \left( \text{by \cref{eq:smoothness_softmax_general_intermediate_M_matrix_norm}} \right) \\
    &\le \frac{2 \cdot \| u \|_2 }{(1 - \gamma)^2} \cdot \| u \|_2 = \frac{2  }{(1 - \gamma)^2} \cdot \| u \|_2^2. 
    \qquad \left( \text{by \cref{eq:smoothness_softmax_general_intermediate_rpi_first_derivative_upper_bound}} \right) 
\end{align}
For the second term,
\begin{align}
\label{eq:smoothness_softmax_general_intermediate_3}
    \left| e_{s}^\top M(\alpha) \frac{\partial^2 P(\alpha)}{\partial \alpha^2} M(\alpha) r_{\theta_\alpha} \Big|_{\alpha=0} \right| &\le \left\| M(\alpha) \frac{\partial^2 P(\alpha)}{\partial \alpha^2} M(\alpha) r_{\theta_\alpha} \Big|_{\alpha=0} \right\|_\infty \\
    &\le \frac{1}{1 - \gamma} \cdot \left\|  \frac{\partial^2 P(\alpha)}{\partial \alpha^2} M(\alpha) r_{\theta_\alpha} \Big|_{\alpha=0} \right\|_\infty \qquad \left( \text{by \cref{eq:smoothness_softmax_general_intermediate_M_matrix_norm}} \right) \\
    &\le \frac{6 \cdot \| u \|_2^2}{1 - \gamma} \cdot \left\| M(\alpha) r_{\theta_\alpha} \Big|_{\alpha=0} \right\|_\infty \qquad \left( \text{by \cref{eq:smoothness_softmax_general_intermediate_P_PI_second}} \right) \\
    &\le \frac{6 \cdot \| u \|_2^2}{(1 - \gamma)^2} \cdot \left\|  r_{\theta_\alpha} \Big|_{\alpha=0} \right\|_\infty \qquad \left( \text{by \cref{eq:smoothness_softmax_general_intermediate_M_matrix_norm}} \right) \\
    &\le \frac{6 }{(1 - \gamma)^2} \cdot \| u \|_2^2. \qquad \left( \text{by \cref{eq:smoothness_softmax_general_intermediate_rpi_upper_bound}} \right) 
\end{align}
For the first term, according to \cref{eq:smoothness_softmax_general_intermediate_P_PI_first}, \cref{eq:smoothness_softmax_general_intermediate_M_matrix_norm,eq:smoothness_softmax_general_intermediate_rpi_upper_bound},
\begin{align}
\label{eq:smoothness_softmax_general_intermediate_4}
    \left| e_{s}^\top M(\alpha) \frac{\partial P(\alpha)}{\partial \alpha} M(\alpha) \frac{\partial P(\alpha)}{\partial \alpha} M(\alpha) r_{\theta_\alpha} \Big|_{\alpha=0} \right| &\le \left\| M(\alpha) \frac{\partial P(\alpha)}{\partial \alpha} M(\alpha) \frac{\partial P(\alpha)}{\partial \alpha} M(\alpha) r_{\theta_\alpha} \Big|_{\alpha=0} \right\|_\infty \\ &\le \frac{1}{1-\gamma} \cdot 2 \cdot \| u \|_2 \cdot \frac{1}{1-\gamma} \cdot 2 \cdot \| u \|_2 \cdot \frac{1}{1-\gamma} \cdot 1 \\
    &= \frac{4  }{(1 - \gamma)^3} \cdot \| u \|_2^2.
\end{align}
Combining \cref{eq:smoothness_softmax_general_intermediate_1,eq:smoothness_softmax_general_intermediate_2,eq:smoothness_softmax_general_intermediate_3,eq:smoothness_softmax_general_intermediate_4} with \cref{eq:smoothness_softmax_general_intermediate_V_second_derivative_def},
\begin{align}
\label{eq:smoothness_softmax_general_intermediate_5}
    \left| \frac{\partial^2 V^{\pi_{\theta_\alpha}}(s)}{\partial \alpha^2} \Big|_{\alpha=0} \right| &\le 2 \gamma^2 \cdot \left| e_{s}^\top M(\alpha) \frac{\partial P(\alpha)}{\partial \alpha} M(\alpha) \frac{\partial P(\alpha)}{\partial \alpha} M(\alpha) r_{\theta_\alpha} \Big|_{\alpha=0} \right| + \gamma \cdot \left| e_{s}^\top M(\alpha) \frac{\partial^2 P(\alpha)}{\partial \alpha^2} M(\alpha) r_{\theta_\alpha} \Big|_{\alpha=0} \right| \\
    &\qquad + 2 \gamma \cdot \left| e_{s}^\top M(\alpha) \frac{\partial P(\alpha)}{\partial \alpha} M(\alpha) \frac{\partial r_{\theta_\alpha}}{\partial \alpha} \Big|_{\alpha=0} \right| + \left| e_{s}^\top M(\alpha) \frac{\partial^2 r_{\theta_\alpha}}{\partial \alpha^2} \Big|_{\alpha=0} \right| \\
    &\le \left( 2 \gamma^2 \cdot \frac{4}{(1-\gamma)^3} + \gamma \cdot \frac{6}{(1-\gamma)^2} + 2 \gamma \cdot \frac{2}{(1-\gamma)^2} + \frac{3}{1-\gamma}  \right) \cdot \| u \|_2^2 \\
    &\le \frac{8}{(1-\gamma)^3} \cdot \| u \|_2^2,
\end{align}
which implies for all $y \in \sR^{S A}$ and $\theta$,
\begin{align}
\label{eq:smoothness_softmax_general_intermediate_6}
    \left| y^\top \frac{\partial^2 V^{\pi_\theta}(s)}{\partial \theta^2} y \right| &= \left| \left(\frac{y}{ \| y \|_2 } \right)^\top \frac{\partial^2 V^{\pi_\theta}(s)}{\partial \theta^2} \left(\frac{y}{ \| y \|_2 } \right) \right| \cdot \| y \|_2^2 \\
    &\le \max_{\| u \|_2 = 1}{ \left| \Big\langle \frac{\partial^2 V^{\pi_\theta}(s)}{\partial \theta^2} u , u \Big\rangle \right| } \cdot \| y \|_2^2 \\
    &= \max_{\| u \|_2 = 1}{ \left| \Big\langle\frac{\partial^2 V^{\pi_{\theta_\alpha}}(s)}{\partial {\theta_\alpha^2}} \Big|_{\alpha=0} \frac{\partial \theta_\alpha}{\partial \alpha}, \frac{\partial \theta_\alpha}{\partial \alpha} \Big\rangle \right| } \cdot \| y \|_2^2 \\
    &= \max_{\| u \|_2 = 1}{ \left| \Big\langle \frac{\partial}{\partial \theta_\alpha} \left\{ \frac{\partial V^{\pi_{\theta_\alpha}}(s)}{\partial \alpha} \right\} \Big|_{\alpha = 0}, \frac{\partial \theta_\alpha}{\partial \alpha} \Big\rangle \right| } \cdot \| y \|_2^2 \\
    &= \max_{\| u \|_2 = 1}{ \left| \frac{\partial^2 V^{\pi_{\theta_\alpha}}(s)}{\partial \alpha^2  } \Big|_{\alpha=0} \right| } \cdot \| y \|_2^2 \\
    &\le \frac{8}{(1 - \gamma)^3} \cdot \| y \|_2^2. \qquad \left(
    \text{by \cref{eq:smoothness_softmax_general_intermediate_5}} \right)
\end{align}
Denote $\theta_{\xi} = \theta + \xi ( \theta^\prime - \theta )$, where $\xi \in [0,1]$. According to Taylor's theorem, $\forall s$, $\forall \theta, \ \theta^\prime$,
\begin{align}
    \left| V^{\pi_{\theta^\prime}}(s) - V^{\pi_{\theta}}(s) - \Big\langle \frac{\partial V^{\pi_\theta}(s)}{\partial \theta}, \theta^\prime - \theta \Big\rangle \right| &= \frac{1}{2} \cdot \left| \left( \theta^\prime - \theta \right)^\top \frac{\partial^2 V^{\pi_{\theta_\xi}}(s)}{\partial \theta_\xi^2} \left( \theta^\prime - \theta \right) \right| \\
    &\le \frac{4}{(1 - \gamma)^3} \cdot \| \theta^\prime - \theta \|_2^2. \qquad \left( \text{by \cref{eq:smoothness_softmax_general_intermediate_6}} \right)
\end{align}
Since $V^{\pi_\theta}(s)$ is $8/(1-\gamma)^3$-smooth, for any state $s$, $V^{\pi_\theta}(\rho) 
= \expectation_{s \sim \rho}{ \left[ V^{\pi_\theta}(s) \right]}$ is also $8/(1-\gamma)^3$-smooth.
\end{proof}

\textbf{\cref{lem:lojasiewicz_softmax_general}} 
(Non-uniform \L{}ojasiewicz)\textbf{.} 
Let  $\pi_\theta(\cdot | s) = \softmax(\theta(s, \cdot))$, $s\in \gS$ and fix an arbitrary optimal policy $\pi^*$.
We have,
\begin{align}
    \left\| \frac{\partial V^{\pi_\theta}(\mu)}{\partial \theta }\right\|_2 \ge \frac{1}{\sqrt{S}} \cdot \left\| \frac{ d_{\rho}^{\pi^*} }{ d_{\mu}^{\pi_\theta} } \right\|_\infty^{-1} \cdot \min_s{ \pi_\theta(a^*(s)|s) } \cdot \left[ V^*(\rho) - V^{\pi_\theta}(\rho) \right],
\end{align}
where $a^*(s) = \argmax_{a}{ \pi^*(a | s) }$ ($s\in \gS$). 
Furthermore,
\begin{align}
    \left\| \frac{\partial V^{\pi_\theta}(\mu)}{\partial \theta }\right\|_2 &\ge \frac{1}{\sqrt{S A}} \cdot \left\| \frac{d_{\rho}^{\pi^*}}{d_{\mu}^{\pi_\theta}} \right\|_\infty^{-1} \cdot \left[ \min_{s} \sum_{\bar{a}(s) \in \bar{\gA}^{\pi_\theta}(s)}{ \pi_\theta(\bar{a}(s)|s) } \right] \cdot \left[ V^*(\rho) - V^{\pi_\theta}(\rho) \right],
\end{align}
where $\bar{\gA}^{\pi}(s) = \left\{ \bar{a}(s) \in \gA : Q^{\pi}(s, \bar{a}(s)) = \max_{a}{ Q^{\pi}(s, a) } \right\}$ is the greedy action set for state $s$ given policy $\pi$. Finally,
\begin{align}
\label{eq:lojasiewicz_softmax_general_result_3}
    \left\| \frac{\partial V^{\pi_\theta}(\mu)}{\partial \theta }\right\|_2 \ge \frac{1}{\sqrt{S A}} \cdot \Bigg\| \frac{d_{\rho}^{\pi_{\theta}^*}}{d_{\mu}^{\pi_\theta}} \Bigg\|_{\infty}^{-1} \cdot \Bigg[ \min_{s}{\sum_{a^\prime \in \gA^*(s)}{\pi_{\theta}(a^\prime |s)} } \Bigg] \cdot \left[ V^{*}(\rho) - V^{\pi_\theta}(\rho) \right],
\end{align}
where $\gA^*(s)$ is the ``optimal action set'' under state $s \in \gS$, defined by,
\begin{align}
    \gA^*(s) \coloneqq \left\{ a^*(s) \in \gA: Q^*(s, a^*(s)) = \max_{a \in \gA}{ Q^*(s, a) }  \right\},
\end{align}
and $\pi_\theta^*$ is the globally optimal policy induced by $\pi_\theta$, where for all $s \in \gS$,
\begin{align}
\label{eq:constructed_optimal_policy}
    \pi_{\theta}^*(a | s) \coloneqq \begin{cases}
		\frac{\pi_{\theta}(a | s)}{ \sum_{a^\prime \in \gA^*(s)}{\pi_{\theta}(a^\prime |s)} } , & \text{if } a \in \gA^*(s), \\
		0, & \text{otherwise}.
	\end{cases}
\end{align}
\begin{proof}
We have,
\begin{align}
\MoveEqLeft
    \left\| \frac{\partial V^{\pi_\theta}(\mu)}{\partial \theta }\right\|_2 = \left[ \sum_{s,a} \left( \frac{\partial V^{\pi_\theta}(\mu)}{\partial \theta(s,a)} \right)^2 \right]^{\frac{1}{2}} \\
    &\ge \left[ \sum_{s} \left( \frac{\partial V^{\pi_\theta}(\mu)}{\partial \theta(s,a^*(s))} \right)^2 \right]^{\frac{1}{2}} \\
    &\ge \frac{1}{\sqrt{S}} \sum_{s} \left| \frac{\partial V^{\pi_\theta}(\mu)}{\partial \theta(s,a^*(s))} \right| \qquad \left(
    \text{by Cauchy-Schwarz, } \| x \|_1 = | \langle \rvone, \ |x| \rangle | \le \| \rvone \|_2 \cdot \| x \|_2 \right) \\
    &= \frac{1}{1-\gamma} \cdot \frac{1}{\sqrt{S}} \sum_{s} \left| d_{\mu}^{\pi_\theta}(s) \cdot \pi_\theta(a^*(s)|s) \cdot A^{\pi_\theta}(s,a^*(s)) \right| \qquad \left( \text{by \cref{lem:policy_gradient_softmax}} \right) \\
    &= \frac{1}{1-\gamma} \cdot \frac{1}{\sqrt{S}} \sum_{s} d_{\mu}^{\pi_\theta}(s) \cdot  \pi_\theta(a^*(s)|s) \cdot \left| A^{\pi_\theta}(s,a^*(s)) \right|. 
    \qquad \left( \text{because } d_{\mu}^{\pi_\theta}(s) \ge 0 \text{ and } \pi_\theta(a^*(s)|s) \ge 0 \right)
\end{align}
Define the distribution mismatch coefficient as $\bigg\| \frac{d_{\rho}^{\pi^*}}{d_{\mu}^{\pi_\theta}} \bigg\|_\infty = \max_{s}{ \frac{d_{\rho}^{\pi^*}(s)}{d_{\mu}^{\pi_\theta}(s)} }$. We have,
\begin{align}
    \left\| \frac{\partial V^{\pi_\theta}(\mu)}{\partial \theta }\right\|_2 &\ge \frac{1}{1-\gamma} \cdot \frac{1}{\sqrt{S}} \sum_{s} \frac{ d_{\mu}^{\pi_\theta}(s) }{ d_{\rho}^{\pi^*}(s) } \cdot  d_{\rho}^{\pi^*}(s) \cdot \pi_\theta(a^*(s)|s) \cdot \left| A^{\pi_\theta}(s,a^*(s)) \right| \\
    &\ge \frac{1}{1-\gamma} \cdot  \frac{1}{\sqrt{S}} \cdot \left\| \frac{d_{\rho}^{\pi^*}}{d_{\mu}^{\pi_\theta}} \right\|_\infty^{-1} \cdot \min_s{ \pi_\theta(a^*(s)|s) } \cdot \sum_s{ d_{\rho}^{\pi^*}(s) \cdot \left| A^{\pi_\theta}(s,a^*(s)) \right| } \\
    &\ge \frac{1}{1-\gamma} \cdot  \frac{1}{\sqrt{S}} \cdot \left\| \frac{d_{\rho}^{\pi^*}}{d_{\mu}^{\pi_\theta}} \right\|_\infty^{-1} \cdot \min_s{ \pi_\theta(a^*(s)|s) } \cdot \sum_s{ d_{\rho}^{\pi^*}(s) \cdot A^{\pi_\theta}(s,a^*(s)) } \\
    &= \frac{1}{\sqrt{S}} \cdot \left\| \frac{d_{\rho}^{\pi^*}}{d_{\mu}^{\pi_\theta}} \right\|_\infty^{-1} \cdot \min_s{ \pi_\theta(a^*(s)|s) } \cdot \frac{1}{1-\gamma} \sum_{s}{ d_{\rho}^{\pi^*}(s) \sum_{a}{\pi^*(a|s) \cdot A^{\pi_\theta}(s,a) } } 
    \\
    &= \frac{1}{\sqrt{S}} \cdot \left\| \frac{d_{\rho}^{\pi^*}}{d_{\mu}^{\pi_\theta}} \right\|_\infty^{-1} \cdot \min_s{ \pi_\theta(a^*(s)|s) } \cdot \left[ V^*(\rho) - V^{\pi_\theta}(\rho) \right],
\end{align}
where the
one but last equality used that $\pi^*$ is deterministic and in state $s$ chooses $a^*(s)$ with probability one,
and the last equality uses the performance difference formula (\cref{lem:performance_difference_general}). 

To prove the second claim,
 given a policy $\pi$, define the greedy action set for each state $s$,
\begin{align}
    \bar{\gA}^{\pi}(s) = \left\{ \bar{a}(s) \in \gA : Q^{\pi}(s, \bar{a}(s)) = \max_{a}{ Q^{\pi}(s, a) } \right\}.
\end{align}
By similar arguments that were used in the first part, we have,
\begin{align}
\MoveEqLeft
    \left\| \frac{\partial V^{\pi_\theta}(\mu)}{\partial \theta }\right\|_2 \ge \frac{1}{\sqrt{S A}} \sum_{s,a}{ \left| \frac{\partial V^{\pi_\theta}(\mu)}{\partial \theta(s,a)} \right| } \qquad \left( \text{by Cauchy-Schwarz} \right) \\
    &= \frac{1}{1-\gamma} \cdot \frac{1}{\sqrt{S A}} \sum_{s}{ d_{\mu}^{\pi_\theta}(s) \sum_{a}{ \pi_\theta(a|s) \cdot \left| A^{\pi_\theta}(s,a) \right| } } \qquad \left( \text{by \cref{lem:policy_gradient_softmax}} \right) \\
    &\ge \frac{1}{1-\gamma} \cdot \frac{1}{\sqrt{S A }} \sum_{s}{ d_{\mu}^{\pi_\theta}(s) \sum_{\bar{a}(s) \in \bar{\gA}^{\pi_\theta}(s)}{ \pi_\theta(\bar{a}(s)|s) \cdot \left| A^{\pi_\theta}(s,\bar{a}(s)) \right| } }  \\
    &\ge \frac{1}{1-\gamma} \cdot \frac{1}{\sqrt{S A}} \cdot \left\| \frac{d_{\rho}^{\pi^*}}{d_{\mu}^{\pi_\theta}} \right\|_\infty^{-1} \cdot \left[ \min_{s} \sum_{\bar{a}(s) \in \bar{\gA}^{\pi_\theta}(s)}{ \pi_\theta(\bar{a}(s)|s) } \right] \cdot \sum_{s} d_{\rho}^{\pi^*}(s) \cdot \left| \max_{a}{ Q^{\pi_\theta}(s, a) } - V^{\pi_\theta}(s) \right|,
\end{align}
where the last inequality is because for any $\bar{a}(s) \in \bar{\gA}^{\pi_\theta}(s)$ we have
\begin{align}
    A^{\pi_\theta}(s,\bar{a}(s)) = \max_{a}{ Q^{\pi_\theta}(s, a) } - V^{\pi_\theta}(s),
\end{align}
which is the same value across all $\bar{a}(s) \in \bar{\gA}^{\pi_\theta}(s)$. Then we have,
\begin{align}
    \left\| \frac{\partial V^{\pi_\theta}(\mu)}{\partial \theta }\right\|_2 &\ge \frac{1}{1-\gamma} \cdot \frac{1}{\sqrt{S A}} \cdot \left\| \frac{d_{\rho}^{\pi^*}}{d_{\mu}^{\pi_\theta}} \right\|_\infty^{-1} \cdot \left[ \min_{s} \sum_{\bar{a}(s) \in \bar{\gA}^{\pi_\theta}(s)}{ \pi_\theta(\bar{a}(s)|s) } \right] \cdot \sum_{s} d_{\rho}^{\pi^*}(s) \cdot \left[ \max_{a}{ Q^{\pi_\theta}(s, a) } - V^{\pi_\theta}(s) \right] \\
    &\ge \frac{1}{\sqrt{S A}} \cdot \left\| \frac{d_{\rho}^{\pi^*}}{d_{\mu}^{\pi_\theta}} \right\|_\infty^{-1} \cdot \left[ \min_{s} \sum_{\bar{a}(s) \in \bar{\gA}^{\pi_\theta}(s)}{ \pi_\theta(\bar{a}(s)|s) } \right] \cdot \frac{1}{1-\gamma} \sum_{s} d_{\rho}^{\pi^*}(s) \cdot \left[ Q^{\pi_\theta}(s, a^*(s)) - V^{\pi_\theta}(s) \right] \\
    &= \frac{1}{\sqrt{S A}} \cdot \left\| \frac{d_{\rho}^{\pi^*}}{d_{\mu}^{\pi_\theta}} \right\|_\infty^{-1} \cdot \left[ \min_{s} \sum_{\bar{a}(s) \in \bar{\gA}^{\pi_\theta}(s)}{ \pi_\theta(\bar{a}(s)|s) } \right] \cdot \frac{1}{1-\gamma} \sum_{s} d_{\rho}^{\pi^*}(s) \sum_{a} \pi^*(a |s) \cdot A^{\pi_\theta}(s, a) \\
    &= \frac{1}{\sqrt{S A}} \cdot \left\| \frac{d_{\rho}^{\pi^*}}{d_{\mu}^{\pi_\theta}} \right\|_\infty^{-1} \cdot \left[ \min_{s} \sum_{\bar{a}(s) \in \bar{\gA}^{\pi_\theta}(s)}{ \pi_\theta(\bar{a}(s)|s) } \right] \cdot \left[ V^*(\rho) - V^{\pi_\theta}(\rho) \right],
\end{align}
where the last equation is again according to \cref{lem:performance_difference_general}.

To prove the third claim, using the similar arguments in the second part, we have,
\begin{align}
\MoveEqLeft
    \left\| \frac{\partial V^{\pi_\theta}(\mu)}{\partial \theta }\right\|_2 \ge \frac{1}{\sqrt{S A}} \cdot \sum_{s,a}{ \left| \frac{\partial V^{\pi_\theta}(\mu)}{\partial \theta(s,a)} \right| } \qquad \left( \text{by Cauchy-Schwarz} \right) \\
    &= \frac{1}{1-\gamma} \cdot \frac{1}{\sqrt{S A}} \cdot \sum_{s}{ d_{\mu}^{\pi_\theta}(s) \cdot \sum_{a}{ \pi_\theta(a|s) \cdot \left| A^{\pi_\theta}(s,a) \right| } } \qquad \left( \text{by \cref{lem:policy_gradient_softmax}} \right) \\
    &\ge \frac{1}{1-\gamma} \cdot \frac{1}{\sqrt{S A}} \cdot \sum_{s}{ d_{\mu}^{\pi_\theta}(s) \cdot \sum_{a \in \gA^*(s)}{ \pi_\theta(a|s) \cdot \left| A^{\pi_\theta}(s,a) \right| } } \qquad \left( \text{fewer terms} \right) \\
    &=  \frac{1}{1-\gamma} \cdot \frac{1}{\sqrt{S A}} \cdot \sum_{s}{ d_{\mu}^{\pi_\theta}(s) \cdot \sum_{a \in \gA^*(s)}{ \frac{ \pi_\theta(a|s) }{ \sum_{a^\prime \in \gA^*(s)}{\pi_{\theta}(a^\prime |s)} } \cdot \Bigg[ \sum_{a^\prime \in \gA^*(s)}{\pi_{\theta}(a^\prime |s)} \Bigg] \cdot \left| A^{\pi_\theta}(s,a) \right| } } \\
    &=  \frac{1}{1-\gamma} \cdot \frac{1}{\sqrt{S A}} \cdot \sum_{s}{ d_{\mu}^{\pi_\theta}(s) \cdot \sum_{a \in \gA^*(s)}{ \pi_{\theta}^*(a | s) \cdot \Bigg[ \sum_{a^\prime \in \gA^*(s)}{\pi_{\theta}(a^\prime |s)} \Bigg] \cdot \left| A^{\pi_\theta}(s,a) \right| } }. \qquad \left( \text{by \cref{eq:constructed_optimal_policy}} \right)
\end{align}
Taking minimum over all states, we have,
\begin{align}
\MoveEqLeft
    \left\| \frac{\partial V^{\pi_\theta}(\mu)}{\partial \theta }\right\|_2 \ge  \frac{1}{1-\gamma} \cdot \frac{1}{\sqrt{S A}} \cdot \Bigg[ \min_{s}{\sum_{a^\prime \in \gA^*(s)}{\pi_{\theta}(a^\prime |s)} } \Bigg] \cdot \sum_{s}{ d_{\mu}^{\pi_\theta}(s) \cdot \sum_{a \in \gA^*(s)}{ \pi_{\theta}^*(a | s) \cdot  \left| A^{\pi_\theta}(s,a) \right| } }  \\
    &\ge \frac{1}{1-\gamma} \cdot \frac{1}{\sqrt{S A}} \cdot \Bigg\| \frac{d_{\rho}^{\pi_{\theta}^*}}{d_{\mu}^{\pi_\theta}} \Bigg\|_{\infty}^{-1} \cdot \Bigg[ \min_{s}{\sum_{a^\prime \in \gA^*(s)}{\pi_{\theta}(a^\prime |s)} } \Bigg]  \cdot \sum_{s}{ d_{\mu}^{\pi_\theta^*}(s) \cdot \sum_{a \in \gA^*(s)}{ \pi_{\theta}^*(a | s) \cdot  \left| A^{\pi_\theta}(s,a) \right| } } \\
    &= \frac{1}{1-\gamma} \cdot \frac{1}{\sqrt{S A}} \cdot \Bigg\| \frac{d_{\rho}^{\pi_{\theta}^*}}{d_{\mu}^{\pi_\theta}} \Bigg\|_{\infty}^{-1} \cdot \Bigg[ \min_{s}{\sum_{a^\prime \in \gA^*(s)}{\pi_{\theta}(a^\prime |s)} } \Bigg] \cdot \sum_{s}{ d_{\mu}^{\pi_\theta^*}(s) \cdot \sum_{a}{ \pi_{\theta}^*(a | s) \cdot  \left| A^{\pi_\theta}(s,a) \right| } }. \qquad \left( \text{by \cref{eq:constructed_optimal_policy}} \right) \\
    &\ge \frac{1}{\sqrt{S A}} \cdot \Bigg\| \frac{d_{\rho}^{\pi_{\theta}^*}}{d_{\mu}^{\pi_\theta}} \Bigg\|_{\infty}^{-1} \cdot \Bigg[ \min_{s}{\sum_{a^\prime \in \gA^*(s)}{\pi_{\theta}(a^\prime |s)} } \Bigg] \cdot \frac{1}{1-\gamma} \cdot  \sum_{s}{ d_{\mu}^{\pi_\theta^*}(s) \cdot \sum_{a}{ \pi_{\theta}^*(a | s) \cdot A^{\pi_\theta}(s,a) } }. \qquad \left( |x| \ge x \right) \\
    &= \frac{1}{\sqrt{S A}} \cdot  \Bigg\| \frac{d_{\rho}^{\pi_{\theta}^*}}{d_{\mu}^{\pi_\theta}} \Bigg\|_{\infty}^{-1} \cdot \Bigg[ \min_{s}{\sum_{a^\prime \in \gA^*(s)}{\pi_{\theta}(a^\prime |s)} } \Bigg] \cdot \left[ V^{\pi_\theta^*}(\rho) - V^{\pi_\theta}(\rho) \right] \qquad \left( \text{by \cref{lem:performance_difference_general}} \right) \\
    &= \frac{1}{\sqrt{S A}} \cdot \Bigg\| \frac{d_{\rho}^{\pi_{\theta}^*}}{d_{\mu}^{\pi_\theta}} \Bigg\|_{\infty}^{-1} \cdot \Bigg[ \min_{s}{\sum_{a^\prime \in \gA^*(s)}{\pi_{\theta}(a^\prime |s)} } \Bigg] \cdot \left[ V^{*}(\rho) - V^{\pi_\theta}(\rho) \right], \qquad \left( V^{\pi_\theta^*}(\rho) = V^{*}(\rho) \right) 
\end{align}
where the last equation is because of $\pi_\theta^*$ is an optimal policy, for all $s \in \gS$,
\begin{align}
    \sum_{a \in \gA^*(s)} \pi_{\theta}^*(a | s) = \frac{ \sum_{a \in \gA^*(s)} \pi_{\theta}(a | s)}{ \sum_{a^\prime \in \gA^*(s)}{\pi_{\theta}(a^\prime |s)} } = 1,
\end{align}
thus finishing the proofs.
\end{proof}

\textbf{\cref{lem:lower_bound_cT_softmax_general}.}
Let \cref{ass:posinit} hold. Using \cref{alg:policy_gradient_softmax}, 
we have $c:=\inf_{s\in \cS,t\ge 1} \pi_{\theta_t}(a^*(s)|s) > 0$.
\begin{proof}
The proof is an extension of the proof for  \cref{lem:lower_bound_cT_softmax_special}. Denote $\Delta^*(s) = Q^*(s, a^*(s)) - \max_{a \not= a^*(s)}{ Q^*(s, a) } > 0$ as the optimal value gap of state $s$, where $a^*(s)$ is the action that the optimal policy selects under state $s$, and $\Delta^* = \min_{s \in \gS}{ \Delta^*(s) } > 0$ as the optimal value gap of the MDP. For each state $s \in \gS$, define the following sets:
\begin{align}
    \gR_1(s) &= \left\{ \theta : \frac{\partial V^{\pi_\theta}(\mu)}{\partial \theta(s, a^*(s))} \ge \frac{\partial V^{\pi_\theta}(\mu)}{\partial \theta(s, a)}, \ \forall a \not= a^* \right\}, \\
    \gR_2(s) &= \left\{ \theta :  Q^{\pi_\theta}(s,a^*(s)) \ge Q^*(s,a^*(s)) - \Delta^*(s) / 2 \right\}, \\
    \gR_3(s) &= \left\{ \theta_t: V^{\pi_{\theta_t}}(s) \ge Q^{\pi_{\theta_t}}(s, a^*(s)) - \Delta^*(s) / 2, 
    \text{ for all } t \ge 1 \text{ large enough} \right\}, \\
    \gN_c(s) &= \left\{ \theta : \pi_\theta(a^*(s) | s) \ge \frac{c(s)}{c(s)+1} \right\}, \text{ where } c(s) = \frac{A}{(1-\gamma) \cdot \Delta^*(s)} - 1.
\end{align}
Similarly to the previous proof, we have the following claims: 
\begin{description}
    \item[Claim I.] $\gR_1(s) \cap \gR_2(s) \cap \gR_3(s)$ is a ``nice" region, in the sense that, following a gradient update, (i) if $\theta_{t} \in \gR_1(s) \cap \gR_2(s) \cap \gR_3(s)$, then $\theta_{t+1} \in \gR_1(s) \cap \gR_2(s) \cap \gR_3(s)$;
    while we also have (ii) $\pi_{\theta_{t+1}}(a^*(s) | s) \ge \pi_{\theta_{t}}(a^*(s) | s)$.
    \item[Claim II.] $\gN_c(s) \cap \gR_2(s) \cap \gR_3(s) \subset \gR_1(s) \cap \gR_2(s) \cap \gR_3(s)$.
    \item[Claim III.] There exists a finite time $t_0(s) \ge 1$, such that $\theta_{t_0(s)} \in \gN_c(s) \cap \gR_2(s) \cap \gR_3(s)$, and thus $\theta_{t_0(s)} \in \gR_1(s) \cap \gR_2(s) \cap \gR_3(s)$, which implies $\inf_{t\ge 1} \pi_{\theta_t}(a^*(s) | s) = \min_{1 \le t \le t_0(s)}{ \pi_{\theta_t}(a^*(s) | s) }$.
    \item[Claim IV.] Define $t_0 = \max_{s}{ t_0(s)}$. Then, we have $\inf_{s\in \cS, t\ge 1} \pi_{\theta_t}(a^*(s)|s) = \min_{1 \le t \le t_0}{ \min_{s} \pi_{\theta_t}(a^*(s) | s) }$. 
\end{description}
Clearly, claim IV suffices to prove the lemma since for any $\theta$, $\min_{s,a}\pi_{\theta}(a|s)>0$.
In what follows we provide the proofs of these four claims.
\paragraph{Claim I.}
First we prove part (i) of the claim. If $\theta_{t} \in \gR_1(s) \cap \gR_2(s) \cap \gR_3(s)$, then $\theta_{t+1} \in \gR_1(s) \cap \gR_2(s) \cap \gR_3(s)$. Suppose $\theta_{t} \in \gR_1(s) \cap \gR_2(s) \cap \gR_3(s)$. We have $\theta_{t+1} \in \gR_3(s)$ by the definition of $\gR_3(s)$. We have,
\begin{align}
    Q^{\pi_{\theta_t}}(s,a^*(s)) \ge Q^*(s,a^*(s)) - \Delta^*(s) / 2.
\end{align}
According to smoothness arguments as \cref{eq:smoothness_progress_general}, we have $V^{\pi_{\theta_{t+1}}}(s^\prime) \ge V^{\pi_{\theta_t}}(s^\prime)$, and
\begin{align}
    Q^{\pi_{\theta_{t+1}}}(s,a^*(s)) &= Q^{\pi_{\theta_t}}(s,a^*(s)) + Q^{\pi_{\theta_{t+1}}}(s,a^*(s)) - Q^{\pi_{\theta_t}}(s,a^*(s)) \\
    &= Q^{\pi_{\theta_t}}(s,a^*(s)) + \gamma \sum_{s^\prime}{ \gP(s^\prime | s, a^*(s))} \cdot \left[ V^{\pi_{\theta_{t+1}}}(s^\prime) - V^{\pi_{\theta_t}}(s^\prime) \right] \\
    &\ge  Q^{\pi_{\theta_t}}(s,a^*(s)) + 0\\
    & \ge Q^*(s,a^*(s)) - \Delta^*(s) / 2,
\end{align}
which means $\theta_{t+1} \in \gR_2(s)$. Next we prove $\theta_{t+1} \in \gR_1(s)$. Note that $\forall a \not= a^*(s)$,
\begin{align}
\label{eq:lower_bound_cT_softmax_general_intermediate_1}
    Q^{\pi_{\theta_t}}(s,a^*(s)) - Q^{\pi_{\theta_t}}(s,a) &= Q^{\pi_{\theta_t}}(s,a^*(s)) - Q^*(s,a^*(s)) + Q^*(s,a^*(s)) - Q^{\pi_{\theta_t}}(s,a) \\
    &\ge - \Delta^*(s) / 2 + Q^*(s,a^*(s)) - Q^*(s,a) + Q^*(s,a) - Q^{\pi_{\theta_t}}(s,a) \\
    &\ge - \Delta^*(s) / 2 + Q^*(s,a^*(s)) - \max_{a \not= a^*(s)}{ Q^*(s,a) } + Q^*(s,a) - Q^{\pi_{\theta_t}}(s,a) \\
    &= - \Delta^*(s) / 2 + \Delta^*(s) + \gamma \sum_{s^\prime}{ \gP(s^\prime | s, a)} \cdot \left[ V^*(s^\prime) - V^{\pi_{\theta_t}}(s^\prime) \right] \\
    &\ge - \Delta^*(s) / 2 + \Delta^*(s) + 0\\
    & = \Delta^*(s) / 2.
\end{align}
Using similar arguments we also have $Q^{\pi_{\theta_{t+1}}}(s,a^*(s)) - Q^{\pi_{\theta_{t+1}}}(s,a) \ge \Delta^*(s) / 2$.
According to \cref{lem:policy_gradient_softmax},
\begin{align}
    \frac{\partial V^{\pi_{\theta_t}}(\mu)}{\partial \theta_t(s, a)} &= \frac{1}{1-\gamma} \cdot d_{\mu}^{\pi_{\theta_t}}(s) \cdot \pi_{\theta_t}(a | s) \cdot A^{\pi_{\theta_t}}(s, a) \\
    &= \frac{1}{1-\gamma} \cdot  d_{\mu}^{\pi_{\theta_t}}(s) \cdot \pi_{\theta_t}(a | s) \cdot \left[ Q^{\pi_{\theta_t}}(s, a) - V^{\pi_{\theta_t}}(s) \right].
\end{align}
Furthermore, since $\frac{\partial V^{\pi_{\theta_t}}(\mu)}{\partial {\theta_t}(s, a^*(s))} \ge \frac{\partial V^{\pi_{\theta_t}}(\mu)}{\partial {\theta_t}(s, a)}$, we have
\begin{align}
    \pi_{\theta_t}(a^*(s) | s) \cdot \left[Q^{\pi_{\theta_t}}(s, a^*(s)) - V^{\pi_{\theta_t}}(s) \right] \ge \pi_{\theta_t}(a | s) \cdot  \left[Q^{\pi_{\theta_t}}(s, a) - V^{\pi_{\theta_t}}(s) \right].
\end{align}
Similarly to the first part in the proof for \cref{lem:lower_bound_cT_softmax_special}. There are two cases.

Case (a): If $\pi_{\theta_t}(a^*(s) | s) \ge \pi_{\theta_t}(a | s)$, then $\theta_t(s, a^*(s)) \ge \theta_t(s, a)$.  
After an update of the parameters,
\begin{align}
    \theta_{t+1}(s, a^*(s)) &= \theta_{t}(s, a^*(s)) + \eta \cdot \frac{\partial V^{\pi_{\theta_t}}(\mu)}{\partial {\theta_t}(s, a^*(s))} \\ 
    &\ge \theta_t(s, a) + \eta \cdot \frac{\partial V^{\pi_{\theta_t}}(\mu)}{\partial {\theta_t}(s, a)} = \theta_{t+1}(s, a),
\end{align}
which implies $\pi_{\theta_{t+1}}(a^*(s) | s) \ge \pi_{\theta_{t+1}}(a | s)$. Since $Q^{\pi_{\theta_{t+1}}}(s,a^*(s)) - Q^{\pi_{\theta_{t+1}}}(s,a) \ge \Delta^*(s) / 2 \ge 0$, $\forall a$, we have $Q^{\pi_{\theta_{t+1}}}(s,a^*(s)) - V^{\pi_{\theta_{t+1}}}(s) = Q^{\pi_{\theta_{t+1}}}(s,a^*(s)) - \sum_{a}{ \pi_{\theta_{t+1}}(a | s) \cdot Q^{\pi_{\theta_{t+1}}}(s,a) } \ge 0$, and
\begin{align}
    \pi_{\theta_{t+1}}(a^*(s) | s) \cdot \left[Q^{\pi_{\theta_{t+1}}}(s, a^*(s)) - V^{\pi_{\theta_{t+1}}}(s) \right] \ge \pi_{\theta_{t+1}}(a | s) \cdot  \left[Q^{\pi_{\theta_{t+1}}}(s, a) - V^{\pi_{\theta_{t+1}}}(s) \right],
\end{align}
which is equivalent to $\frac{\partial V^{\pi_{\theta_{t+1}}}(\mu)}{\partial {\theta_{t+1}}(s, a^*(s))} \ge \frac{\partial V^{\pi_{\theta_{t+1}}}(\mu)}{\partial {\theta_{t+1}}(s, a)}$, i.e., $\theta_{t+1} \in \gR_1(s)$.

Case (b): If $\pi_{\theta_t}(a^*(s) | s) < \pi_{\theta_t}(a | s)$, then by $\frac{\partial V^{\pi_{\theta_t}}(\mu)}{\partial {\theta_t}(s, a^*(s))} \ge \frac{\partial V^{\pi_{\theta_t}}(\mu)}{\partial {\theta_t}(s, a)}$,
\begin{align}
\MoveEqLeft
    \pi_{\theta_t}(a^*(s) | s) \cdot \left[Q^{\pi_{\theta_t}}(s, a^*(s)) - V^{\pi_{\theta_t}}(s) \right] \ge \pi_{\theta_t}(a | s) \cdot \left[Q^{\pi_{\theta_t}}(s, a) - V^{\pi_{\theta_t}}(s) \right] \\
    &= \pi_{\theta_t}(a | s) \cdot \left[Q^{\pi_{\theta_t}}(s, a^*(s)) - V^{\pi_{\theta_t}}(s) + Q^{\pi_{\theta_t}}(s, a) - Q^{\pi_{\theta_t}}(s, a^*(s)) \right],
\end{align}
which, after rearranging, is equivalent to
\begin{align}
    Q^{\pi_{\theta_t}}(s, a^*(s)) - Q^{\pi_{\theta_t}}(s, a) &\ge \left( 1 - \frac{\pi_{\theta_t}(a^*(s) | s)}{\pi_{\theta_t}(a | s)} \right) \cdot \left[ Q^{\pi_{\theta_t}}(s, a^*(s)) - V^{\pi_{\theta_t}}(s)  \right] \\
    &= \left( 1 - \exp\left\{ \theta_{t}(s, a^*(s)) - \theta_{t}(s, a) \right\} \right) \cdot \left[ Q^{\pi_{\theta_t}}(s, a^*(s)) - V^{\pi_{\theta_t}}(s) \right].
\end{align}
Since $\theta_{t+1} \in \gR_3(s)$, we have, 
\begin{align}
    Q^{\pi_{\theta_{t+1}}}(s, a^*(s)) - V^{\pi_{\theta_{t+1}}}(s) \le \Delta^*(s) / 2 \le Q^{\pi_{\theta_{t+1}}}(s,a^*(s)) - Q^{\pi_{\theta_{t+1}}}(s,a).
\end{align}
On the other hand,
\begin{align}
    \theta_{t+1}(s, a^*(s)) - \theta_{t+1}(s, a) &= \theta_{t}(s, a^*(s)) + \eta \cdot \frac{\partial V^{\pi_{\theta_t}}(\mu)}{\partial {\theta_t}(s, a^*(s))} - \theta_{t}(s, a) - \eta \cdot \frac{\partial V^{\pi_{\theta_t}}(\mu)}{\partial {\theta_t}(s, a)} \\
    &\ge \theta_{t}(s, a^*(s)) - \theta_{t}(s, a),
\end{align}
which implies
\begin{align}
    1 - \exp\left\{ \theta_{t+1}(s, a^*(s)) - \theta_{t+1}(s, a) \right\} \le 1 - \exp\left\{ \theta_{t}(s, a^*(s)) - \theta_{t}(s, a) \right\}.
\end{align}
Furthermore, since $1 - \exp\left\{ \theta_{t}(s, a^*(s)) - \theta_{t}(s, a) \right\} = 1 - \frac{\pi_{\theta_{t}}(a^*(s) | s)}{\pi_{\theta_{t}}(a | s)} > 0$ (in this case $\pi_{\theta_t}(a^*(s) |s) < \pi_{\theta_t}(a | s))$,
\begin{align}
    \left( 1 - \exp\left\{ \theta_{t+1}(s, a^*(s)) - \theta_{t+1}(s, a) \right\} \right) \cdot \left[ Q^{\pi_{\theta_{t+1}}}(s, a^*(s)) - V^{\pi_{\theta_{t+1}}}(s) \right] \le Q^{\pi_{\theta_{t+1}}}(s, a^*(s)) - Q^{\pi_{\theta_{t+1}}}(s, a),
\end{align}
which after rearranging is equivalent to
\begin{align}
    \pi_{\theta_{t+1}}(a^*(s) | s) \cdot \left[Q^{\pi_{\theta_{t+1}}}(s, a^*(s)) - V^{\pi_{\theta_{t+1}}}(s) \right] \ge \pi_{\theta_{t+1}}(a | s) \cdot  \left[Q^{\pi_{\theta_{t+1}}}(s, a) - V^{\pi_{\theta_{t+1}}}(s) \right],
\end{align}
which means $\frac{\partial V^{\pi_{\theta_{t+1}}}(\mu)}{\partial {\theta_{t+1}}(s, a^*(s))} \ge \frac{\partial V^{\pi_{\theta_{t+1}}}(\mu)}{\partial {\theta_{t+1}}(s, a)}$ i.e., $\theta_{t+1} \in \gR_1(s)$. Now we have (i) if $\theta_{t} \in \gR_1(s) \cap \gR_2(s) \cap \gR_3(s)$, then $\theta_{t+1} \in \gR_1(s) \cap \gR_2(s) \cap \gR_3(s)$.

Let us now turn to proving part~(ii).
We have $\pi_{\theta_{t+1}}(a^*(s) | s) \ge \pi_{\theta_{t}}(a^*(s) | s)$. If $\theta_{t} \in \gR_1(s) \cap \gR_2(s) \cap \gR_3(s)$, then $\frac{\partial V^{\pi_{\theta_t}}(\mu)}{\partial {\theta_t}(s, a^*(s))} \ge \frac{\partial V^{\pi_{\theta_t}}(\mu)}{\partial {\theta_t}(s, a)}, \ \forall a \not= a^*$. After an update of the parameters,
\begin{align}
    \pi_{\theta_{t+1}}(a^*(s) | s) &= \frac{\exp\left\{ \theta_{t+1}(s, a^*(s)) \right\}}{ \sum_{a}{ \exp\left\{ \theta_{t+1}(s, a) \right\}} } \\
    &= \frac{\exp\left\{ \theta_{t}(s, a^*(s)) + \eta \cdot \frac{\partial V^{\pi_{\theta_t}}(\mu)}{\partial {\theta_t}(s, a^*(s))} \right\}}{ \sum_{a}{ \exp\left\{ \theta_{t}(s, a) + \eta \cdot \frac{\partial V^{\pi_{\theta_t}}(\mu)}{\partial {\theta_t}(s, a)} \right\}} } \\
    &\ge \frac{\exp\left\{ \theta_{t}(s, a^*(s)) + \eta \cdot \frac{\partial V^{\pi_{\theta_t}}(\mu)}{\partial {\theta_t}(s, a^*(s))} \right\}}{ \sum_{a}{ \exp\left\{ \theta_{t}(s, a) + \eta \cdot \frac{\partial V^{\pi_{\theta_t}}(\mu)}{\partial {\theta_t}(s, a^*(s))} \right\}} } 
    \qquad \left( \text{because } \frac{\partial V^{\pi_{\theta_t}}(\mu)}{\partial {\theta_t}(s, a^*(s))} \ge \frac{\partial V^{\pi_{\theta_t}}(\mu)}{\partial {\theta_t}(s, a)} \right) \\
    &= \frac{\exp\left\{ {\theta_t}(s, a^*(s)) \right\}}{ \sum_{a}{ \exp\left\{ {\theta_t}(s, a) \right\}} } = \pi_{\theta_t}(a^*(s) | s).
\end{align}
\paragraph{Claim II.}$\gN_c(s) \cap \gR_2(s) \cap \gR_3(s) \subset \gR_1(s) \cap \gR_2(s) \cap \gR_3(s)$. Suppose $\theta \in \gR_2(s) \cap \gR_3(s)$ and $\pi_{\theta}(a^*(s) | s) \ge \frac{c(s)}{c(s)+1}$. There are two cases.

Case (a): If $\pi_\theta(a^*(s) | s ) \ge \max_{a \not= a^*(s)}\{ \pi_\theta(a | s) \}$, then we have,
\begin{align}
    \frac{\partial V^{\pi_{\theta}}(\mu)}{\partial \theta(s, a^*(s))} &= \frac{1}{1-\gamma} \cdot  d_{\mu}^{\pi_{\theta}}(s) \cdot \pi_{\theta}(a^*(s) | s) \cdot \left[ Q^{\pi_{\theta}}(s, a^*(s)) - V^{\pi_{\theta}}(s) \right] \\
    &> \frac{1}{1-\gamma} \cdot  d_{\mu}^{\pi_{\theta}}(s) \cdot \pi_{\theta}(a | s) \cdot \left[ Q^{\pi_{\theta}}(s, a) - V^{\pi_{\theta}}(s) \right] \\
    &= \frac{\partial V^{\pi_{\theta}}(\mu)}{\partial \theta(s, a)},
\end{align}
where the inequality is since $Q^{\pi_{\theta}}(s, a^*(s)) - Q^{\pi_{\theta}}(s, a) \ge \Delta^*(s) /2 > 0$, $\forall a \not= a^*(s)$, similarly to \cref{eq:lower_bound_cT_softmax_general_intermediate_1}.

Case (b): $\pi_\theta(a^*(s) | s) < \max_{a \not= a^*(s)}\{ \pi_\theta(a | s) \}$, which is not possible. Suppose there exists an $a \not= a^*(s)$, such that $\pi_\theta(a^*(s) | s) < \pi_\theta(a | s)$. Then we have the following contradiction,
\begin{align}
    \pi_\theta(a^*(s) | s) + \pi_\theta(a | s) &> \frac{2 \cdot c(s)}{c(s)+1} = 2 - \frac{2 \cdot (1 - \gamma) \cdot \Delta^*(s)}{A} > 1,
\end{align}
where the last inequality is according to $A \ge 2$ (there are at least two actions), and $\Delta^*(s) \le 1/(1 - \gamma)$.

\paragraph{Claim III.}(1) According to the asymptotic convergence results of \citet[Theorem 5.1]{AgKaLeMa19},
which we can use thanks to \cref{ass:posinit},
 $\pi_{\theta_{t}}(a^*(s) | s) \to 1$. Hence, there exists $t_1(s) \ge 1$, such that $\pi_{\theta_{t_1(s)}}(a^*(s) | s) \ge \frac{c(s)}{c(s)+1}$. (2)  $Q^{\pi_{\theta_t}}(s,a^*(s)) \to Q^*(s,a^*(s))$, as $t \to \infty$. There exists $t_2(s) \ge 1$, such that $Q^{\pi_{\theta_{t_2(s)}}}(s,a^*(s)) \ge Q^*(s,a^*(s)) - \Delta^*(s)/2$. (3) $Q^{\pi_{\theta_t}}(s, a^*(s)) \to V^*(s)$, and $V^{\pi_{\theta_t}}(s) \to V^*(s)$, as $t \to \infty$. There exists $t_3(s) \ge 1$, such that $\forall t \ge t_3(s)$, $Q^{\pi_{\theta_t}}(s, a^*(s)) - V^{\pi_{\theta_t}}(s) \le \Delta^*(s) / 2$.

Define $t_0(s) = \max\{ t_1(s), t_2(s), t_3(s) \}$. We have  $\theta_{t_0(s)} \in \gN_c(s) \cap \gR_2(s) \cap \gR_3(s)$, and thus $\theta_{t_0(s)} \in \gR_1(s) \cap \gR_2(s) \cap \gR_3(s)$. According to the first part in our proof, i.e., once $\theta_t$ is in $\gR_1(s) \cap \gR_2(s) \cap \gR_3(s)$, following gradient update $\theta_{t+1}$ will be in $\gR_1(s) \cap \gR_2(s) \cap \gR_3(s)$, and $\pi_{\theta_t}(a^*(s) | s)$ is increasing in $\gR_1(s) \cap \gR_2(s) \cap \gR_3(s)$, we have $\inf_{t}{ \pi_{\theta_t}(a^*(s) | s)} = \min_{1 \le t \le t_0(s)}{ \pi_{\theta_t}(a^*(s) | s)}$. $t_0(s)$ depends on initialization and $c(s)$, which only depends on the MDP and state $s$.

\paragraph{Claim IV.}Define $t_0 = \max_{s}{ t_0(s)}$. Then we have $\inf_{s\in \cS, t\ge 1} \pi_{\theta_t}(a^*(s)|s) = \min_{1 \le t \le t_0}{ \min_{s}\pi_{\theta_t}(a^*(s) | s) }>0$.
\end{proof}

\textbf{\cref{thm:final_rates_softmax_general}.}
Let \cref{ass:posinit} hold and
let $\{\theta_t\}_{t\ge 1}$ be generated using \cref{alg:policy_gradient_softmax} with $\eta = (1 - \gamma)^3/8$,
$c$ the positive constant from \cref{lem:lower_bound_cT_softmax_general}.
Then, for all $t\ge 1$,
\begin{align}
    V^*(\rho) - V^{\pi_{\theta_t}}(\rho) \le \frac{16 S }{c^2(1-\gamma)^6 t} \cdot \bigg\| \frac{d_{\mu}^{\pi^*}}{\mu} \bigg\|_\infty^2 \cdot \bigg\| \frac{1}{\mu} \bigg\|_\infty\,.
\end{align}
\begin{proof}
Let us first note that for any $\theta$ and $\mu$,
\begin{align}
\label{eq:stationary_distribution_dominate_initial_state_distribution}
    d_{\mu}^{\pi_\theta}(s) &= \expectation_{s_0 \sim \mu}{ \left[ d_{\mu}^{\pi_\theta}(s) \right] } \\
    &= \expectation_{s_0 \sim \mu}{ \left[ (1 - \gamma) \sum_{t=0}^{\infty}{ \gamma^t \probability(s_t = s | s_0, \pi_\theta, \gP) } \right] } \\
    &\ge \expectation_{s_0 \sim \mu}{ \left[ (1 - \gamma) \probability(s_0 = s | s_0)  \right] } \\
    &= (1 - \gamma) \cdot \mu(s)\,.
\end{align}
According to the value sub-optimality lemma of \cref{lem:value_suboptimality},
\begin{align}
\MoveEqLeft
    V^*(\rho) - V^{\pi_\theta}(\rho) = \frac{1}{1 - \gamma} \sum_{s}{ d_{\rho}^{\pi_\theta}(s)  \sum_{a}{ \left( \pi^*(a | s) - \pi_\theta(a | s) \right) \cdot Q^*(s,a) } } \\
    &= \frac{1}{1 - \gamma} \sum_{s} \frac{d_{\rho}^{\pi_\theta}(s)}{d_{\mu}^{\pi_\theta}(s)} \cdot d_{\mu}^{\pi_\theta}(s) \sum_{a}{ \left( \pi^*(a | s) - \pi_\theta(a | s) \right) \cdot Q^*(s,a) } \\
    &\le \frac{1}{1 - \gamma} \cdot \left\| \frac{1}{d_{\mu}^{\pi_\theta}} \right\|_\infty \sum_{s} d_{\mu}^{\pi_\theta}(s) \sum_{a}{ \left( \pi^*(a | s) - \pi_\theta(a | s) \right) \cdot Q^*(s,a) } \\
    &\le \frac{1}{(1 - \gamma)^2} \cdot \left\| \frac{1}{\mu} \right\|_\infty \sum_{s} d_{\mu}^{\pi_\theta}(s) \sum_{a}{ \left( \pi^*(a | s) - \pi_\theta(a | s) \right) \cdot Q^*(s,a) } 
    \qquad \left( \text{by \cref{eq:stationary_distribution_dominate_initial_state_distribution} and } \min_s\mu(s)>0 \right)
    \\
    &= \frac{1}{1 - \gamma} \cdot \left\| \frac{1}{\mu} \right\|_\infty \cdot \left[ V^*(\mu) - V^{\pi_\theta}(\mu) \right],
\end{align}
where the first inequality is because of
\begin{align}
    \sum_{a}{ \left( \pi^*(a | s) - \pi_\theta(a | s) \right) \cdot Q^*(s,a) } \ge 0,
\end{align}
and the last equation is again by \cref{lem:value_suboptimality}. According to \cref{lem:smoothness_softmax_general}, $V^{\pi_\theta}(\mu)$ is $\beta$-smooth with $\beta = 8/(1 -\gamma)^3$. Denote $\delta_t = V^*(\mu) - V^{\pi_{\theta_{t}}}(\mu)$. And note $ \eta = \frac{(1 - \gamma)^3}{8}$. We have,
\begin{align}
\label{eq:smoothness_progress_general}
    \delta_{t+1} - \delta_t &= V^{\pi_{\theta_{t}}}(\mu) - V^{\pi_{\theta_{t+1}}}(\mu) \\
    &\le - \frac{(1-\gamma)^3}{16} \cdot \left\| \frac{\partial V^{\pi_{\theta_t}}(\mu)}{\partial \theta_t} \right\|_2^2 
    \qquad \left( \text{by \cref{lem:ascent_lemma_smooth_function}} \right) \\
    &\le - \frac{(1-\gamma)^3}{16 S} \cdot \left\| \frac{d_{\mu}^{\pi^*}}{d_{\mu}^{\pi_{\theta_t}}} \right\|_\infty^{-2} \cdot \min_s{ \pi_{\theta_t}(a^*(s)|s)^2 } \cdot \left[ V^*(\mu) - V^{\pi_{\theta_t}}(\mu) \right]^2 \qquad \left( \text{by \cref{lem:lojasiewicz_softmax_general}} \right) \\
    &\le - \frac{(1-\gamma)^5}{16 S} \cdot \left\| \frac{d_{\mu}^{\pi^*}}{\mu} \right\|_\infty^{-2} \cdot \min_s{ \pi_{\theta_t}(a^*(s)|s)^2 } \cdot \delta_t^2 \\
    &\le - \frac{(1-\gamma)^5}{16 S} \cdot \left\| \frac{d_{\mu}^{\pi^*}}{\mu} \right\|_\infty^{-2} \cdot \inf_{s\in \cS,t\ge 1} \pi_{\theta_t}(a^*(s)|s)^2 \cdot \delta_t^2,
\end{align}
where the second to last inequality is by $d_{\mu}^{\pi_{\theta_t}}(s) \ge (1 - \gamma) \cdot \mu(s)$ (cf. \cref{eq:stationary_distribution_dominate_initial_state_distribution}). According to \cref{lem:lower_bound_cT_softmax_general}, $c=\inf_{s\in \cS,t\ge 1} \pi_{\theta_t}(a^*(s)|s) > 0$.  Using similar induction arguments as in \cref{eq:one_over_t_induction},
\begin{align}
    V^*(\mu) - V^{\pi_{\theta_{t}}}(\mu) \le \frac{16 S }{c^2(1-\gamma)^5 t} \cdot \left\| \frac{d_{\mu}^{\pi^*}}{\mu} \right\|_\infty^2,
\end{align}
which leads to the final result,
\begin{align}
    V^*(\rho) - V^{\pi_{\theta_t}}(\rho) \le \frac{1}{1 - \gamma} \cdot \left\| \frac{1}{\mu} \right\|_\infty \cdot \left[ V^*(\mu) - V^{\pi_{\theta_t}}(\mu) \right] \le \frac{16 S }{c^2(1-\gamma)^6 t} \cdot \left\| \frac{d_{\mu}^{\pi^*}}{\mu} \right\|_\infty^2 \cdot \left\| \frac{1}{\mu} \right\|_\infty,
\end{align}
thus, finishing the proof.
\end{proof}

\subsection{Proofs for \cref{sec:entropy_policy_gradient}: entropy regularized softmax policy gradient}
\label{sec:proofs_entropy_policy_gradient}

\subsubsection{Preliminaries}
\label{sec:proofs_entropy_policy_gradient_priliminaries}

\textbf{\cref{lem:policy_gradient_entropy}.}
Entropy regularized policy gradient w.r.t. $\theta$ is
\begin{align}
\label{eq:policy_gradient_entropy_vector_form}
    \frac{\partial \tilde{V}^{\pi_\theta}(\mu)}{\partial \theta(s,a)} &= \frac{1}{1-\gamma} \cdot d_{\mu}^{\pi_\theta}(s) \cdot \pi_\theta(a|s) \cdot  \tilde{A}^{\pi_\theta}(s,a) \\
    \frac{\partial \tilde{V}^{\pi_\theta}(\mu)}{\partial \theta(s, \cdot)} &= \frac{1}{1-\gamma} \cdot d_{\mu}^{\pi_\theta}(s) \cdot H(\pi_\theta( \cdot |s)) \left[ \tilde{Q}^{\pi_\theta}(s, \cdot) - \tau \log{ \pi_\theta(\cdot | s) } \right] \\
    &= \frac{1}{1-\gamma} \cdot d_{\mu}^{\pi_\theta}(s) \cdot H(\pi_\theta( \cdot |s)) \left[ \tilde{Q}^{\pi_\theta}(s, \cdot) - \tau \theta(s, \cdot) \right], \ \forall s
\end{align}
where $\tilde{A}^{\pi_\theta}(s, a)$ is soft advantage function defined as
\begin{align}
     \tilde{A}^{\pi_\theta}(s, a) &= \tilde{Q}^{\pi_\theta}(s, a) - \tau \log{\pi_\theta(a | s)} - \tilde{V}^{\pi_\theta}(s) \\
     \tilde{Q}^{\pi_\theta}(s, a) &= r(s,a) + \gamma \sum_{s^\prime}{ \gP( s^\prime | s, a)  \tilde{V}^{{\pi_\theta}}(s^\prime) }.
\end{align}
\begin{proof}
According to the definition of $\tilde{V}^{\pi_\theta}$,
\begin{align}
    \tilde{V}^{{\pi_\theta}}(\mu) = \expectation_{s \sim \mu}{ \sum_{a}{ {\pi_\theta}(a | s)} } \cdot { \left[ \tilde{Q}^{\pi_\theta}(s, a) - \tau \log{{\pi_\theta}(a | s)} \right]}.
\end{align}
Taking derivative w.r.t. $\theta$,
\begin{align}
    \frac{\partial \tilde{V}^{\pi_\theta}(\mu)}{\partial \theta} &= \expectation_{s \sim \mu}{ \sum_{a}{ \frac{ \partial {\pi_\theta}(a | s) }{\partial \theta} } } \cdot { \left[ \tilde{Q}^{\pi_\theta}(s, a) - \tau \log{{\pi_\theta}(a | s)} \right]} + \expectation_{s \sim \mu}{ \sum_{a}{ {\pi_\theta}(a | s)} } \cdot { \left[ \frac{ \partial \tilde{Q}^{\pi_\theta}(s, a)}{\partial \theta} - \tau \cdot \frac{1}{{\pi_\theta}(a | s)} \cdot \frac{ \partial {\pi_\theta}(a | s) }{\partial \theta} \right]} \\
    &= \expectation_{s \sim \mu}{ \sum_{a}{ \frac{ \partial {\pi_\theta}(a | s) }{\partial \theta} } } \cdot { \left[ \tilde{Q}^{\pi_\theta}(s, a) - \tau \log{{\pi_\theta}(a | s)} \right]} + \expectation_{s \sim \mu}{ \sum_{a}{ {\pi_\theta}(a | s)}  } \cdot { \frac{ \partial \tilde{Q}^{\pi_\theta}(s, a)}{\partial \theta} } \\
    &= \expectation_{s \sim \mu}{ \sum_{a}{ \frac{ \partial {\pi_\theta}(a | s) }{\partial \theta} } } \cdot { \left[ \tilde{Q}^{\pi_\theta}(s, a) - \tau \log{{\pi_\theta}(a | s)} \right]} + \gamma \cdot \expectation_{s \sim \mu}{ \sum_{a}{ {\pi_\theta}(a | s)} }{ \sum_{s^\prime} \gP(s^\prime | s, a) \cdot \frac{ \partial \tilde{V}^{\pi_\theta}(s^\prime)}{\partial \theta} } \\
    &= \frac{1}{1-\gamma} \sum_{s} d_{\mu}^{\pi_\theta}(s) \sum_{a} \frac{ \partial {\pi_\theta}(a | s) }{\partial \theta} \cdot \left[ \tilde{Q}^{\pi_\theta}(s, a) - \tau \log{{\pi_\theta}(a | s)} \right],
\end{align}
where the second equation is because of
\begin{align}
    \sum_{a}{ {\pi_\theta}(a | s)} \cdot { \left[ \frac{1}{{\pi_\theta}(a | s)} \cdot \frac{ \partial {\pi_\theta}(a | s) }{\partial \theta} \right]} = \sum_{a}{ \frac{ \partial {\pi_\theta}(a | s) }{\partial \theta} } = \frac{ \partial }{\partial \theta} \sum_{a}{ {\pi_\theta}(a | s) } = \frac{ \partial 1 }{\partial \theta} = 0.
\end{align}
Using similar arguments as in the proof for \cref{lem:policy_gradient_softmax}, i.e., for $s^\prime \not= s$, $\frac{\partial \pi_\theta(a | s)}{\partial \theta(s^\prime, \cdot)} = \rvzero$,
\begin{align}
    \frac{\partial \tilde{V}^{\pi_\theta}(\mu)}{\partial \theta(s, \cdot)} 
    &= \frac{1}{1-\gamma} \cdot d_{\mu}^{\pi_\theta}(s) \cdot { \left[ \sum_{a} \frac{\partial \pi_\theta(a | s)}{\partial \theta(s, \cdot)} \cdot \left[ \tilde{Q}^{\pi_\theta}(s, a) - \tau \log{{\pi_\theta}(a | s)} \right] \right] } \\
    &= \frac{1}{1-\gamma} \cdot d_{\mu}^{\pi_\theta}(s) \cdot \left( \frac{d \pi(\cdot | s)}{d \theta(s, \cdot)} \right)^\top \left[ \tilde{Q}^{\pi_\theta}(s, \cdot) - \tau \log{{\pi_\theta}(\cdot | s)} \right] \\
    &= \frac{1}{1-\gamma} \cdot d_{\mu}^{\pi_\theta}(s) \cdot H(\pi_\theta(\cdot | s)) \left[ \tilde{Q}^{\pi_\theta}(s, \cdot) - \tau \log{{\pi_\theta}(\cdot | s)} \right] 
    \qquad \left( \text{by \cref{eq:H_matrix}} \right) \\
    &= \frac{1}{1-\gamma} \cdot d_{\mu}^{\pi_\theta}(s) \cdot H(\pi_\theta(\cdot | s)) \left[ \tilde{Q}^{\pi_\theta}(s, \cdot) - \tau \theta(\cdot | s) + \tau \log \sum_{a} \exp\{ \theta(s, a)\} \cdot \rvone \right] \\
    &= \frac{1}{1-\gamma} \cdot d_{\mu}^{\pi_\theta}(s) \cdot H(\pi_\theta(\cdot | s)) \left[ \tilde{Q}^{\pi_\theta}(s, \cdot) - \tau \theta(\cdot | s) \right]. 
    \qquad \left( H(\pi_\theta(\cdot | s)) \rvone = \rvzero \text{ in \cref{lem:golub_rank_one_perturb}} \right) 
\end{align}
For each component $a$, we have
\begin{align}
    \frac{\partial \tilde{V}^{\pi_\theta}(\mu)}{\partial \theta(s,a)} &= \frac{1}{1-\gamma} \cdot d_{\mu}^{\pi_\theta}(s) \cdot \pi_\theta(a|s) \cdot \left[ \tilde{Q}^{\pi_\theta}(s, a) - \tau \log{{\pi_\theta}(a | s)} - \sum_{a} \pi_\theta(a | s) \cdot \left[ \tilde{Q}^{\pi_\theta}(s, a) - \tau \log{{\pi_\theta}(a | s)} \right] \right] \\
    &= \frac{1}{1-\gamma} \cdot d_{\mu}^{\pi_\theta}(s) \cdot \pi_\theta(a|s) \cdot \left[ \tilde{Q}^{\pi_\theta}(s, a) - \tau \log{{\pi_\theta}(a | s)} - \tilde{V}^{\pi_\theta}(s) \right] \\
    &= \frac{1}{1-\gamma} \cdot d_{\mu}^{\pi_\theta}(s) \cdot \pi_\theta(a|s) \cdot  \tilde{A}^{\pi_\theta}(s,a). \qedhere
\end{align}
\end{proof}

\subsubsection{Proofs for bandits and non-uniform contraction}
\label{sec:proofs_entropy_policy_gradient_bandits}

\textbf{\cref{lem:contraction_entropy_special}} (Non-uniform contraction)\textbf{.}
Using \cref{update_rule:entropy_special} with $\tau \eta \le 1$, $\forall t \ge 1$,
\begin{align}
    \| \zeta_{t+1} \|_2 \le \left(1 - \tau \eta \cdot \min_{a}  \pi_{\theta_t}(a)  \right) \cdot \| \zeta_{t} \|_2,
\end{align}
where $\zeta_t = \tau \theta_{t} - r - \frac{(\tau \theta_{t} - r)^\top \rvone}{K} \cdot \rvone$.
\begin{proof}
\cref{update_rule:entropy_special} can be written as
\begin{align}
    \theta_{t+1} &= \theta_t - \eta \cdot H(\pi_{\theta_t}) (\tau \log{\pi_{\theta_t}} - r ) \\
    &= \theta_t - \eta \cdot H(\pi_{\theta_t}) \left[ \tau \theta_t - r - \left( \log \sum_{a} \exp\{ \theta_t(a) \} \right) \cdot \rvone \right] \\
    &= \theta_{t} - \eta \cdot H(\pi_{ \theta_{t} })( \tau \theta_{t} - r ) \\
    &= \theta_{t} - \eta \cdot H(\pi_{ \theta_{t} }) \left( \tau \theta_{t} - r - \frac{(\tau \theta_{t} - r)^\top \rvone }{K} \cdot \rvone \right),
\end{align}
where the last two equations are from $H(\pi_{ \theta_{t} } ) \rvone = \rvzero$ as shown in \cref{lem:golub_rank_one_perturb}. For all $t \ge 1$,
\begin{align}
    \zeta_{t+1} &= \tau \theta_{t+1} - r - \frac{(\tau \theta_{t+1} - r)^\top \rvone}{K} \cdot \rvone \\
    &= \tau \theta_{t} - r - \frac{(\tau \theta_{t} - r)^\top \rvone}{K} \cdot \rvone + \tau (\theta_{t+1} - \theta_{t}) + \left( \frac{(\tau \theta_{t} - r)^\top \rvone}{K} - \frac{(\tau \theta_{t+1} - r)^\top \rvone}{K} \right) \cdot \rvone \\
    &= \tau \theta_{t} - r - \frac{(\tau \theta_{t} - r)^\top \rvone}{K} \cdot \rvone + \tau (\theta_{t+1} - \theta_{t}) + \frac{\tau ( \theta_{t} - \theta_{t+1} )^\top \rvone}{K} \cdot \rvone.
\end{align}
For the last term,
\begin{align}
\label{eq:contraction_entropy_special_intermediate_1}
    \frac{\tau ( \theta_{t} -  \theta_{t+1} )^\top \rvone}{K} \cdot \rvone &= \frac{\tau}{K} \cdot \left( \eta \cdot H(\pi_{ \theta_{t} }) \left( \tau \theta_{t} - r - \frac{(\tau \theta_{t} - r)^\top \rvone }{K} \cdot \rvone \right) \right)^\top \rvone \cdot \rvone = \rvzero,
\end{align}
where the last equation is again by $H(\pi_{ \theta_{t} } )^\top \rvone = H(\pi_{ \theta_{t} } ) \rvone = \rvzero$. Using the update rule and combining the above,
\begin{align}
    \zeta_{t+1} &= \tau \theta_{t} - r - \frac{(\tau \theta_{t} - r)^\top \rvone}{K} \cdot \rvone + \tau (\theta_{t+1} - \theta_{t}) \\
    &= \left( \identitymatrix - \tau \eta \cdot H(\pi_{ \theta_{t} } ) \right) \left( \tau \theta_{t} - r - \frac{(\tau \theta_{t} - r)^\top \rvone}{K} \cdot \rvone \right) \\
    &= \left( \identitymatrix - \tau \eta \cdot H(\pi_{ \theta_{t} } ) \right) \zeta_t.
\end{align}
According to \cref{lem:norm_decay_entropy_special}, with $\tau \eta \le 1$,
\begin{align}
    \| \zeta_{t+1} \|_2 &= \left\| \left( \identitymatrix - \tau \eta \cdot H(\pi_{ \theta_{t} } ) \right) \zeta_t \right\|_2 \\
    &\le \left(1 - \tau \eta \cdot \min_{a}{ \pi_{\theta_t}(a) } \right) \cdot \| \zeta_{t} \|_2. \qedhere
\end{align}
\end{proof}

\textbf{\cref{lem:matching_entropy_special}.}
Let $\pi_{\theta_t} = \softmax(\theta_t)$. Using \cref{update_rule:entropy_special} with $\tau \eta \le 1$, $\forall t \ge 1$,
\begin{align}
    \| \zeta_{t} \|_2 \le \frac{ 2 ( \tau \norm{\theta_1}_\infty +1 ) \sqrt{K} }{\exp\left\{ \tau \eta \sum_{s=1}^{t-1}{ \min_{a}{ \pi_{\theta_s}(a) } }  \right\}}.
\end{align}
\begin{proof}
According to \cref{lem:contraction_entropy_special}, for all $t \ge 1$,
\begin{align}
    \| \zeta_{t+1} \|_2 &\le \left(1 - \tau \eta \cdot \min_{a}{ \pi_{\theta_t}(a) } \right) \cdot \| \zeta_{t} \|_2 \\
    &\le \frac{ 1 }{\exp\left\{ \tau \eta \cdot \min_{a}{ \pi_{\theta_t}(a) } \right\}} \cdot \| \zeta_{t} \|_2 \\
    &\le \frac{ 1 }{\exp\left\{ \tau \eta \cdot \min_{a}{ \pi_{\theta_t}(a) } \right\}} \cdot \left( 1- \tau \eta \cdot \min_{a}{ \pi_{ \theta_{t-1} }(a) } \right) \cdot \| \zeta_{t-1} \|_2 \\
    &\le \frac{ 1 }{\exp\left\{ \tau \eta \sum_{s=t-1}^{t}{ \min_{a}{ \pi_{\theta_s}(a) }  } \right\}} \cdot \| \zeta_{t-1} \|_2 \\
    &\le \frac{ 1 }{\exp\left\{ \tau \eta \sum_{s=1}^{t}{ \min_{a}{ \pi_{\theta_s}(a) } } \right\}} \cdot \| \zeta_{1} \|_2.
\end{align}
For the initial logit $\theta_1$, 
\begin{align}
\label{eq:zeta_1_upper_bound}
    \| \zeta_{1} \|_2 &= \left\| \tau \theta_{1} - r - \frac{(\tau \theta_{1} - r)^\top \rvone}{K} \cdot \rvone \right\|_2 \\
    &\le \| \tau \theta_{1} - r \|_2 + \left\| \frac{(\tau \theta_{1} - r)^\top \rvone}{K} \cdot \rvone \right\|_2 \qquad \left( \text{by triangle inequality} \right) \\
    &= \| \tau \theta_{1} - r \|_2 + \frac{ \left| (\tau \theta_{1} - r)^\top \rvone \right|}{\sqrt{K}} \\
    &\le \| \tau \theta_{1} - r \|_2 + \frac{ \| \tau \theta_{1} - r \|_2 \cdot \| \rvone \|_2}{\sqrt{K}} \qquad \left( \text{by Cauchy-Schwarz} \right) \\
    &= 2 \cdot \| \tau \theta_{1} - r \|_2 \\
    &\le 2 \cdot \left( \| \tau \theta_{1} \|_2 + \| r \|_2 \right) \\
    &\le 2 ( \tau \norm{\theta_1}_\infty +1 ) \sqrt{K}\,,
\end{align}
finishing the proof.
\end{proof}

\textbf{\cref{lem:lower_bound_min_prob_entropy_special}.}
There exists $c=c(\tau,K,\norm{\theta_1}_\infty)>0$, such that for all $t \ge 1$, $\min_{a}{ \pi_{\theta_t}(a) } \ge c$.
Thus, $\sum_{s=1}^{t-1}{ \min_{a}{ \pi_{\theta_s}(a) } } \ge c \cdot (t-1) $. 
\begin{proof}
Define the constant $c = c(\tau,K,\norm{\theta_1}_\infty)$ as
\begin{align}
    c = \frac{1}{K} \cdot \frac{1}{ \exp\{ 1 / \tau \}} \cdot \frac{1}{ \exp\{ 4 (\norm{\theta_1}_\infty + 1 / \tau ) \sqrt{K} \} }.
\end{align}
First, according to \cref{eq:zeta_1_upper_bound}, we have,
\begin{align}
    \| \zeta_{1} \|_2 \le 2 ( \tau \norm{\theta_1}_\infty +1 ) \sqrt{K}.
\end{align}
Next, according to \cref{lem:contraction_entropy_special}, with $\tau \eta \le 1$,
\begin{align}
    \| \zeta_{t+1} \|_2 \le \left(1 - \tau \eta \cdot \min_{a}{ \pi_{\theta_t}(a) } \right) \cdot \| \zeta_{t} \|_2 \le 2 ( \tau \norm{\theta_1}_\infty +1 ) \sqrt{K}.
\end{align}
Therefore, for all $t \ge 1$, we have,
\begin{align}
    \| \zeta_{t} \|_2 \le 2 ( \tau \norm{\theta_1}_\infty +1 ) \sqrt{K}.
\end{align}
We now prove $\min_{a}{ \pi_{\theta_t}(a) } \ge c$. We have, $\forall a$,
\begin{align}
    \left| \theta_{t}(a) - \frac{r(a)}{\tau} - \frac{( \theta_{t} - r / \tau )^\top \rvone}{K} \right| &= \frac{1}{\tau} \cdot \left| \tau \theta_{t}(a) - r(a) - \frac{( \tau  \theta_{t} - r )^\top \rvone}{K} \right| \\
    &\le \frac{1}{\tau} \cdot \left\| \tau \theta_{t} - r - \frac{(\tau \theta_{t} - r)^\top \rvone}{K} \cdot \rvone \right\|_2 \\
    &= \frac{1}{\tau} \cdot \left\| \zeta_t \right\|_2 \\
    &\le 2 (\norm{\theta_1}_\infty + 1 / \tau ) \sqrt{K}.
\end{align}
Denote $a_1 = \argmin_{a}{ \theta_t(a) } $, and $a_2 = \argmax_{a}{ \theta_t(a) } $. According to the above, we have the following results,
\begin{align}
    \theta_t(a_1) &\ge \frac{r(a_1)}{\tau} + \frac{( \tau  \theta_{t} - r )^\top \rvone}{K} - 2 (\norm{\theta_1}_\infty + 1 / \tau ) \sqrt{K}, \\
    - \theta_t(a_2) &\ge - \frac{r(a_2)}{\tau} - \frac{( \tau  \theta_{t} - r )^\top \rvone}{K} - 2 (\norm{\theta_1}_\infty + 1 / \tau ) \sqrt{K},
\end{align}
which can be used to lower bound the minimum probability as,
\begin{align}
    \min_{a}{ \pi_{ \theta_t }(a) } &= \frac{ \exp\{ \theta_t(a_1) \}}{ \sum_{a}{ \exp\{ \theta_t(a) \}} } \ge \frac{ \exp\{ \theta_t(a_1) \}}{ \sum_{a}{ \exp\{ \theta_t(a_2) \}} } = \frac{1}{K} \cdot \exp\left\{ \theta_t(a_1) - \theta_t(a_2) \right\}, \left( \text{since } \theta_t(a) \le \theta_t(a_2), \ \forall a \right)
\end{align}
which can be further lower bounded using the above results,
\begin{align}
\MoveEqLeft
    \min_{a}{ \pi_{ \theta_t }(a) } \ge \frac{1}{K} \cdot \exp\left\{ \theta_t(a_1) - \theta_t(a_2) \right\} \\
    &\ge \frac{1}{K} \cdot \exp\left\{ \frac{r(a_1)}{\tau} + \frac{( \tau  \theta_{t} - r )^\top \rvone}{K} - 2 (\norm{\theta_1}_\infty + 1 / \tau ) \sqrt{K} - \frac{r(a_2)}{\tau} - \frac{( \tau  \theta_{t} - r )^\top \rvone}{K} - 2 (\norm{\theta_1}_\infty + 1 / \tau ) \sqrt{K} \right\} \\
    &= \frac{1}{K} \cdot \exp\left\{ \frac{r(a_1) - r(a_2)}{\tau} - 4 (\norm{\theta_1}_\infty + 1 / \tau ) \sqrt{K} \right\} \\
    &\ge \frac{1}{K} \cdot \exp\left\{ - \frac{1}{\tau} - 4 (\norm{\theta_1}_\infty + 1 / \tau ) \sqrt{K} \right\} 
    \qquad \left( \text{because } r \in [0, 1]^K \text{ and } r(a_1) - r(a_2) \ge -1 \right) \\
    &= \frac{1}{K} \cdot \frac{1}{ \exp\{ 1 / \tau \}} \cdot \frac{1}{ \exp\{ 4 (\norm{\theta_1}_\infty + 1 / \tau ) \sqrt{K} \} } = c. \qedhere
\end{align}
\end{proof}

\textbf{\cref{thm:rates_entropy_special}.} Let $\pi_{\theta_t} = \softmax(\theta_t)$. Using \cref{update_rule:entropy_special} with $\eta \le 1 / \tau$, for all $t \ge 1$,
\begin{align}
    \left( \pi_\tau^* - \pi_{\theta_t} \right)^\top r &\le \frac{2 \sqrt{K} ( \norm{\theta_1}_\infty + 1 / \tau )}{\exp\left\{ \tau \eta \cdot c \cdot (t-1) \right\}}, \\
    \tilde{\delta}_t &\le \frac{ 2 ( \tau \norm{\theta_1}_\infty +1 )^2 K / \tau }{\exp\left\{ 2 \tau \eta \cdot c \cdot (t-1) \right\}},
\end{align}
where $\tilde{\delta}_t \coloneqq { \pi_\tau^* }^\top \left( r - \tau \log{ \pi_\tau^* } \right)  - { \pi_{\theta_t} }^\top \left( r - \tau \log{ \pi_{\theta_t} } \right)$ and $c>0$ is from \cref{lem:lower_bound_min_prob_entropy_special}.
\begin{proof}
According to H{\" o}lder's inequality,
\begin{align}
    \left( \pi_\tau^* - \pi_{\theta_t} \right)^\top r &\le \left\| \pi_\tau^* - \pi_{\theta_t} \right\|_1 \cdot \left\| r \right\|_\infty \qquad \left( \text{by H{\" o}lder's inequality} \right) \\
    &\le \left\| \pi_\tau^* - \pi_{\theta_t} \right\|_1  
    \qquad \qquad \left( \text{because $r \in [0, 1]^K$} \right) \\
    &\le \left\| \frac{r}{\tau} - \theta_t + \frac{(\tau \theta_{t} - r)^\top \rvone}{\tau K} \cdot \rvone \right\|_\infty \qquad \left( \text{by \cref{lem:policy_logit_inequality_special}} \right) \\
    &= \frac{1}{\tau} \cdot \left\| \tau \theta_t - r - \frac{(\tau \theta_{t} - r)^\top \rvone}{K} \cdot \rvone \right\|_\infty \\
    &\le \frac{1}{\tau} \cdot \left\| \tau \theta_t - r - \frac{(\tau \theta_{t} - r)^\top \rvone}{K} \cdot \rvone \right\|_2 \\
    &\le \frac{1}{\tau} \cdot \frac{ 2 ( \tau \norm{\theta_1}_\infty +1 ) \sqrt{K} }{\exp\left\{ \tau \eta \sum_{s=1}^{t-1}{ \min_{a}{ \pi_{\theta_s}(a) } }  \right\}} 
    \qquad \left( \text{by \cref{lem:matching_entropy_special}} \right) \\
    &\le \frac{2 \sqrt{K}}{\tau} \cdot \frac{ \tau \norm{\theta_1}_\infty +1 }{\exp\left\{ \tau \eta \cdot c \cdot (t-1) \right\}}. 
    \qquad \left( \text{by \cref{lem:lower_bound_min_prob_entropy_special}} \right)
\end{align}
On the other hand, we have,
\begin{align}
    { \pi_\tau^* }^\top \left( r - \tau \log{ \pi_\tau^* } \right)  - { \pi_{\theta_t} }^\top \left( r - \tau \log{ \pi_{\theta_t} } \right) &= { \pi_\tau^* }^\top \left( r - \tau \log{ \pi_\tau^* } \right)  - { \pi_{\theta_t} }^\top \left( r - \tau \log{ \pi_\tau^* } + \tau \log{ \pi_\tau^* } - \tau \log{ \pi_{\theta_t} } \right) \\
    &= \left( \pi_\tau^* - \pi_{\theta_t} \right)^\top \left( r - \tau \log{ \pi_\tau^* } \right) + \tau \cdot \KL( \pi_{\theta_t} \| \pi_\tau^* ) \\
    &= \left( \pi_\tau^* - \pi_{\theta_t} \right)^\top \rvone \cdot \tau \cdot \log{ \sum_{a}{ \exp\{ r(a) / \tau \} } } + \tau \cdot \KL( \pi_{\theta_t} \| \pi_\tau^* ) \\
    &= \tau \cdot \KL( \pi_{\theta_t} \| \pi_\tau^* ) \\
    &\le \frac{\tau}{2} \cdot \left\| \theta_t - \frac{r}{\tau} - \frac{(\tau \theta_t - r )^\top \rvone}{\tau K} \cdot \rvone \right\|_\infty^2 
    \qquad \left( \text{by \cref{lem:kl_logit_inequality}} \right)
    \\
    &= \frac{1}{2 \tau} \cdot \left\| \tau \theta_t - r - \frac{(\tau \theta_t - r )^\top \rvone}{K} \cdot \rvone \right\|_\infty^2 \\
    &\le \frac{1}{2 \tau} \cdot \left\| \tau \theta_t - r - \frac{(\tau \theta_t - r )^\top \rvone}{K} \cdot \rvone \right\|_2^2 \\
    &\le  \frac{1}{2 \tau} \cdot \frac{ 4 ( \tau \norm{\theta_1}_\infty +1 )^2 K }{\exp\left\{ 2 \tau \eta \sum_{s=1}^{t-1}{ \min_{a}{ \pi_{\theta_s}(a) } }  \right\}} 
    \qquad \left( \text{by \cref{lem:matching_entropy_special}} \right)
    \\
    &\le \frac{1}{ \tau} \cdot  \frac{ 2 ( \tau \norm{\theta_1}_\infty +1 )^2 K }{\exp\left\{ 2 \tau \eta \cdot c \cdot (t-1) \right\}}. 
    \qquad \left( \text{by \cref{lem:lower_bound_min_prob_entropy_special}} \right)
    \qedhere
\end{align}
\end{proof}

\subsubsection{Proofs for MDPs and entropy regularization}
\label{sec:proofs_entropy_policy_gradient_general_mdps}

\textbf{\cref{lem:smoothness_entropy_general}} (Smoothness)\textbf{.}
$\sH(\rho, \pi_\theta)$ is $(4 + 8 \log{A}) /(1-\gamma)^3$-smooth, where $A = | \gA |$ is the total number of actions.
\begin{proof}
Denote $\sH^{\pi_\theta}(s) = \sH(s, \pi_\theta)$. Also denote $\theta_\alpha = \theta + \alpha u $, where $\alpha \in \sR$ and $u \in \sR^{SA}$. According to \cref{eq:discounted_entropy},
\begin{align}
	\sH^{\pi_{\theta_\alpha}}(s) &= \expectation_{\substack{s_0 = s, a_t \sim \pi_{\theta_\alpha}(\cdot | s_t), \\ s_{t+1} \sim \gP( \cdot | s_t, a_t)}}{\left[ \sum_{t=0}^{\infty}{- \gamma^t \log{\pi_{\theta_\alpha}(a_t | s_t)}} \right]} \\
	&= - \sum_{a}{ \pi_{\theta_\alpha}(a | s) \cdot \log{  \pi_{\theta_\alpha}(a | s) } } + \gamma \sum_{a}{ \pi_{\theta_\alpha}(a | s) \sum_{s^\prime} { \gP(s^\prime | s, a) \cdot \sH^{\pi_{\theta_\alpha}}(s^\prime) } },
\end{align}
which implies,
\begin{align}
\label{eq:discounted_entropy_bellman_equation} 
    \sH^{\pi_{\theta_\alpha}}(s) = e_{s}^\top M(\alpha) h_{\theta_\alpha},
\end{align}
where $M(\alpha) = \left( \identitymatrix - \gamma P(\alpha) \right)^{-1}$ is defined in \cref{eq:smoothness_softmax_general_intermediate_M_matrix_def}, $P(\alpha)$ is defined in \cref{eq:smoothness_softmax_general_intermediate_P_PI_def}, and $h_{\theta_\alpha} \in \sR^{S}$
for $s\in \gS$ is given by
\begin{align}
\label{eq:one_step_entropy} 
    h_{\theta_\alpha}(s) = - \sum_{a}{ \pi_{\theta_\alpha}(a | s) \cdot \log{  \pi_{\theta_\alpha}(a | s) } }.
\end{align}
According to \cref{eq:one_step_entropy}, $h_{\theta_\alpha}(s) \in [0, \log{A}]$, $\forall s$. Then we have,
\begin{align}
\label{eq:smoothness_entropy_general_intermediate_hpi_upper_bound}
    \| h_{\theta_\alpha} \|_\infty = \max_{s}{ \left| h_{\theta_\alpha}(s) \right| } \le \log{A}.
\end{align}
For any state $s \in \gS$,
\begin{align}
    \left| \frac{\partial h_{\theta_\alpha}(s)}{\partial \alpha} \right| &= \left| \Big\langle \frac{\partial h_{\theta_\alpha}(s)}{\partial \theta_\alpha}, \frac{\partial \theta_\alpha}{\partial \alpha} \Big\rangle \right| \\
    &= \left| \Big\langle \frac{\partial h_{\theta_\alpha}(s)}{\partial \theta_\alpha(\cdot | s)}, u(s, \cdot) \Big\rangle \right| \\
    &= \left| \left( H(\pi_{\theta_\alpha}(\cdot | s)) \log{ \pi_{\theta_\alpha}(\cdot | s) } \right)^\top u(s, \cdot) \right| \\
    &\le \left\| H(\pi_{\theta_\alpha}(\cdot | s)) \log{ \pi_{\theta_\alpha}(\cdot | s) } \right\|_1 \cdot \left\| u(s, \cdot) \right\|_\infty.
\end{align}
The $\ell_1$ norm is upper bounded as
\begin{align}
\label{eq:H_matrix_log_pi_1_norm}
    \left\| H(\pi_{\theta_\alpha}(\cdot | s)) \log{ \pi_{\theta_\alpha}(\cdot | s) } \right\|_1 &= \sum_{a}{ \pi_{\theta_\alpha}(a | s) \cdot \left| \log{ \pi_{\theta_\alpha}(a | s) } - \pi_{\theta_\alpha}(\cdot | s)^\top \log{ \pi_{\theta_\alpha}(\cdot | s) } \right|} \\
    &\le \sum_{a}{ \pi_{\theta_\alpha}(a | s) \cdot \left( \left| \log{ \pi_{\theta_\alpha}(a | s) } \right| + \left| \pi_{\theta_\alpha}(\cdot | s)^\top \log{ \pi_{\theta_\alpha}(\cdot | s) } \right|  \right) } \\
    &= - 2 \cdot \sum_{a}{ \pi_{\theta_\alpha}(a | s) \cdot \log{ \pi_{\theta_\alpha}(a | s) } } \le 2 \cdot \log{A}.
\end{align}
Therefore we have,
\begin{align}
\label{eq:smoothness_entropy_general_intermediate_hpi_first_derivative_upper_bound}
    \left\| \frac{\partial h_{\theta_\alpha}}{\partial \alpha} \right\|_\infty &= \max_{s}{ \left| \frac{\partial h_{\theta_\alpha}(s)}{\partial \alpha} \right| } \\
    &\le \max_{s}{ \left\| H(\pi_{\theta_\alpha}(\cdot | s)) \log{ \pi_{\theta_\alpha}(\cdot | s) } \right\|_1 \cdot \left\| u(s, \cdot) \right\|_\infty } \\
    &\le 2 \cdot \log{A} \cdot \| u \|_2.
\end{align}
The second derivative w.r.t. $\alpha$ is
\begin{align}
    \left| \frac{\partial^2 h_{\theta_\alpha}(s)}{\partial \alpha^2} \right| &= \left| \left( \frac{\partial}{\partial \theta_\alpha} \left\{ \frac{\partial h_{\theta_\alpha}(s)}{\partial \alpha} \right\} \right)^\top \frac{\partial \theta_\alpha}{\partial \alpha}  \right| \\
    &= \left| \left( \frac{\partial^2 h_{\theta_\alpha}(s)}{\partial \theta_\alpha^2} \frac{\partial \theta_\alpha}{\partial \alpha} \right)^\top \frac{\partial \theta_\alpha}{\partial \alpha} \right| \\
    &= \left| u(s, \cdot)^\top \frac{\partial^2 h_{\theta_\alpha}(s)}{\partial \theta_\alpha^2(s, \cdot)} u(s, \cdot) \right|.
\end{align}
Denote the Hessian $T(s, \theta_\alpha) = \frac{\partial^2 h_{\theta_\alpha}(s)}{\partial \theta^2(s, \cdot)}$.
Then,
\begin{align}
    T(s, \theta_\alpha) = \frac{\partial^2 h_{\theta_\alpha}(s)}{\partial \theta_\alpha^2(s, \cdot)} &= \frac{\partial}{\partial \theta_\alpha(s, \cdot)} \left\{ \frac{\partial h_{\theta_\alpha}(s)}{\partial \theta_\alpha(s, \cdot)} \right\} \\
    &= \frac{\partial}{\partial \theta_\alpha(s, \cdot)} \left\{ \left( \frac{ \partial \pi_{\theta_\alpha}(\cdot | s) }{ \partial \theta_\alpha( s, \cdot )}  \right)^\top \frac{\partial h_{\theta_\alpha}(s)}{\partial \pi_{\theta_\alpha}( \cdot | s)} \right\} \\
    &= \frac{\partial}{\partial \theta_\alpha(s, \cdot)} \left\{  H( \pi_{\theta_\alpha}( \cdot | s) ) ( - \log{ \pi_{\theta_\alpha}( \cdot | s)} ) \right\}.
\end{align}
Note $T(s, \theta_\alpha) \in \sR^{A \times A}$, and $\forall i, j \in \gA$, the value of $T(s, \theta_\alpha)$ is,
\begin{align}
    T_{i, j} &= \frac{ d \{ \pi_{\theta_\alpha}(i | s) \cdot ( - \log{ \pi_{\theta_\alpha}(i | s) } - h_{\theta_\alpha}(s) ) \} }{d \theta_\alpha(s, j)} \\
    &= \frac{d \pi_{\theta_\alpha}(i | s) }{d \theta_\alpha(s, j) } \cdot ( - \log{ \pi_{\theta_\alpha}(i | s) } -  h_{\theta_\alpha}(s)  ) + \pi_{\theta_\alpha}(i | s) \cdot \frac{d \{ - \log{ \pi_{\theta_\alpha}(i | s) } - h_{\theta_\alpha}(s) \} }{d \theta_\alpha(s, j) } \\
    &= (\delta_{ij} \pi_{\theta_\alpha}(j | s) -  \pi_{\theta_\alpha}(i | s) \pi_{\theta_\alpha}(j | s) ) \cdot ( - \log{ \pi_{\theta_\alpha}(i | s) } -  h_{\theta_\alpha}(s)) \\
    &\qquad + \pi_{\theta_\alpha}(i | s) \cdot \left( - \frac{1}{ \pi_{\theta_\alpha}(i | s)  } \cdot \left( \delta_{ij} \pi_{\theta_\alpha}(j | s) - \pi_{\theta_\alpha}(i | s) \pi_{\theta_\alpha}(j | s) \right) - \pi_{\theta_\alpha}(j | s) \cdot \left( - \log{ \pi_{\theta_\alpha}(j | s) } -  h_{\theta_\alpha}(s) \right) \right) \\
    &= \delta_{ij} \pi_{\theta_\alpha}(j | s) \cdot ( - \log{ \pi_{\theta_\alpha}(i | s) } -  h_{\theta_\alpha}(s) - 1 ) - \pi_{\theta_\alpha}(i | s) \pi_{\theta_\alpha}(j | s) \cdot ( - \log{ \pi_{\theta_\alpha}(i | s) } -  h_{\theta_\alpha}(s) - 1) \\
    &\qquad - \pi_{\theta_\alpha}(i | s) \pi_{\theta_\alpha}(j | s) \cdot ( - \log{ \pi_{\theta_\alpha}(j | s) } -  h_{\theta_\alpha}(s) ).
\end{align}
For any vector $y \in \sR^{A}$,
\begin{align}
    \left| y^\top T(s, \theta_\alpha) y \right| &= \left| \sum\limits_{i=1}^{A}{ \sum\limits_{j=1}^{A}{ T_{i,j} y(i) y(j)} } \right| \\
    &\le \left| \sum_{i}{ \pi_{\theta_\alpha}(i | s) \cdot ( - \log{ \pi_{\theta_\alpha}(i | s) } -  h_{\theta_\alpha}(s) - 1 ) \cdot y(i)^2 } \right| \\
    &\qquad + 2 \cdot \left| \sum_{i} \pi_{\theta_\alpha}(i | s) \cdot y(i) \sum_{j}\pi_{\theta_\alpha}(j | s) \cdot ( - \log{ \pi_{\theta_\alpha}(j | s) } -  h_{\theta_\alpha}(s) ) \cdot y(j) \right| + \left( \pi_{\theta_\alpha}(\cdot | s)^\top y \right)^2 \\
    &= \left| \left( H( \pi_{\theta_\alpha}(\cdot | s) ) ( - \log{ \pi_{\theta_\alpha}(\cdot | s) }) - \pi_{\theta_\alpha}(\cdot | s) \right)^\top \left( y \odot y \right) \right| \\
    &\qquad + 2 \cdot \left| \left( \pi_{\theta_\alpha}(\cdot | s)^\top y \right) \cdot \left( H( \pi_{\theta_\alpha}(\cdot | s) ) ( - \log{ \pi_{\theta_\alpha}(\cdot | s) }) \right)^\top y \right| + \left( \pi_{\theta_\alpha}(\cdot | s)^\top y \right)^2 \\
    &\le \left\| H( \pi_{\theta_\alpha}(\cdot | s) ) ( - \log{ \pi_{\theta_\alpha}(\cdot | s) }) \right\|_\infty \cdot \left\| y \odot y \right\|_1 + \| \pi_{\theta_\alpha}(\cdot | s) \|_\infty \cdot \left\| y \odot y \right\|_1 \\
    &\qquad + 2 \cdot \| \pi_{\theta_\alpha}(\cdot | s) \|_1 \cdot \| y \|_\infty \cdot \left\| H( \pi_{\theta_\alpha}(\cdot | s) ) ( - \log{ \pi_{\theta_\alpha}(\cdot | s) }) \right\|_1 \cdot \| y \|_\infty + \| \pi_{\theta_\alpha}(\cdot | s) \|_2^2 \cdot \| y \|_2^2,
\end{align}
where the last inequality is by  H{\" o}lder's inequality. Note that $\| y \odot y \|_1 = \| y \|_2^2$, $\| \pi_{\theta_\alpha}(\cdot | s) \|_\infty \le \| \pi_{\theta_\alpha}(\cdot | s) \|_1$, $\| \pi_{\theta_\alpha}(\cdot | s) \|_2 \le \| \pi_{\theta_\alpha}(\cdot | s) \|_1 = 1$, and $\| y \|_\infty \le \| y \|_2$. The $\ell_\infty$ norm is upper bounded as
\begin{align}
\label{eq:H_matrix_log_pi_infty_norm}
    \left\| H( \pi_{\theta_\alpha}(\cdot | s) ) ( - \log{ \pi_{\theta_\alpha}(\cdot | s) }) \right\|_\infty &= \max_{a}{ \left| \pi_{\theta_\alpha}(a | s)  \cdot \left( - \log{ \pi_{\theta_\alpha}(a | s) }  + \pi_{\theta_\alpha}(\cdot | s)^\top \log{ \pi_{\theta_\alpha}(\cdot | s) } \right) \right| } \\
    &\le \max_{a}{ - \pi_{\theta_\alpha}(a | s)  \cdot \log{ \pi_{\theta_\alpha}(a | s) } } - \pi_{\theta_\alpha}(\cdot | s)^\top \log{ \pi_{\theta_\alpha}(\cdot | s) } \\
    &\le \frac{1}{e} + \log{A}. \qquad \left( \text{since } - x \cdot \log{ x } \le \frac{1}{e} \text{ for all } x \in [0, 1] \right)
\end{align}
Therefore we have,
\begin{align}
    \left| y^\top T(s, \theta_\alpha) y \right| &\le \left\| H( \pi_{\theta_\alpha}(\cdot | s) ) ( - \log{ \pi_{\theta_\alpha}(\cdot | s) }) \right\|_\infty \cdot \| y \|_2^2 \\
    &\qquad + \| y \|_2^2 + 2 \cdot \left\| H( \pi_{\theta_\alpha}(\cdot | s) ) ( - \log{ \pi_{\theta_\alpha}(\cdot | s) }) \right\|_1 \cdot \| y \|_2^2 + \| y \|_2^2 \\
    &\le \left( \frac{1}{e} + \log{A} + 2 \right) \cdot \| y \|_2^2 + 2 \cdot \left\| H( \pi_{\theta_\alpha}(\cdot | s) ) ( - \log{ \pi_{\theta_\alpha}(\cdot | s) }) \right\|_1 \cdot \| y \|_2^2 \qquad \left( \text{by \cref{eq:H_matrix_log_pi_infty_norm}} \right) \\
    &\le \left( \frac{1}{e} + \log{A} + 2 + 2 \cdot \log{A} \right) \cdot \| y \|_2^2 \qquad \left( \text{by \cref{eq:H_matrix_log_pi_1_norm}} \right) \\
    &\le 3 \cdot \left( 1 + \log{A} \right) \cdot \| y \|_2^2.
\end{align}
According to the above results,
\begin{align}
\label{eq:smoothness_entropy_general_intermediate_hpi_second_derivative_upper_bound}
    \left\| \frac{\partial^2 h_{\theta_\alpha}}{\partial \alpha^2} \right\|_\infty &= \max_{s}{ \left| \frac{\partial^2 h_{\theta_\alpha}(s)}{\partial \alpha^2} \right| } \\
    &= \max_{s}{ \left| u(s, \cdot)^\top \frac{\partial^2 h_{\theta_\alpha}(s)}{\partial \theta_\alpha^2(s, \cdot)} u(s, \cdot) \right|}  \\
    &= \max_{s}{ \left| u(s, \cdot)^\top T(s, \theta_\alpha) u(s, \cdot) \right| } \\
    &\le  3 \cdot \left( 1 + \log{A} \right) \cdot \max_{s}{ \| u(s, \cdot) \|_2^2 } \\
    &\le 3 \cdot \left( 1 + \log{A} \right) \cdot \| u \|_2^2.
\end{align}
Taking derivative w.r.t. $\alpha$ in \cref{eq:discounted_entropy_bellman_equation},
\begin{align}
    \frac{\partial \sH^{\pi_{\theta_\alpha}}(s)}{\partial \alpha} = \gamma \cdot e_{s}^\top M(\alpha) \frac{\partial P(\alpha)}{\partial \alpha} M(\alpha) h_{\theta_\alpha} + e_{s}^\top M(\alpha) \frac{\partial h_{\theta_\alpha}}{\partial \alpha}.
\end{align}
Taking second derivative w.r.t. $\alpha$,
\begin{align}
\label{eq:smoothness_entropy_general_intermediate_H_second_derivative_def}
    \frac{\partial^2 \sH^{\pi_{\theta_\alpha}}(s)}{\partial \alpha^2} &= 2 \gamma^2 \cdot e_{s}^\top M(\alpha) \frac{\partial P(\alpha)}{\partial \alpha} M(\alpha) \frac{\partial P(\alpha)}{\partial \alpha} M(\alpha) h_{\theta_\alpha} + \gamma \cdot e_{s}^\top M(\alpha) \frac{\partial^2 P(\alpha)}{\partial \alpha^2} M(\alpha) h_{\theta_\alpha} \\
    &\qquad + 2 \gamma \cdot e_{s}^\top M(\alpha) \frac{\partial P(\alpha)}{\partial \alpha} M(\alpha) \frac{\partial h_{\theta_\alpha}}{\partial \alpha} + e_{s}^\top M(\alpha) \frac{\partial^2 h_{\theta_\alpha}}{\partial \alpha^2}.
\end{align}
For the last term,
\begin{align}
\label{eq:smoothness_entropy_general_intermediate_1}
    \left| e_{s}^\top M(\alpha) \frac{\partial^2 h_{\theta_\alpha}}{\partial \alpha^2} \Big|_{\alpha=0} \right| &\le \left\| e_{s} \right\|_1 \cdot \left\| M(\alpha) \frac{\partial^2 h_{\theta_\alpha}}{\partial \alpha^2} \Big|_{\alpha=0} \right\|_\infty \\
    &\le \frac{1}{1 - \gamma} \cdot \left\| \frac{\partial^2 h_{\theta_\alpha}}{\partial \alpha^2} \Big|_{\alpha=0} \right\|_\infty \qquad \left( \text{by \cref{eq:smoothness_softmax_general_intermediate_M_matrix_norm}} \right) \\
    &\le \frac{3 \cdot (1 + \log{A})}{1 - \gamma} \cdot \| u \|_2^2. 
    \qquad \left( \text{by \cref{eq:smoothness_entropy_general_intermediate_hpi_second_derivative_upper_bound}} \right)
\end{align}
For the second last term,
\begin{align}
\label{eq:smoothness_entropy_general_intermediate_2} 
    \left| e_{s}^\top M(\alpha) \frac{\partial P(\alpha)}{\partial \alpha} M(\alpha) \frac{\partial h_{\theta_\alpha}}{\partial \alpha} \Big|_{\alpha=0} \right| &\le \left\| M(\alpha) \frac{\partial P(\alpha)}{\partial \alpha} M(\alpha) \frac{\partial h_{\theta_\alpha}}{\partial \alpha} \Big|_{\alpha=0} \right\|_\infty \\
    &\le \frac{1}{1 - \gamma} \cdot \left\| \frac{\partial P(\alpha)}{\partial \alpha} M(\alpha) \frac{\partial h_{\theta_\alpha}}{\partial \alpha} \Big|_{\alpha=0} \right\|_\infty \qquad \left( \text{by \cref{eq:smoothness_softmax_general_intermediate_M_matrix_norm}} \right) \\
    &\le \frac{2 \cdot \| u \|_2 }{1 - \gamma} \cdot \left\| M(\alpha) \frac{\partial h_{\theta_\alpha}}{\partial \alpha} \Big|_{\alpha=0} \right\|_\infty \qquad \left( \text{by \cref{eq:smoothness_softmax_general_intermediate_P_PI_first}} \right) \\
    &\le \frac{2 \cdot \| u \|_2 }{(1 - \gamma)^2} \cdot \left\|  \frac{\partial h_{\theta_\alpha}}{\partial \alpha} \Big|_{\alpha=0} \right\|_\infty \qquad \left( \text{by \cref{eq:smoothness_softmax_general_intermediate_M_matrix_norm}} \right) \\
    &\le \frac{2 \cdot \| u \|_2 }{(1 - \gamma)^2} \cdot 2 \cdot \log{A} \cdot \| u \|_2 = \frac{4 \cdot \log{A} }{(1 - \gamma)^2} \cdot \| u \|_2^2. 
    \qquad \left( \text{by \cref{eq:smoothness_entropy_general_intermediate_hpi_first_derivative_upper_bound}} \right)
\end{align}
For the second term,
\begin{align}
\label{eq:smoothness_entropy_general_intermediate_3}
    \left| e_{s}^\top M(\alpha) \frac{\partial^2 P(\alpha)}{\partial \alpha^2} M(\alpha) h_{\theta_\alpha} \Big|_{\alpha=0} \right| &\le \left\| M(\alpha) \frac{\partial^2 P(\alpha)}{\partial \alpha^2} M(\alpha) h_{\theta_\alpha} \Big|_{\alpha=0} \right\|_\infty \\
    &\le \frac{1}{1 - \gamma} \cdot \left\|  \frac{\partial^2 P(\alpha)}{\partial \alpha^2} M(\alpha) h_{\theta_\alpha} \Big|_{\alpha=0} \right\|_\infty \qquad \left( \text{by \cref{eq:smoothness_softmax_general_intermediate_M_matrix_norm}} \right) \\
    &\le \frac{6 \cdot \| u \|_2^2}{1 - \gamma} \cdot \left\| M(\alpha) h_{\theta_\alpha} \Big|_{\alpha=0} \right\|_\infty \qquad \left( \text{by \cref{eq:smoothness_softmax_general_intermediate_P_PI_second}} \right) \\
    &\le \frac{6 \cdot \| u \|_2^2}{(1 - \gamma)^2} \cdot \left\|  h_{\theta_\alpha} \Big|_{\alpha=0} \right\|_\infty \qquad \left( \text{by \cref{eq:smoothness_softmax_general_intermediate_M_matrix_norm}} \right) \\
    &\le \frac{6 \cdot \log{A} }{(1 - \gamma)^2} \cdot \| u \|_2^2. \qquad \left( \text{by \cref{eq:smoothness_entropy_general_intermediate_hpi_upper_bound}} \right)
\end{align}
For the first term, according to \cref{eq:smoothness_softmax_general_intermediate_P_PI_first,eq:smoothness_softmax_general_intermediate_M_matrix_norm}, \cref{eq:smoothness_entropy_general_intermediate_hpi_upper_bound},
\begin{align}
\label{eq:smoothness_entropy_general_intermediate_4}
    \left| e_{s}^\top M(\alpha) \frac{\partial P(\alpha)}{\partial \alpha} M(\alpha) \frac{\partial P(\alpha)}{\partial \alpha} M(\alpha) h_{\theta_\alpha} \Big|_{\alpha=0} \right| &\le \left\| M(\alpha) \frac{\partial P(\alpha)}{\partial \alpha} M(\alpha) \frac{\partial P(\alpha)}{\partial \alpha} M(\alpha) h_{\theta_\alpha} \Big|_{\alpha=0} \right\|_\infty \\ &\le \frac{1}{1-\gamma} \cdot 2 \cdot \| u \|_2 \cdot \frac{1}{1-\gamma} \cdot 2 \cdot \| u \|_2 \cdot \frac{1}{1-\gamma} \cdot \log{A} \\
    &= \frac{4 \cdot \log{A} }{(1 - \gamma)^3} \cdot \| u \|_2^2.
\end{align}
Combining \cref{eq:smoothness_entropy_general_intermediate_1,eq:smoothness_entropy_general_intermediate_2,eq:smoothness_entropy_general_intermediate_3,eq:smoothness_entropy_general_intermediate_4} with \cref{eq:smoothness_entropy_general_intermediate_H_second_derivative_def},
\begin{align}
\label{eq:smoothness_entropy_general_intermediate_5}
    \left| \frac{\partial^2 \sH^{\pi_{\theta_\alpha}}(s)}{\partial \alpha^2} \Big|_{\alpha=0} \right| &\le 2 \gamma^2 \cdot \left| e_{s}^\top M(\alpha) \frac{\partial P(\alpha)}{\partial \alpha} M(\alpha) \frac{\partial P(\alpha)}{\partial \alpha} M(\alpha) h_{\theta_\alpha} \Big|_{\alpha=0} \right| + \gamma \cdot \left| e_{s}^\top M(\alpha) \frac{\partial^2 P(\alpha)}{\partial \alpha^2} M(\alpha) h_{\theta_\alpha} \Big|_{\alpha=0} \right| \\
    &\qquad + 2 \gamma \cdot \left| e_{s}^\top M(\alpha) \frac{\partial P(\alpha)}{\partial \alpha} M(\alpha) \frac{\partial h_{\theta_\alpha}}{\partial \alpha} \Big|_{\alpha=0} \right| + \left| e_{s}^\top M(\alpha) \frac{\partial^2 h_{\theta_\alpha}}{\partial \alpha^2} \Big|_{\alpha=0} \right| \\
    &\le \left( 2 \gamma^2 \cdot \frac{4 \cdot \log{A}}{(1-\gamma)^3} + \gamma \cdot \frac{6 \cdot \log{A}}{(1-\gamma)^2} + 2 \gamma \cdot \frac{4 \cdot \log{A }}{(1-\gamma)^2} + \frac{3 \cdot (1 + \log{A})}{1-\gamma}  \right) \cdot \| u \|_2^2 \\
    &\le \left( \frac{8 \cdot \log{A}}{(1-\gamma)^3} + \frac{3}{1-\gamma} \right) \cdot \| u \|_2^2 \\
    &\le \frac{4 + 8 \cdot \log{A}}{(1-\gamma)^3}  \cdot \| u \|_2^2,
\end{align}
which implies for all $y \in \sR^{S A}$ and $\theta$,
\begin{align}
\label{eq:smoothness_entropy_general_intermediate_6}
    \left| y^\top \frac{\partial^2 \sH^{\pi_\theta}(s)}{\partial \theta^2} y \right| &= \left| \left(\frac{y}{ \| y \|_2 } \right)^\top \frac{\partial^2 \sH^{\pi_\theta}(s)}{\partial \theta^2} \left(\frac{y}{ \| y \|_2 } \right) \right| \cdot \| y \|_2^2 \\
    &\le \max_{\| u \|_2 = 1}{ \left| \Big\langle \frac{\partial^2 \sH^{\pi_\theta}(s)}{\partial \theta^2} u , u \Big\rangle \right| } \cdot \| y \|_2^2 \\
    &= \max_{\| u \|_2 = 1}{ \left| \Big\langle \frac{\partial^2 \sH^{\pi_{\theta_\alpha}}(s)}{\partial {\theta_\alpha^2}} \Big|_{\alpha=0} \frac{\partial \theta_\alpha}{\partial \alpha}, \frac{\partial \theta_\alpha}{\partial \alpha} \Big\rangle \right| } \cdot \| y \|_2^2 \\
    &= \max_{\| u \|_2 = 1}{ \left| \Big\langle \frac{\partial}{\partial \theta_\alpha} \left\{ \frac{\partial \sH^{\pi_{\theta_\alpha}}(s)}{\partial \alpha} \right\} \Big|_{\alpha = 0}, \frac{\partial \theta_\alpha}{\partial \alpha} \Big\rangle \right| } \cdot \| y \|_2^2 \\
    &= \max_{\| u \|_2 = 1}{ \left| \frac{\partial^2 \sH^{\pi_{\theta_\alpha}}(s)}{\partial \alpha^2  } \Big|_{\alpha=0} \right| } \cdot \| y \|_2^2 \\
    &\le \frac{4 + 8 \cdot \log{A}}{(1-\gamma)^3} \cdot \| y \|_2^2. \qquad \left( \text{by \cref{eq:smoothness_entropy_general_intermediate_5}} \right)
\end{align}
Denote $\theta_{\xi} = \theta + \xi ( \theta^\prime - \theta )$, where $\xi \in [0,1]$. According to Taylor's theorem, $\forall s$, $\forall \theta, \ \theta^\prime$,
\begin{align}
    \left| \sH^{\pi_{\theta^\prime}}(s) - \sH^{\pi_{\theta}}(s) - \Big\langle \frac{\partial \sH^{\pi_\theta}(s)}{\partial \theta}, \theta^\prime - \theta \Big\rangle \right| &= \frac{1}{2} \cdot \left| \left( \theta^\prime - \theta \right)^\top \frac{\partial^2 \sH^{\pi_{\theta_\xi}}(s)}{\partial \theta_\xi^2} \left( \theta^\prime - \theta \right) \right| \\
    &\le \frac{2 + 4 \cdot \log{A}}{(1-\gamma)^3} \cdot \| \theta^\prime - \theta \|_2^2. \qquad \left( \text{by \cref{eq:smoothness_entropy_general_intermediate_6}} \right)
\end{align}
Since $\sH^{\pi_\theta}(s)$ is $(4 + 8 \log{A}) /(1-\gamma)^3$-smooth, $\forall s$, $\sH(\rho, \pi_\theta) = \expectation_{s \sim \rho}{ \left[ \sH^{\pi_\theta}(s) \right]}$ is also $(4 + 8 \log{A}) /(1-\gamma)^3$-smooth.
\end{proof}

\textbf{\cref{lem:lojasiewicz_entropy_general}} (Non-uniform \L{}ojasiewicz)\textbf{.}
Suppose $\mu(s) > 0$ for all states $s \in \gS$
and $\pi_\theta(\cdot | s) = \softmax(\theta(s, \cdot))$. Then,
\begin{align}
    \left\| \frac{\partial \tilde{V}^{{\pi_\theta}}(\mu) }{\partial \theta} \right\|_2 \ge \frac{\sqrt{2 \tau}}{\sqrt{S}} \cdot \min_{s}{\sqrt{ \mu(s) } } \cdot \min_{s,a}{ \pi_\theta(a | s)  } \cdot \left\| \frac{d_{\rho}^{\pi_\tau^*} }{ d_{\mu}^{\pi_\theta}} \right\|_\infty^{-\frac{1}{2}} \cdot \left[ \tilde{V}^{\pi_\tau^*}(\rho) - \tilde{V}^{{\pi_\theta}}(\rho) \right]^{\frac{1}{2}}.
\end{align}
\begin{proof}
According to the definition of soft value functions,
\begin{align}
\MoveEqLeft
    \tilde{V}^{\pi_\tau^*}(\rho) - \tilde{V}^{{\pi_\theta}}(\rho) = \expectation_{\substack{s_0 \sim \rho, a_t \sim {\pi_\tau^*}(\cdot | s_t), \\ s_{t+1} \sim \gP( \cdot | s_t, a_t)}}{\left[ \sum_{t=0}^{\infty}{\gamma^t ( r(s_t, a_t) - \tau \log{\pi_\tau^*(a_t | s_t)} ) } \right]} - \tilde{V}^{{\pi_\theta}}(\rho) \\
    &= \expectation_{\substack{s_0 \sim \rho, a_t \sim {\pi_\tau^*}(\cdot | s_t), \\ s_{t+1} \sim \gP( \cdot | s_t, a_t)}}{\left[ \sum_{t=0}^{\infty}{\gamma^t ( r(s_t, a_t) - \tau \log{\pi_\tau^*(a_t | s_t)} + \tilde{V}^{{\pi_\theta}}(s_t) - \tilde{V}^{{\pi_\theta}}(s_t) ) } \right]} - \tilde{V}^{{\pi_\theta}}(\rho) \\
    &= \expectation_{\substack{s_0 \sim \rho, a_t \sim {\pi_\tau^*}(\cdot | s_t), \\ s_{t+1} \sim \gP( \cdot | s_t, a_t)}}{\left[ \sum_{t=0}^{\infty}{\gamma^t ( r(s_t, a_t) - \tau \log{\pi_\tau^*(a_t | s_t)} + \gamma \tilde{V}^{{\pi_\theta}}(s_{t+1}) - \tilde{V}^{{\pi_\theta}}(s_t) ) } \right]} \\
    &= \frac{1}{1 - \gamma} \sum_{s}{ d_{\rho}^{\pi_\tau^*}(s) \cdot { \left[ \sum_{a}{ \pi_\tau^*(a | s) \cdot \left( r(s,a) - \tau \log{ \pi_\tau^*(a | s) + \gamma \sum_{s^\prime}{\gP(s^\prime|s,a) } \tilde{V}^{{\pi_\theta}}(s^\prime) - \tilde{V}^{{\pi_\theta}}(s) } \right)} \right] } } \\
    &= \frac{1}{1 - \gamma} \sum_{s}{ d_{\rho}^{\pi_\tau^*}(s) \cdot { \left[ \sum_{a}{ \pi_\tau^*(a | s) \cdot \left[ \tilde{Q}^{\pi_\theta}(s,a) - \tau \log{ \pi_\tau^*(a | s) } \right]} - \tilde{V}^{{\pi_\theta}}(s) \right] } }.
\end{align}
Next, define the ``soft greedy policy" $\bar{\pi}_{\theta}(\cdot | s) = \softmax(\tilde{Q}^{\pi_\theta}(s, \cdot) / \tau)$, $\forall s$, i.e.,
\begin{align}
\label{eq:intermediate_justification_sac_1}
    \bar{\pi}_{\theta}(a | s) = \frac{\exp\left\{ \tilde{Q}^{\pi_\theta}(s, a) / \tau \right\}}{ \sum_{a^\prime}{ \exp\left\{ \tilde{Q}^{\pi_\theta}(s, a^\prime) / \tau \right\}  }}, \ \forall a.
\end{align}
We have, $\forall s$,
\begin{align}
    \sum_{a}{ \pi_\tau^*(a | s) \cdot \left[ \tilde{Q}^{\pi_\theta}(s,a) - \tau \log{ \pi_\tau^*(a | s) } \right]} &\le \max_{\pi(\cdot | s)}{ \sum_{a}{ \pi(a | s) \cdot \left[ \tilde{Q}^{\pi_\theta}(s,a) - \tau \log{ \pi(a | s) } \right]} } \\
    &= \sum_{a}{ \bar{\pi}_{\theta}(a | s) \cdot \left[ \tilde{Q}^{\pi_\theta}(s,a) - \tau \log{ \bar{\pi}_{\theta}(a | s) } \right]} \\
    &= \tau \log \sum_{a} \exp\left\{ \tilde{Q}^{\pi_\theta}(s,a) / \tau \right\}.
\end{align}
Also note that,
\begin{align}
    \tilde{V}^{{\pi_\theta}}(s) &= \sum_{a}{ \pi_\theta(a | s) \cdot \left[ \tilde{Q}^{\pi_\theta}(s,a) - \tau \log{ \pi_\theta(a | s) } \right] } \\
    &= \sum_{a}{ \pi_\theta(a | s) \cdot \left[ \tilde{Q}^{\pi_\theta}(s,a) - \tau \log{ \bar{\pi}_{\theta}(a | s) } + \tau \log{ \bar{\pi}_{\theta}(a | s) } - \tau \log{ \pi_\theta(a | s) } \right] } \\
    &= \sum_{a}{ \pi_\theta(a | s) \cdot \left[ \tilde{Q}^{\pi_\theta}(s,a) - \tau \log{ \bar{\pi}_{\theta}(a | s) } \right] } - \tau \KL( \pi_\theta(\cdot | s) \| \bar{\pi}_{\theta}(\cdot | s) ) \\
    &= \tau \log \sum_{a} \exp\left\{ \tilde{Q}^{\pi_\theta}(s,a) / \tau \right\} - \tau \cdot \KL( \pi_\theta(\cdot | s) \| \bar{\pi}_{\theta}(\cdot | s) ).
\end{align}
Combining the above,
\begin{align}
\MoveEqLeft
    \tilde{V}^{\pi_\tau^*}(\rho) - \tilde{V}^{{\pi_\theta}}(\rho) 
    = \frac{1}{1 - \gamma} \sum_{s}{ d_{\rho}^{\pi_\tau^*}(s) \cdot { \left[ \sum_{a}{ \pi_\tau^*(a | s) \cdot \left[ \tilde{Q}^{\pi_\theta}(s,a) - \tau \log{ \pi_\tau^*(a | s) } \right]} - \tilde{V}^{{\pi_\theta}}(s) \right] } } \\
    &\le \frac{1}{1 - \gamma} \sum_{s}{ d_{\rho}^{\pi_\tau^*}(s) \cdot { \left[ \tau \log \sum_{a} \exp\left\{ \tilde{Q}^{\pi_\theta}(s,a) / \tau \right\} - \tilde{V}^{{\pi_\theta}}(s) \right] } } \\
    \label{eq:intermediate_justification_sac_2}
    &= \frac{1}{1 - \gamma} \sum_{s}{ d_{\rho}^{\pi_\tau^*}(s) \cdot  \tau \cdot \KL( \pi_\theta(\cdot | s) \| \bar{\pi}_{\theta}(\cdot | s) )  } \\
    \label{eq:intermediate_justification_pcl_1}
    &\le \frac{1}{1 - \gamma} \sum_{s}{ d_{\rho}^{\pi_\tau^*}(s) \cdot \frac{\tau}{2} \cdot \left\| \frac{ \tilde{Q}^{\pi_\theta}(s,\cdot) }{\tau} - \theta(s, \cdot) - \frac{(\tilde{Q}^{\pi_\theta}(s,\cdot) / \tau - \theta(s, \cdot) )^\top \rvone }{A} \cdot \rvone\right\|_\infty^2  } 
    \,\,\, 
    \qquad \left( \text{by \cref{lem:kl_logit_inequality}} \right)
    \\
    &= \frac{1}{1 - \gamma} \sum_{s}{ d_{\rho}^{\pi_\tau^*}(s) \cdot \frac{1}{2 \tau} \cdot \left\| \tilde{Q}^{\pi_\theta}(s,\cdot) - \tau \theta(s, \cdot) - \frac{( \tilde{Q}^{\pi_\theta}(s,\cdot) - \tau \theta(s, \cdot) )^\top \rvone }{A} \cdot \rvone\right\|_\infty^2  },
\end{align}
where $A = \left| \gA \right|$ is the total number of actions. Taking square root of soft sub-optimality,
\begin{align}
\MoveEqLeft
    \left[ \tilde{V}^{\pi_\tau^*}(\rho) - \tilde{V}^{{\pi_\theta}}(\rho) \right]^\frac{1}{2} 
    \le \frac{1}{\sqrt{1 - \gamma}} \cdot \left[ \sum_{s}{ d_{\rho}^{\pi_\tau^*}(s) \cdot \frac{1}{2 \tau} \cdot \left\| \tilde{Q}^{\pi_\theta}(s,\cdot) - \tau \theta(s, \cdot) - \frac{( \tilde{Q}^{\pi_\theta}(s,\cdot) - \tau \theta(s, \cdot) )^\top \rvone }{A} \cdot \rvone\right\|_\infty^2  } \right]^{\frac{1}{2}} \\
    &= \frac{1}{\sqrt{1 - \gamma}} \cdot \left[ \sum_{s}{ \left(  \sqrt{ d_{\rho}^{\pi_\tau^*}(s) } \cdot \frac{1}{\sqrt{2 \tau}}  \cdot \left\| \tilde{Q}^{\pi_\theta}(s,\cdot) - \tau \theta(s, \cdot) - \frac{( \tilde{Q}^{\pi_\theta}(s,\cdot) - \tau \theta(s, \cdot) )^\top \rvone }{A} \cdot \rvone\right\|_\infty \right)^2  } \right]^\frac{1}{2} \\
    &\le \frac{1}{\sqrt{1 - \gamma}} \cdot \sum_{s}{  \sqrt{ d_{\rho}^{\pi_\tau^*}(s) } \cdot \frac{1}{\sqrt{2 \tau}} \cdot \left\| \tilde{Q}^{\pi_\theta}(s,\cdot) - \tau \theta(s, \cdot) - \frac{(\tilde{Q}^{\pi_\theta}(s,\cdot) - \tau \theta(s, \cdot) )^\top \rvone }{A} \cdot \rvone\right\|_\infty } 
    \qquad \left( \text{by $\| x \|_2 \le \| x \|_1$} \right)
    \\
    \label{eq:intermediate_justification_pcl_2}
    &\le \frac{1}{\sqrt{1 - \gamma}} \cdot \frac{1}{\sqrt{2 \tau }} \cdot \left\| \frac{d_{\rho}^{\pi_\tau^*} }{ d_{\mu}^{\pi_\theta}} \right\|_\infty^\frac{1}{2} \sum_{s}{ \sqrt{ d_{\mu}^{\pi_\theta}(s) } \cdot \left\| \tilde{Q}^{\pi_\theta}(s,\cdot) - \tau \theta(s, \cdot) - \frac{(\tilde{Q}^{\pi_\theta}(s,\cdot) - \tau \theta(s, \cdot) )^\top \rvone }{A} \cdot \rvone \right\|_\infty }.
\end{align}
On the other hand, the entropy regularized policy gradient norm is lower bounded as
\begin{align}
    \left\| \frac{\partial \tilde{V}^{{\pi_\theta}}(\mu) }{\partial \theta} \right\|_2 &= \left[ \sum_{s,a}{\left( \frac{\partial \tilde{V}^{{\pi_\theta}}(\mu) }{\partial \theta(s, a)} \right)^2 } \right]^\frac{1}{2} \\
    &= \left[ \sum_{s}{\left\| \frac{\partial \tilde{V}^{{\pi_\theta}}(\mu) }{\partial \theta(s, \cdot)} \right\|_2^2 } \right]^\frac{1}{2} \\
    &\ge \frac{1}{\sqrt{S}} \sum_{s}{\left\| \frac{\partial \tilde{V}^{{\pi_\theta}}(\mu) }{\partial \theta(s, \cdot)} \right\|_2 }, \qquad \left( \text{by Cauchy-Schwarz, } \| x \|_1 = | \langle \rvone, \ |x| \rangle | \le \| \rvone \|_2 \cdot \| x \|_2 \right)
\end{align}
which is further lower bounded as
\begin{align}
\MoveEqLeft
    \left\| \frac{\partial \tilde{V}^{{\pi_\theta}}(\mu) }{\partial \theta} \right\|_2
    \ge \frac{1}{\sqrt{S}} \cdot \frac{1}{1 - \gamma} \sum_{s}{ d_\mu^{\pi_\theta}(s) \cdot \left\| H( \pi_\theta(\cdot | s) ) \left[ \tilde{Q}^{\pi_\theta}(s, \cdot) - \tau \theta(s, \cdot) \right] \right\|_2 } \qquad \left( \text{by \cref{eq:policy_gradient_entropy_vector_form}, \cref{lem:policy_gradient_entropy}} \right) \\
    &= \frac{1}{\sqrt{S}} \cdot \frac{1}{1 - \gamma} \sum_{s}{ d_\mu^{\pi_\theta}(s) \cdot \left\| H( \pi_\theta(\cdot | s) ) \left[ \tilde{Q}^{\pi_\theta}(s, \cdot) - \tau \theta(s, \cdot) - \frac{(\tilde{Q}^{\pi_\theta}(s,\cdot) - \tau \theta(s, \cdot) )^\top \rvone }{A} \cdot \rvone \right] \right\|_2 } \qquad \left( \text{by \cref{lem:golub_rank_one_perturb}} \right) \\
    &\ge \frac{1}{\sqrt{S}} \cdot \frac{1}{1 - \gamma} \sum_{s}{ d_\mu^{\pi_\theta}(s) \cdot \min_{a}{ \pi_\theta(a | s)  } \cdot \left\| \tilde{Q}^{\pi_\theta}(s, \cdot) - \tau \theta(s, \cdot) - \frac{(\tilde{Q}^{\pi_\theta}(s,\cdot) - \tau \theta(s, \cdot) )^\top \rvone }{A} \cdot \rvone \right\|_2 } \qquad \left( \text{by \cref{lem:norm_decay_entropy_special}} \right) \\
    &\ge \frac{1}{\sqrt{S}} \cdot \frac{1}{1 - \gamma} \sum_{s}{ d_\mu^{\pi_\theta}(s) \cdot \min_{a}{ \pi_\theta(a | s)  } \cdot \left\| \tilde{Q}^{\pi_\theta}(s, \cdot) - \tau \theta(s, \cdot) - \frac{(\tilde{Q}^{\pi_\theta}(s,\cdot) - \tau \theta(s, \cdot) )^\top \rvone }{A} \cdot \rvone \right\|_\infty }.
\end{align}
Denote $\zeta_\theta(s) = \tilde{Q}^{\pi_\theta}(s, \cdot) - \tau \theta(s, \cdot) - \frac{(\tilde{Q}^{\pi_\theta}(s,\cdot) - \tau \theta(s, \cdot) )^\top \rvone }{K} \cdot \rvone$. We have,
\begin{align}
\MoveEqLeft
    \left\| \frac{\partial \tilde{V}^{{\pi_\theta}}(\mu) }{\partial \theta} \right\|_2
    \ge \frac{1}{\sqrt{S}} \cdot \frac{1}{1 - \gamma} \sum_{s}{ d_\mu^{\pi_\theta}(s) \cdot \min_{a}{ \pi_\theta(a | s)  } \cdot \left\| \zeta_\theta(s) \right\|_\infty } \\
    &\ge \frac{1}{\sqrt{S}} \cdot \frac{1}{\sqrt{1 - \gamma}} \cdot \min_{s}{\sqrt{ d_{\mu}^{\pi_\theta}(s) } } \cdot \min_{s,a}{ \pi_\theta(a | s)  } \cdot \sqrt{2 \tau} \cdot \left\| \frac{d_{\rho}^{\pi_\tau^*} }{ d_{\mu}^{\pi_\theta}} \right\|_\infty^{-\frac{1}{2}} \cdot \left[ \frac{1}{\sqrt{1 - \gamma}} \cdot \frac{1}{\sqrt{2 \tau }} \cdot \left\| \frac{d_{\rho}^{\pi_\tau^*} }{ d_{\mu}^{\pi_\theta}} \right\|_\infty^\frac{1}{2} \sum_{s}{ \sqrt{ d_{\mu}^{\pi_\theta}(s) } \cdot \left\| \zeta_\theta(s) \right\|_\infty } \right] \\
    &\ge \frac{1}{\sqrt{S}} \cdot \frac{1}{\sqrt{1 - \gamma}} \cdot \min_{s}{\sqrt{ d_{\mu}^{\pi_\theta}(s) } } \cdot \min_{s,a}{ \pi_\theta(a | s)  } \cdot \sqrt{2 \tau} \cdot \left\| \frac{d_{\rho}^{\pi_\tau^*} }{ d_{\mu}^{\pi_\theta}} \right\|_\infty^{-\frac{1}{2}} \cdot \left[ \tilde{V}^{\pi_\tau^*}(\rho) - \tilde{V}^{{\pi_\theta}}(\rho) \right]^\frac{1}{2} 
    \\
    &\ge \frac{\sqrt{2 \tau}}{\sqrt{S}} \cdot \min_{s}{\sqrt{ \mu(s) } } \cdot \min_{s,a}{ \pi_\theta(a | s)  } \cdot \left\| \frac{d_{\rho}^{\pi_\tau^*} }{ d_{\mu}^{\pi_\theta}} \right\|_\infty^{-\frac{1}{2}} \cdot \left[ \tilde{V}^{\pi_\tau^*}(\rho) - \tilde{V}^{{\pi_\theta}}(\rho) \right]^\frac{1}{2},
\end{align}
where the last inequality is by $d_{\mu}^{\pi_\theta}(s) \ge (1 - \gamma) \cdot \mu(s)$ (cf. \cref{eq:stationary_distribution_dominate_initial_state_distribution}).
\end{proof}

\textbf{\cref{lem:lower_bound_min_prob_entropy_general}.}
Using \cref{alg:policy_gradient_softmax} with 
the entropy regularized objective, we have $c:=\inf_{t \ge 1} \min_{s,a} { \pi_{\theta_t}(a | s) } > 0$.
\begin{proof}
The augmented value function $\tilde{V}^{{\pi_{\theta_t}}}(\rho)$ is monotonically increasing following gradient update due to smoothness, i.e., \cref{lem:smoothness_softmax_general,lem:smoothness_entropy_general}. It follows then that $\tilde{V}^{{\pi_{\theta_t}}}(\rho)$ is upper bounded. Indeed,
\begin{align}
\label{eq:augmented_value_upper_bound}
\MoveEqLeft
    \tilde{V}^{{\pi_{\theta_t}}}(\rho) = \expectation_{\substack{s_0 \sim \rho, a_t \sim {\pi_{\theta_t}}(\cdot | s_t), \\ s_{t+1} \sim \gP( \cdot | s_t, a_t)}}{\left[ \sum_{t=0}^{\infty}{\gamma^t ( r(s_t, a_t) - \tau \log{\pi_{\theta_t}(a_t | s_t)} ) } \right]} \\
    &= \frac{1}{1-\gamma} \sum_{s}{ d_{\rho}^{\pi_{\theta_t}}(s) \cdot { \left[ \sum_{a}{ \pi_{\theta_t}(a | s) \cdot \left( r(s,a)  - \tau \log{\pi_{\theta_t}(a | s)} \right) } \right]} } \\
    &\le \frac{1}{1-\gamma} \sum_{s}{ d_{\rho}^{\pi_{\theta_t}}(s) \cdot { \left( 1 + \tau \log{A} \right) } } \qquad \left( \text{by } r(s,a) \le 1 \text{ and } -\sum_{a}{ \pi_{\theta_t}(a | s) \cdot \log{\pi_{\theta_t}(a | s)} } \le \log{A} \right) \\
    &\le \frac{1 + \tau \log{A}}{1 - \gamma}.
\end{align}
According to the monotone convergence theorem, $\tilde{V}^{{\pi_{\theta_t}}}(\rho)$ converges to a finite value. Suppose $ \pi_{\theta_t}(a | s) \to \pi_{\theta_\infty}(a | s)$. For any state $s \in \gS$, define the following sets,
\begin{align}
    \gA_0(s) &= \left\{ a : \pi_{\theta_\infty}(a | s) = 0 \right\}, \\
    \gA_+(s) &= \left\{ a : \pi_{\theta_\infty}(a | s) > 0 \right\}.
\end{align}
Note that $\gA = \gA_0(s) \cup \gA_+(s)$ since $\pi_\infty(a | s) \ge 0$, $\forall a \in \gA$. We prove that for any state $s \in \gS$, $\gA_0(s) = \emptyset$ by contradiction. Suppose $\exists s \in \gS$, such that $\gA_0(s)$ is non-empty. For any $a_0 \in \gA_0(s)$, we have $ \pi_{\theta_t}(a_0 | s) \to \pi_{\theta_\infty}(a_0 | s) = 0$, which implies $- \log{\pi_{\theta_t}(a_0 | s)} \to \infty$. There exists $t_0 \ge 1$, such that $\forall t \ge t_0$,
\begin{align}
    - \log{\pi_{\theta_t}(a_0 | s)} \ge \frac{1 + \tau \log{A}}{\tau(1 - \gamma)}.
\end{align}
According to \cref{lem:policy_gradient_entropy}, $\forall t \ge t_0$,
\begin{align}
\label{eq:logit_increasing_a0}
    \frac{\partial \tilde{V}^{\pi_{\theta_t}}(\mu)}{\partial {\theta_t}(s,a_0)} &= \frac{1}{1-\gamma} \cdot d_{\mu}^{\pi_{\theta_t}}(s) \cdot \pi_{\theta_t}(a_0|s) \cdot \tilde{A}^{\pi_{\theta_t}}(s,a_0) \\
    &= \frac{1}{1-\gamma} \cdot d_{\mu}^{\pi_{\theta_t}}(s) \cdot \pi_{\theta_t}(a_0|s) \cdot \left[ \tilde{Q}^{\pi_{\theta_t}}(s, a_0) - \tau \log{\pi_{\theta_t}(a_0 | s)} - \tilde{V}^{\pi_{\theta_t}}(s) \right] \\
    &\ge \frac{1}{1-\gamma} \cdot d_{\mu}^{\pi_{\theta_t}}(s) \cdot \pi_{\theta_t}(a_0|s) \cdot \left[ 0 - \tau \log{\pi_{\theta_t}(a_0 | s)} - \frac{1 + \tau \log{A}}{1 - \gamma} \right] \\
    &\ge \frac{1}{1-\gamma} \cdot d_{\mu}^{\pi_{\theta_t}}(s) \cdot \pi_{\theta_t}(a_0|s) \cdot \left[ 0 + \tau \cdot  \frac{1 + \tau \log{A}}{\tau(1 - \gamma)} - \frac{1 + \tau \log{A}}{1 - \gamma} \right] = 0,
\end{align}
where the first inequality is by
\begin{align}
    \tilde{Q}^{\pi_{\theta_t}}(s, a_0) = r(s,a_0) + \gamma \sum_{s^\prime}{ \gP( s^\prime | s, a_0)  \tilde{V}^{{\pi_{\theta_t}}}(s^\prime) } \ge 0. \qquad \left( \text{by } r(s,a_0) \ge 0 \text{ and } \tilde{V}^{{\pi_{\theta_t}}}(s^\prime) \ge 0 \right)
\end{align}
This means that $\theta_t(s, a_0)$ is increasing for any $t \ge t_0$, which in turn implies that $\theta_\infty(s, a_0)$ is lower bounded by constant, i.e., $\theta_\infty(s, a_0) \ge c$ for some constant $c$, and thus $\exp\left\{ \theta_\infty(a_0 | s) \right\} \ge e^c > 0$. According to
\begin{align}
    \pi_{\theta_\infty}(a_0 | s) = \frac{ \exp\left\{ \theta_\infty(a_0 | s) \right\}}{ \sum_{a}{ \exp\left\{ \theta_\infty(a | s) \right\} } } = 0,
\end{align}
we have,
\begin{align}
    \sum_{a}{ \exp\left\{ \theta_\infty(a | s) \right\} } = \infty.
\end{align}
On the other hand, for any $a_+ \in \gA_+(s)$, according to
\begin{align}
    \pi_{\theta_\infty}(a_+ | s) = \frac{ \exp\left\{ \theta_\infty(a_+ | s) \right\}}{ \sum_{a}{ \exp\left\{ \theta_\infty(a | s) \right\} } } > 0,
\end{align}
we have,
\begin{align}
    \exp\left\{ \theta_\infty(a_+ | s) \right\} = \infty, \quad \forall a_+ \in \gA_+(s)
\end{align}
which implies,
\begin{align}
\label{eq:sum_logit_aplus_infty}
    \sum_{a_+ \in \gA_+(s)}{ \theta_\infty(a_+ | s) } = \infty.
\end{align}
Note that $\forall t$, the summation of logit incremental over all actions is zero:
\begin{align}
\label{eq:logit_incremental_summation_zero}
    \sum_{a} \frac{\partial \tilde{V}^{\pi_{\theta_t}}(\mu)}{\partial {\theta_t}(s,a)} &= \sum_{a_0 \in \gA_0(s)}{ \frac{\partial \tilde{V}^{\pi_{\theta_t}}(\mu)}{\partial {\theta_t}(s,a_0)} } + \sum_{a_+ \in \gA_+(s)}{ \frac{\partial \tilde{V}^{\pi_{\theta_t}}(\mu)}{\partial {\theta_t}(s,a_+)} } \\
    &= \frac{1}{1-\gamma} \cdot  d_{\mu}^{\pi_{\theta_t}}(s) \sum_{a}{ \pi_{\theta_t}(a|s) \cdot \tilde{A}^{\pi_{\theta_t}}(s,a) } \\
    &= \frac{1}{1-\gamma} \cdot d_{\mu}^{\pi_{\theta_t}}(s) \cdot \left[ \tilde{V}^{{\pi_{\theta_t}}}(s) - \tilde{V}^{{\pi_{\theta_t}}}(s) \right] = 0.
\end{align}
According to \cref{eq:logit_increasing_a0}, $\forall t \ge t_0$,
\begin{align}
    \sum_{a_0 \in \gA_0(s)}{ \frac{\partial \tilde{V}^{\pi_{\theta_t}}(\mu)}{\partial {\theta_t}(s,a_0)} } \ge 0.
\end{align}
According to \cref{eq:logit_incremental_summation_zero}, $\forall t \ge t_0$,
\begin{align}
    \sum_{a_+ \in \gA_+(s)}{ \frac{\partial \tilde{V}^{\pi_{\theta_t}}(\mu)}{\partial {\theta_t}(s,a_+)} } = 0 -  \sum_{a_0 \in \gA_0(s)}{ \frac{\partial \tilde{V}^{\pi_{\theta_t}}(\mu)}{\partial {\theta_t}(s,a_0)} } \le 0.
\end{align}
which means $\sum_{a_+ \in \gA_+(s)}{ \theta_t(s, a_+) }$ will decrease for all large enough $t \ge 1$. This contradicts with \cref{eq:sum_logit_aplus_infty}, i.e., $\sum_{a_+ \in \gA_+(s)}{ \theta_t(s, a_+) } \to \infty$.

To this point, we have shown that $\gA_0(s) = \emptyset$ for any state $s \in \gS$, i.e.,  $\pi_{\theta_t}(\cdot | s)$ will converge in the interior of probabilistic simplex $\Delta(\gA)$. Furthermore, at the convergent point $\pi_{\theta_\infty}(\cdot | s)$, the gradient is zero, otherwise by smoothness the objective can be further improved, which is a contradiction with convergence. According to \cref{lem:policy_gradient_entropy}, $\forall s$,
\begin{align}
    \frac{\partial \tilde{V}^{\pi_{\theta_\infty}}(\mu)}{\partial {\theta_\infty}(s, \cdot)} = \frac{1}{1-\gamma} \cdot  d_{\mu}^{\pi_{\theta_\infty}}(s) \cdot H(\pi_{\theta_\infty}( \cdot |s)) \left[ \tilde{Q}^{\pi_{\theta_\infty}}(s, \cdot) - \tau \log{ \pi_{\theta_\infty}(\cdot | s) } \right] = \rvzero.
\end{align}

We have $d_{\mu}^{\pi_{\theta_\infty}}(s) \ge (1 - \gamma) \cdot \mu(s) > 0$ for all states $s$
(cf. \cref{eq:stationary_distribution_dominate_initial_state_distribution}). 
Therefore we have, $\forall s$,
\begin{align}
    H(\pi_{\theta_\infty}( \cdot |s)) \left[ \tilde{Q}^{\pi_{\theta_\infty}}(s, \cdot) - \tau \log{ \pi_{\theta_\infty}(\cdot | s) } \right] = \rvzero.
\end{align}
According to \cref{lem:golub_rank_one_perturb}, $H(\pi_{\theta_\infty}( \cdot |s))$ has eigenvalue $0$ with multiplicity $1$, and its corresponding eigenvector is $c \cdot \rvone$ for some constant $c \in \sR$. Therefore, the gradient is zero implies that for all states $s$,
\begin{align}
    \tilde{Q}^{\pi_{\theta_\infty}}(s, \cdot) - \tau \log{ \pi_{\theta_\infty}(\cdot | s) } = c \cdot \rvone,
\end{align}
which is equivalent to
\begin{align}
    \pi_{\theta_\infty}(\cdot | s) = \softmax( \tilde{Q}^{\pi_{\theta_\infty}}(s, \cdot) / \tau ),
\end{align}
which, according to \citet[Theorem 3]{nachum2017bridging}, is the softmax optimal policy $\pi_\tau^*$. Since $\tau \in \Omega(1) > 0$ and,
\begin{align}
    0 \le \tilde{Q}^{\pi_{\theta_\infty}}(s, a) \le \frac{1 + \tau \log{A}}{1 - \gamma},
\end{align}
we have $\pi_{\theta_\infty}(a | s) \in \Omega(1)$, $\forall (s,a)$. Since $\pi_{\theta_t}(a | s) \to \pi_{\theta_\infty}(a |s)$, there exists $t_0 \ge 1$, such that $\forall t \ge t_0$,
\begin{align}
    0.9 \cdot \pi_{\theta_\infty}(a |s) \le \pi_{\theta_t}(a | s) \le 1.1 \cdot \pi_{\theta_\infty}(a |s), \ \forall (s,a),
\end{align}
which means $\inf_{t \ge t_0}{ \min_{s,a} \pi_{\theta_t}(a | s) } \in \Omega(1)$, and thus 
\begin{equation*}
    \inf_{t \ge 1} \min_{s,a} \pi_{\theta_t}(a | s) = \min\left\{ \min_{1 \le t \le t_0} \min_{s,a} \pi_{\theta_t}(a | s), \ \inf_{t \ge t_0}{ \min_{s,a} \pi_{\theta_t}(a | s) } \right\} = \min\{ \Omega(1), \ \Omega(1) \} \in \Omega(1). \qedhere
\end{equation*}
\end{proof}

\textbf{\cref{thm:final_rates_entropy_general}.}
Suppose $\mu(s) > 0$ for all state $s$. Using \cref{alg:policy_gradient_softmax} with the entropy regularized objective and softmax parametrization and $\eta = (1 - \gamma)^3/(8 + \tau ( 4 + 8 \log{A}))$,
there exists a constant $C>0$ such that
 for all $t \ge 1$,
\begin{align}
    \tilde{V}^{\pi_\tau^*}(\rho) - \tilde{V}^{\pi_{\theta_t}}(\rho) \le \left\| \frac1\mu \right\|_\infty  \cdot \frac{1 + \tau \log{A}}{(1 - \gamma)^2}
    \, \cdot \, e^{-C(t-1)}
    \,.
\end{align}
\begin{proof}
According to the soft sub-optimality lemma of \cref{lem:soft_suboptimality},
\begin{align}
    \tilde{V}^{\pi_\tau^*}(\rho) - \tilde{V}^{\pi_{\theta_t}}(\rho) &= \frac{1}{1 - \gamma} \sum_{s}{ \left[  d_{\rho}^{\pi_{\theta_t}}(s) \cdot \tau \cdot { \KL(\pi_{\theta_t}( \cdot | s) \| \pi_\tau^*(\cdot | s) ) } \right] } \\
    &= \frac{1}{1 - \gamma} \sum_{s}{ \frac{d_{\rho}^{\pi_{\theta_t}}(s)}{d_{\mu}^{\pi_{\theta_t}}(s)} \cdot \left[  d_{\mu}^{\pi_{\theta_t}}(s) \cdot \tau \cdot { \KL(\pi_{\theta_t}( \cdot | s) \| \pi_\tau^*(\cdot | s) ) } \right] } \\ 
    &\le \frac{1}{(1 - \gamma)^2} \sum_{s}{ \frac{1}{\mu(s)} \cdot \left[  d_{\mu}^{\pi_{\theta_t}}(s) \cdot \tau \cdot { \KL(\pi_{\theta_t}( \cdot | s) \| \pi_\tau^*(\cdot | s) ) } \right] } \\ 
    &\le \frac{1}{(1 - \gamma)^2} \cdot \left\| \frac{1}{\mu} \right\|_\infty \sum_{s}{ \left[  d_{\mu}^{\pi_{\theta_t}}(s) \cdot \tau \cdot { \KL(\pi_{\theta_t}( \cdot | s) \| \pi_\tau^*(\cdot | s) ) } \right] } \\ 
    &= \frac{1}{1 - \gamma} \cdot \left\| \frac{1}{\mu} \right\|_\infty \cdot \left[ \tilde{V}^{\pi_\tau^*}(\mu) - \tilde{V}^{\pi_{\theta_t}}(\mu) \right],
\end{align}
where the last equation is again by \cref{lem:soft_suboptimality}, and the first inequality is according to $d_{\mu}^{\pi_{\theta_t}}(s) \ge (1 - \gamma) \cdot \mu(s)$ 
(cf. \cref{eq:stationary_distribution_dominate_initial_state_distribution}). 
According to \cref{lem:smoothness_softmax_general,lem:smoothness_entropy_general}, $V^{\pi_\theta}(\mu)$ is $8/(1-\gamma)^3$-smooth, and $\sH(\mu, \pi_\theta)$ is $(4 + 8 \log{A}) /(1-\gamma)^3$-smooth. Therefore, $\tilde{V}^{\pi_\theta}(\mu) = V^{\pi_\theta}(\mu) + \tau \cdot \sH(\mu, \pi_\theta)$ is $\beta$-smooth with $\beta = (8 + \tau ( 4 + 8 \log{A}) )/(1-\gamma)^3$. Denote $\tilde{\delta}_t = \tilde{V}^{\pi_\tau^*}(\mu) - \tilde{V}^{\pi_{\theta_t}}(\mu)$. And note $ \eta = \frac{(1 - \gamma)^3}{8 + \tau ( 4 + 8 \log{A})}$. We have,
\begin{align}
\MoveEqLeft
    \tilde{\delta}_{t+1} - \tilde{\delta}_t = \tilde{V}^{\pi_{\theta_{t}}}(\mu) - \tilde{V}^{\pi_{\theta_{t+1}}}(\mu) \\
    &\le - \frac{(1-\gamma)^3}{16 + \tau ( 8 + 16 \log{A})} \cdot \left\| \frac{\partial V^{\pi_{\theta_t}}(\mu)}{\partial \theta_t} \right\|_2^2 \qquad \left( \text{by \cref{lem:ascent_lemma_smooth_function}} \right) \\
    &\le - \frac{(1-\gamma)^3}{16 + \tau ( 8 + 16 \log{A})} \cdot \frac{2 \tau}{S} \cdot \min_{s}{ \mu(s) } \cdot \min_{s,a}{ \pi_{\theta_t}(a | s)^2  } \cdot \left\| \frac{d_{\mu}^{\pi_\tau^*} }{ d_{\mu}^{\pi_{\theta_t}}} \right\|_\infty^{-1} \cdot \left[ \tilde{V}^{\pi_\tau^*}(\mu) - \tilde{V}^{{\pi_{\theta_t}}}(\mu) \right] \qquad \left( \text{by \cref{lem:lojasiewicz_entropy_general}} \right) \\
    &\le - \frac{(1-\gamma)^4}{(8 / \tau + 4 + 8 \log{A}) \cdot S} \cdot \min_{s}{ \mu(s) } \cdot \min_{s,a}{ \pi_{\theta_t}(a | s)^2  } \cdot \left\| \frac{d_{\mu}^{\pi_\tau^*} }{ \mu } \right\|_\infty^{-1} \cdot \tilde{\delta}_t 
    \qquad \left( \text{by $ d_{\mu}^{\pi_{\theta_t}}(s) \ge (1 - \gamma) \cdot \mu(s)$} \right)  \\
    &\le - \frac{(1-\gamma)^4}{(8 / \tau + 4 + 8 \log{A}) \cdot S} \cdot \min_{s}{ \mu(s) } \cdot \inf_{t \ge 1} \min_{s,a}{ \pi_{\theta_t}(a | s)^2  }  \cdot \left\| \frac{d_{\mu}^{\pi_\tau^*} }{ \mu } \right\|_\infty^{-1} \cdot \tilde{\delta}_t,
\end{align}
According to \cref{lem:lower_bound_min_prob_entropy_general}, $c = \inf_{t \ge 1} \min_{s,a}{ \pi_{\theta_t}(a | s)  } > 0$ is independent with $t$. We have, 
\begin{align}
    \tilde{\delta}_{t} &\le \left[ 1 - \frac{(1-\gamma)^4 }{(8 / \tau + 4 + 8 \log{A}) \cdot S} \cdot \min_{s}{ \mu(s) } \cdot c^2 \cdot \left\| \frac{d_{\mu}^{\pi_\tau^*} }{ \mu } \right\|_\infty^{-1} \right] \cdot \tilde{\delta}_{t-1} \\
    &\le \exp\left\{ - \frac{(1-\gamma)^4}{(8 / \tau + 4 + 8 \log{A}) \cdot S} \cdot \min_{s}{ \mu(s) } \cdot c^2 \cdot \left\| \frac{d_{\mu}^{\pi_\tau^*} }{ \mu } \right\|_\infty^{-1} \right\} \cdot \tilde{\delta}_{t-1} \\
    &\le \exp\left\{ - \frac{(1-\gamma)^4 }{(8 / \tau + 4 + 8 \log{A}) \cdot S} \cdot \min_{s}{ \mu(s) } \cdot c^2 \cdot \left\| \frac{d_{\mu}^{\pi_\tau^*} }{ \mu } \right\|_\infty^{-1} \cdot (t - 1) \right\} \cdot \tilde{\delta}_1 \\
    &\le \exp\left\{ - \frac{(1-\gamma)^4 }{(8 / \tau + 4 + 8 \log{A}) \cdot S} \cdot \min_{s}{ \mu(s) } \cdot c^2 \cdot \left\| \frac{d_{\mu}^{\pi_\tau^*} }{ \mu } \right\|_\infty^{-1} \cdot (t - 1) \right\} \cdot \frac{1 + \tau \log{A}}{1 - \gamma},
\end{align}
where the last inequality is according to \cref{eq:augmented_value_upper_bound}. Therefore we have the final result,
\begin{align}
    \tilde{V}^{\pi_\tau^*}(\rho) - \tilde{V}^{\pi_{\theta_t}}(\rho) &\le \frac{1}{1 - \gamma} \cdot \left\| \frac{1}{ \mu} \right\|_\infty \cdot \left[ \tilde{V}^{\pi_\tau^*}(\mu) - \tilde{V}^{\pi_{\theta_t}}(\mu) \right] \\
    &\le \frac{1}{\exp\left\{ C \cdot (t-1) \right\}} \cdot \frac{1 + \tau \log{A}}{(1 - \gamma)^2} \cdot \left\| \frac{1}{ \mu} \right\|_\infty,
\end{align}
where
\begin{align}
    C = \frac{(1-\gamma)^4 }{(8 / \tau + 4 + 8 \log{A}) \cdot S} \cdot \min_{s}{ \mu(s) } \cdot c^2 \cdot \left\| \frac{d_{\mu}^{\pi_\tau^*} }{ \mu } \right\|_\infty^{-1} > 0,
\end{align}
is independent with $t$.
\end{proof}

\subsubsection{Proofs for two-stage and decaying entropy regularization}
\label{sec:proofs_entropy_policy_gradient_decaying_entropy_bandits}

\textbf{\cref{thm:rates_two_stage_special}} (Two-stage)\textbf{.} Denote $\Delta = r(a^*) - \max_{a \not= a^*}{ r(a) } > 0$. Using \cref{update_rule:entropy_special} for $t_1 \in O( e^{ 1/ \tau }  \cdot \log{( \frac{ \tau + 1}{\Delta } } ) )$ iterations and then \cref{update_rule:softmax_special} for $t_2 \ge 1$ iterations, we have,
\begin{align}
    ( \pi^* - \pi_{\theta_t} )^\top r \le 5 / ( {C}^2 \cdot t_2),
\end{align}
where $t = t_1 + t_2$, and $C \in [1/K, 1)$.
\begin{proof}
In particular, using \cref{update_rule:entropy_special} with $\eta \le 1/ \tau$ for the following number of iterations,
\begin{align}
    t_1 &= \frac{1}{\tau \eta} \cdot K  \cdot \exp\Big\{ 4 \| \theta_1 \|_\infty \sqrt{K} \Big\} \cdot \exp\Big\{ \frac{1 + 4 \sqrt{K}}{ \tau }  \Big\} \cdot \log{\left( \frac{4 (\tau \| \theta_1 \|_\infty + 1) \sqrt{K}}{\Delta} \right) } + 1 \\
    &\in O\left( e^{ 1/ \tau }  \cdot \log{\Big( \frac{ \tau + 1}{\Delta } } \Big) \right),
\end{align}
we have,
\begin{align}
    t_1 - 1 &\ge \frac{1}{\tau \eta} \cdot K  \cdot \exp\Big\{ 4 \| \theta_1 \|_\infty \sqrt{K} \Big\} \cdot \exp\Big\{ \frac{1 + 4 \sqrt{K}}{ \tau }  \Big\} \cdot \log{\left( \frac{4 (\tau \| \theta_1 \|_\infty + 1) \sqrt{K}}{\Delta} \right) } \\
    &= \frac{1}{\tau \eta} \cdot K \cdot \exp\{ 1 / \tau \} \cdot \exp\{ 4 (\norm{\theta_1}_\infty + 1 / \tau ) \sqrt{K} \} \cdot \log{\left( \frac{4 (\tau \| \theta_1 \|_\infty + 1) \sqrt{K}}{\Delta} \right) } \\
    &\ge \frac{1}{\tau \eta} \cdot \frac{1}{ c } \cdot \log{\left( \frac{4 (\tau \| \theta_1 \|_\infty + 1) \sqrt{K}}{\Delta} \right) }. \qquad \left( c \text{ is from \cref{lem:lower_bound_min_prob_entropy_special}} \right)
\end{align}
Therefore we have,
\begin{align}
    \log{\left( \frac{4 (\tau \| \theta_1 \|_\infty + 1) \sqrt{K}}{\Delta} \right) } &\le \tau \eta \cdot c \cdot (t_1 - 1 ) \\
    &\le \tau \eta \sum_{s=1}^{t_1-1}{ \min_{a}{ \pi_{\theta_s}(a) } } \qquad \left( \text{by \cref{lem:lower_bound_min_prob_entropy_special}} \right) \\
    &\le \log{\left( \frac{ 2 ( \tau \norm{\theta_1}_\infty +1 ) \sqrt{K} }{ \| \zeta_{t_1} \|_2 } \right)}, \qquad \left( \text{by \cref{lem:matching_entropy_special}} \right)
\end{align}
which is equivalent to,
\begin{align}
    \| \zeta_{t_1} \|_2 = \left\| \tau \theta_{t_1} - r - \frac{(\tau \theta_{t_1} - r)^\top \rvone}{K} \cdot \rvone \right\|_2 \le \frac{\Delta}{2}.
\end{align}
Then we have,  for all $a$,
\begin{align}
    \left|  \theta_{t_1}(a) - \frac{r(a)}{\tau} -  \frac{ ( \tau \theta_{t_1} - r )^\top \rvone }{ \tau K} \right| &\le \left\| \theta_{t_1} - \frac{r}{\tau} -  \frac{ ( \tau \theta_{t_1} - r )^\top \rvone }{ \tau K} \cdot \rvone \right\|_2 \\
    &= \frac{1}{\tau} \cdot \left\| \tau \theta_{t_1} - r -  \frac{ ( \tau \theta_{t_1} - r )^\top \rvone }{ K} \cdot \rvone \right\|_2 \le \frac{\Delta}{2 \tau},
\end{align}
which implies,
\begin{align}
    \theta_{t_1}(a^*) &\ge \frac{r(a^*)}{\tau} - \frac{\Delta}{2 \tau} + \frac{ ( \tau \theta_{t_1} - r )^\top \rvone }{ \tau K}, \quad \text{and} \\
    \theta_{t_1}(a) &\le \frac{r(a)}{\tau} + \frac{\Delta}{2 \tau} + \frac{ ( \tau \theta_{t_1} - r )^\top \rvone }{ \tau K}. \quad \text{for all } a \not= a^*
\end{align}
Then we have, for all $a \not= a^*$,
\begin{align}
    \theta_{t_1}(a^*) - \theta_{t_1}(a) &\ge \frac{r(a^*)}{\tau} - \frac{\Delta}{2 \tau} - \left( \frac{r(a)}{\tau} + \frac{\Delta}{2 \tau} \right) \\
    &= \frac{r(a^*)}{\tau} - \frac{r(a)}{\tau} - \frac{\Delta}{\tau} \ge 0,
\end{align}
which means $\pi_{\theta_{t_1}}(a^*) \ge \pi_{\theta_{t_1}}(a)$. Now we turn off the regularization and use \cref{update_rule:softmax_special} for $t_2 \ge 1$ iterations. According to similar arguments as in \cref{thm:rates_uniform_softmax_special}, we have,
\begin{align}
    ( \pi^* - \pi_{\theta_t} )^\top r \le 5 / ( {C}^2 \cdot t_2),
\end{align}
where $t = t_1 + t_2$, and $C \in [1/K, 1)$.
\end{proof}

\textbf{\cref{thm:rates_decaying_entropy_special}} (Decaying entropy regularization)\textbf{.}
Using \cref{update_rule:decaying_entropy_special} with $\tau_t = \frac{\alpha \cdot \Delta}{\log{t}}$ for $t \ge 2$, where $\alpha > 0$, and $\eta_t = 1/\tau_t$, we have, for all $t \ge 1$,
\begin{align}
    ( \pi^* - \pi_{\theta_t} )^\top r \le \frac{K}{t^{1/\alpha}} + \frac{ \log{t} }{ \exp\left\{ \sum_{s=1}^{t-1}{ \min_{a}{ \pi_{\theta_s}(a) } }  \right\}} \cdot \frac{2 ( \tau_1 \norm{\theta_1}_\infty +1 ) \sqrt{K}}{\alpha \cdot \Delta}.
\end{align}
\begin{proof}
Denote $\pi_{\tau_t}^* = \softmax(r / \tau_t)$ as the softmax optimal policy at time $t$. We have,
\begin{align}
\label{eq:rates_decaying_entropy_special_intermediate_0}
    (\pi^* - \pi_{\theta_t})^\top r = \underbrace{(\pi^* - \pi_{\tau_t}^*)^\top r}_{\text{``decaying''}} + \underbrace{(\pi_{\tau_t}^* - \pi_{\theta_t})^\top r}_{\text{``tracking''}}.
\end{align}
\paragraph{``decaying'' part.} Note $a^*$ is the optimal action. Denote $\Delta(a) = r(a^*) - r(a)$, and $\Delta = \min_{a \not= a^*}{ \Delta(a) }$. We have,
\begin{align}
\MoveEqLeft
    (\pi^* - \pi_{\tau_t}^*)^\top r = \sum_{a}{ \pi_{\tau_t}^*(a) \cdot r(a^*) } -  \sum_{a}{ \pi_{\tau_t}^*(a) \cdot r(a) } = \sum_{a \not= a^*}{ \pi_{\tau_t}^*(a) \cdot \Delta(a) } \\
    &= \frac{ \sum_{a \not= a^*}{ e^{ \frac{r(a)}{\tau_t}  } \cdot \Delta(a) } }{ \sum_{a^\prime}{ e^{ \frac{r(a^\prime)}{ \tau_t } } } } \\
    &\le \frac{1}{ e^{ \frac{r(a^*)}{\tau_t} } + \max_{a \not= a^*}{ e^{ \frac{r(a)}{\tau_t} } } } \cdot \left[  \sum_{a \not= a^*}{ e^{ \frac{r(a)}{\tau_t}  } \cdot \Delta(a) } \right] \qquad \left( \sum_{a^\prime}{ e^{ \frac{r(a^\prime)}{\tau_t}  } } \ge e^{ \frac{r(a^*)}{\tau_t} } + \max_{a \not= a^*}{ e^{ \frac{r(a)}{\tau_t} } } \right) \\
    &= \frac{1}{ e^{ \frac{r(a^*)}{\tau_t} } + \max_{a \not= a^*}{ e^{ \frac{r(a)}{\tau_t} } } } \cdot \left[ \sum_{a \not= a^*}{ \frac{ e^{ \frac{r(a)}{\tau_t} } \cdot \Delta(a)}{ e^{ \frac{r(a^*)}{\tau_t} } + e^{ \frac{r(a)}{\tau_t} } } \cdot \left( e^{ \frac{r(a^*)}{\tau_t} } + e^{ \frac{r(a)}{\tau_t} } \right) } \right] \\
    &\le  \frac{1}{ \bcancel{ e^{ \frac{r(a^*)}{\tau_t} } + \max_{a \not= a^*}{ e^{ \frac{r(a)}{\tau_t} } } } } \cdot \sum_{a \not= a^*}{ \frac{  e^{ \frac{r(a)}{\tau_t} } \cdot \Delta(a)}{  e^{ \frac{r(a^*)}{\tau_t} } +  e^{ \frac{r(a)}{\tau_t} } } \cdot  \left( \bcancel{ e^{ \frac{r(a^*)}{\tau_t} } + \max_{a \not= a^* }{ e^{ \frac{r(a)}{\tau_t} } } } \right)   } \qquad \left( \text{by H{\" o}lder's inequality} \right) \\
    &= \sum_{a \not= a^*}{ \frac{ e^{ \frac{r(a)}{\tau_t} } \cdot \Delta(a)}{ e^{ \frac{r(a^*)}{\tau_t} } + e^{ \frac{r(a)}{\tau_t} } }  } = \sum_{a \not= a^*}{ \frac{ \Delta(a)}{ e^{ \frac{\Delta(a)}{\tau_t} } + 1 }  } \le  \sum_{a \not= a^*}{ \frac{ 1 }{ e^{ \frac{ \Delta }{\tau_t} } + 1 }  } = \frac{K-1}{1 + e^{ \frac{ \Delta }{\tau_t} } } \le \frac{K}{e^{ \frac{ \Delta }{\tau_t} } }.
\end{align}
Using the decaying temperature $\tau_t = \frac{\alpha \cdot \Delta}{\log{t}}$, for $t \ge 2$, where $\alpha > 0$, we have,
\begin{align}
\label{eq:rates_decaying_entropy_special_intermediate_1}
    (\pi^* - \pi_{\tau_t}^*)^\top r \le \frac{K}{t^{1/\alpha}}.
\end{align}

\paragraph{``tracking'' part.} Using \cref{update_rule:decaying_entropy_special}, we have,
\begin{align}
\MoveEqLeft
    \tau_{t+1} \theta_{t+1} - r - \frac{(\tau_{t+1} \theta_{t+1} - r)^\top \rvone}{K} \cdot \rvone = \tau_t \theta_{t} - r - \frac{(\tau_t \theta_{t} - r)^\top \rvone}{K} \cdot \rvone \\
    &\qquad + \left( \tau_{t+1} \theta_{t+1} - \tau_t \theta_{t} \right) + \left( \frac{(\tau_t \theta_{t} - r)^\top \rvone}{K} - \frac{(\tau_{t+1} \theta_{t+1} - r)^\top \rvone}{K} \right) \cdot \rvone \\
    &= \tau_t \theta_{t} - r - \frac{(\tau_t \theta_{t} - r)^\top \rvone}{K} \cdot \rvone + \tau_t \eta_t \cdot H(\pi_{\theta_t}) (r - \tau_t \log{\pi_{\theta_t}} ) + \frac{\left( \tau_t \theta_{t} - \tau_{t+1} \theta_{t+1} \right)^\top \rvone}{K} \cdot \rvone \qquad \left( \text{by \cref{update_rule:decaying_entropy_special}} \right) \\
    &= \left( \identitymatrix - \tau_t \eta_t \cdot H(\pi_{\theta_t}) \right) \left( \tau_t \theta_{t} - r - \frac{(\tau_t \theta_{t} - r)^\top \rvone}{K} \cdot \rvone \right) \qquad \left( H(\pi_{ \theta_{t} } ) \rvone = H(\pi_{ \theta_{t} } )^\top \rvone =  \rvzero, \text{ cf. \cref{eq:contraction_entropy_special_intermediate_1}} \right) \\
    &= \left( \identitymatrix - H(\pi_{\theta_t}) \right) \left( \tau_t \theta_{t} - r - \frac{(\tau_t \theta_{t} - r)^\top \rvone}{K} \cdot \rvone \right) \qquad \left( \eta_t = 1 / \tau_t \right).
\end{align}
Therefore we have,
\begin{align}
\label{eq:rates_decaying_entropy_special_intermediate_2}
\MoveEqLeft
    \left\| \tau_{t+1} \theta_{t+1} - r - \frac{(\tau_{t+1} \theta_{t+1} - r)^\top \rvone}{K} \cdot \rvone \right\|_2 = \left\| \left( \identitymatrix - H(\pi_{\theta_t}) \right) \left( \tau_t \theta_{t} - r - \frac{(\tau_t \theta_{t} - r)^\top \rvone}{K} \cdot \rvone \right) \right\|_2 \\
    &\le \left( 1 - \min_{a}{ \pi_{\theta_t}(a) } \right) \cdot \left\| \tau_t \theta_{t} - r - \frac{(\tau_t \theta_{t} - r)^\top \rvone}{K} \cdot \rvone \right\|_2
    \qquad \left(\text{by \cref{lem:norm_decay_entropy_special}} \right) \\
    &\le \exp\left\{ - \min_{a}{ \pi_{\theta_t}(a) } \right\} \cdot \left\| \tau_t \theta_{t} - r - \frac{(\tau_t \theta_{t} - r)^\top \rvone}{K} \cdot \rvone \right\|_2.
\end{align}
Then we have,
\begin{align}
\MoveEqLeft
    (\pi_{\tau_t}^* - \pi_{\theta_t})^\top r \le \left\| \pi_{\tau_t}^* - \pi_{\theta_t} \right\|_1 \qquad \left( \text{by H{\" o}lder's inequality, and } \|r\|_\infty \le 1 \right) \\
    &\le \left\| \theta_t - \frac{r}{\tau_t} - \frac{(\tau_t \theta_{t} - r)^\top \rvone}{ \tau_t K} \cdot \rvone \right\|_\infty \qquad \left( \text{by \cref{lem:policy_logit_inequality_special}} \right) \\
    &\le \frac{1}{\tau_t} \cdot \left\| \tau_t \theta_{t} - r - \frac{(\tau_t \theta_{t} - r)^\top \rvone}{K} \cdot \rvone \right\|_2 \qquad \left( \|x \|_\infty \le \| x \|_2 \right) \\
    &\le \frac{1}{\tau_t} \cdot \exp\left\{ - \min_{a}{ \pi_{\theta_{t-1}}(a) } \right\} \cdot \left\| \tau_{t-1} \theta_{t-1} - r - \frac{(\tau_{t-1} \theta_{t-1} - r)^\top \rvone}{K} \cdot \rvone \right\|_2 \qquad \left( \text{by \cref{eq:rates_decaying_entropy_special_intermediate_2}} \right) \\
    &\le \frac{1}{\tau_t} \cdot \exp\bigg\{- \sum_{s=1}^{t-1}{ \min_{a}{ \pi_{\theta_s}(a) } }  \bigg\} \cdot \left\| \tau_1 \theta_1 - r - \frac{(\tau_1 \theta_1 - r)^\top \rvone}{K} \cdot \rvone \right\|_2 \\
    &\le \frac{1}{\tau_t} \cdot \exp\bigg\{- \sum_{s=1}^{t-1}{ \min_{a}{ \pi_{\theta_s}(a) } }  \bigg\} \cdot 2 ( \tau_1 \norm{\theta_1}_\infty +1 ) \sqrt{K} \qquad \left( \text{by \cref{eq:zeta_1_upper_bound}} \right) \\
    &= \frac{ \log{t} }{ \exp\left\{ \sum_{s=1}^{t-1}{ \min_{a}{ \pi_{\theta_s}(a) } }  \right\}} \cdot \frac{2 ( \tau_1 \norm{\theta_1}_\infty +1 ) \sqrt{K}}{\alpha \cdot \Delta}. \qedhere
\end{align}
\end{proof}

\subsection{Proofs for \cref{sec:theoretical_understanding_entropy} (Does Entropy Regularization Really Help?)}
\label{sec:proofs_theoretical_understanding_entropy}

\subsubsection{Proofs for the bandit case}
\label{sec:proofs_theoretical_understanding_entropy_bandits}

\textbf{\cref{lem:reverse_lojasiewicz_softmax_special}} (Reversed \L{}ojasiewicz)\textbf{.}
Take any $r\in [0,1]^K$.
Denote $\Delta = r(a^*) - \max_{a \not= a^*}{ r(a) } > 0$. Then,
\begin{align}
    \left\| \frac{d \pi_\theta^\top r}{d \theta} \right\|_2 \le \frac{\sqrt{2}}{\Delta} \cdot (\pi^* - \pi_\theta)^\top r.
\end{align}
\begin{proof}
Note $a^*$ is the optimal action. Denote $\Delta(a) = r(a^*) - r(a)$, and $\Delta = \min_{a \not= a^*}{ \Delta(a) }$.
\begin{align}
    (\pi^* - \pi_\theta)^\top r &= \sum_{a}{ \pi_\theta(a) \cdot r(a^*) } - \sum_{a}{ \pi_\theta(a) \cdot r(a) } \\
    &= \sum_{a \not= a^*}{ \pi_\theta(a) \cdot r(a^*) } - \sum_{a \not= a^*}{ \pi_\theta(a) \cdot r(a) } \\
    &= \sum_{a \not= a^*}{ \pi_\theta(a) \cdot \Delta(a) } \\
    &\ge \sum_{a \not= a^*}{ \pi_\theta(a) \cdot \Delta }.
\end{align}
On the other hand,
\begin{align}
    0 \le r(a^*) - \pi_\theta^\top r &= (\pi^* - \pi_\theta)^\top r = \sum_{a \not= a^*}{ \pi_\theta(a) \cdot \Delta(a) } \le \sum_{a \not= a^*}{ \pi_\theta(a) \cdot 1 } = \sum_{a \not= a^*}{ \pi_\theta(a)}.
\end{align}
Therefore the $\ell_2$ norm of gradient can be upper bounded as
\begin{align}
    \left\| \frac{d \pi_\theta^\top r}{d \theta} \right\|_2 &= \left( \pi_\theta(a^*)^2 \cdot \left[ r(a^*) - \pi_\theta^\top r \right]^2 + \sum_{a \not= a^*}{\left[ \pi_\theta(a)^2 \cdot (r(a) - \pi_\theta^\top r )^2 \right]} \right)^\frac{1}{2} \\
    &\le \left( 1^2 \cdot \left[ \sum_{a \not= a^*}{ \pi_\theta(a)} \right]^2 + \sum_{a \not= a^*}{\left[ \pi_\theta(a)^2 \cdot 1^2 \right]} \right)^\frac{1}{2} \\
    &\le \left( \left[ \sum_{a \not= a^*}{ \pi_\theta(a)} \right]^2 + \left[ \sum_{a \not= a^*}{ \pi_\theta(a)} \right]^2 \right)^\frac{1}{2} 
    \qquad \left(\text{by } \| x \|_2 \le \| x \|_1 \right) \\
    &= \sqrt{2} \cdot \sum_{a \not= a^*}{ \pi_\theta(a)}.
\end{align}
Combining the results, we have 
\begin{equation*}
    \left\| \frac{d \pi_\theta^\top r}{d \theta} \right\|_2 \le \sqrt{2} \cdot \sum_{a \not= a^*}{ \pi_\theta(a)} = \frac{\sqrt{2}}{\Delta} \cdot \Delta \cdot \sum_{a \not= a^*}{ \pi_\theta(a)} \le \frac{\sqrt{2}}{\Delta} \cdot (\pi^* - \pi_\theta)^\top r. \qedhere
\end{equation*}
\end{proof}

\textbf{\cref{thm:lower_bound_softmax_special}} (Lower bound)\textbf{.}
Take any $r\in [0,1]^K$.
For large enough $t \ge 1$, using \cref{update_rule:softmax_special} with learning rate $\eta_t \in ( 0 , 1]$,
\begin{align*}
    (\pi^* - \pi_{\theta_t})^\top r \ge \frac{\Delta^2}{6 \cdot t }.
\end{align*}
\begin{proof}
Denote $\delta_t = (\pi^* - \pi_{\theta_t})^\top r > 0$. Let $\theta_{t+1} = \theta_t + \eta_t \cdot \frac{d \pi_{\theta_t}^\top r}{d \theta_t}$, and $\pi_{\theta_{t+1}} = \softmax(\theta_{t+1})$ be the next policy after one step gradient update. We have,
\begin{align}
\label{eq:lower_bound_softmax_special_intermediate_1}
    \delta_t - \delta_{t+1} &= ( \pi_{\theta_{t+1}} - \pi_{\theta_t}) ^\top r - \Big\langle \frac{d \pi_{\theta_t}^\top r}{d \theta_t}, \theta_{t+1} - \theta_t \Big\rangle + \Big\langle \frac{d \pi_{\theta_t}^\top r}{d \theta_t}, \theta_{t+1} - \theta_t \Big\rangle \\
    &\le \frac{5}{4} \cdot \left\| \theta_{t+1} - \theta_t \right\|_2^2 + \Big\langle \frac{d \pi_{\theta_t}^\top r}{d \theta_t}, \theta_{t+1} - \theta_t \Big\rangle 
    \qquad \left( \text{by \cref{lem:smoothness_softmax_special}} \right) \\
    &= \left( \frac{5 \eta_t^2}{4} + \eta_t \right) \cdot \left\| \frac{d \pi_{\theta_t}^\top r}{d \theta_t} \right\|_2^2 
    \qquad \left( \text{by } \theta_{t+1} = \theta_t + \eta_t \cdot \frac{d \pi_{\theta_t}^\top r}{d \theta_t} \right)\\
    &\le \frac{9}{2} \cdot \frac{1}{ \Delta^2} \cdot \delta_t^2. 
    \qquad \left( \text{by } \eta_t \in (0, 1] \text{ and by \cref{lem:reverse_lojasiewicz_softmax_special}} \right)
\end{align}
According to convergence result \cref{thm:final_rates_softmax_special} we have $\delta_t > 0$, $\delta_t \to 0$ as $t \to \infty$. We prove that for all large enough $t \ge 1$, $\delta_t \le \frac{10}{9} \cdot \delta_{t+1}$ by contradiction. Suppose $\delta_t > \frac{10}{9} \cdot \delta_{t+1}$.
\begin{align}
\label{eq:lower_bound_softmax_special_intermediate_2}
\MoveEqLeft
    \delta_{t+1} \ge \delta_t - \frac{9}{2} \cdot \frac{1}{ \Delta^2} \cdot \delta_t^2 \\
    &> \frac{10}{9} \cdot \delta_{t+1} - \frac{9}{2} \cdot \frac{1}{ \Delta^2} \cdot \left( \frac{10}{9} \cdot \delta_{t+1} \right)^2 
    \qquad \left( \text{since } f(x) = x - a x^2 \text{ is increasing for all } x < \frac{1}{2a} \text{ and } a > 0 \right) \\
    &= \frac{10}{9} \cdot \delta_{t+1} - \frac{50 }{9} \cdot \frac{1}{\Delta^2} \cdot \delta_{t+1}^2,
\end{align}
which implies $\delta_{t+1} > \frac{\Delta^2}{50 }$ for large enough $t \ge 1$. This is a contradiction with $\delta_t \to 0$ as $t \to \infty$. Now we have $\delta_t \le \frac{10}{9} \cdot \delta_{t+1}$. Divide both sides of $\delta_t - \delta_{t+1} \le \frac{9 }{2} \cdot \frac{1}{ \Delta^2} \cdot \delta_t^2$ by $\delta_t \cdot \delta_{t+1}$,
\begin{align}
    \frac{1}{\delta_{t+1}} - \frac{1}{\delta_t} \le \frac{9 }{2} \cdot \frac{1}{ \Delta^2} \cdot \frac{\delta_t}{\delta_{t+1}} \le \frac{9 }{2} \cdot \frac{1}{ \Delta^2} \cdot \frac{10}{9} = \frac{5 }{ \Delta^2}.
\end{align}
Summing up from $T_1$ (some large enough time) to $T_1 + t$, we have
\begin{align}
    \frac{1}{\delta_{T_1+t}} - \frac{1}{\delta_{T_1}} \le \frac{5}{\Delta^2} \cdot (t - 1) \le \frac{5}{\Delta^2} \cdot t.
\end{align}
Since $T_1$ is a finite time, $\delta_{T_1} \ge 1/C$ for some constant $C > 0$. Rearranging, we have
\begin{align}
    (\pi^* - \pi_{\theta_{T_1+t}})^\top r  = \delta_{T_1+t} \ge \frac{1}{ \frac{1}{\delta_{T_1}} + \frac{5}{\Delta^2} \cdot t} \ge \frac{1}{ C + \frac{5}{\Delta^2} \cdot t } \ge \frac{1}{ C + \frac{5 }{\Delta^2} \cdot (T_1 + t) }.
\end{align}
By abusing notation $t \coloneqq T_1 + t$ and $C \le \frac{t}{\Delta^2}$, we have
\begin{align}
    (\pi^* - \pi_{\theta_t})^\top r \ge \frac{1}{ C + \frac{5}{\Delta^2} \cdot t } \ge \frac{1}{  \frac{t}{\Delta^2} + \frac{5}{\Delta^2} \cdot t } = \frac{\Delta^2}{6 \cdot t },
\end{align}
for all large enough $t \ge 1$.
\end{proof}

\subsubsection{Proofs for general MDPs}
\label{sec:proofs_theoretical_understanding_entropy_general_mdps}

\textbf{\cref{thm:lower_bound_softmax_general}} (Lower bound)\textbf{.} 
Take any MDP.
For large enough $t \ge 1$, using \cref{alg:policy_gradient_softmax} with $\eta_t \in ( 0, 1] $,
\begin{align}
    V^*(\mu) - V^{\pi_{\theta_t}}(\mu) \ge \frac{ (1- \gamma)^5 \cdot (\Delta^*)^2}{12 \cdot t},
\end{align}
where $\Delta^* = \min_{s \in \gS, a \not= a^*(s)}\{ Q^*(s, a^*(s)) - Q^*(s, a) \} > 0$ is the optimal value gap of the MDP, and $a^*(s) = \argmax_{a}{ \pi^*(a | s) }$ is the action that the optimal policy selects under state $s$.
\begin{proof}
Suppose \cref{alg:policy_gradient_softmax} can converge faster than $O(1/t)$ for general MDPs, then it can converge faster than $O(1/t)$ for any one-state MDPs, which are special cases of general MDPs. This is a contradiction with \cref{thm:lower_bound_softmax_special}.

The above one-sentence argument implies a $\Omega(1/t)$ rate lower bound. To calculate the constant in the lower bound, we need results similar to \cref{lem:reverse_lojasiewicz_softmax_special}. According to the reversed \L{}ojasiewicz inequality of  \cref{lem:reverse_lojasiewicz_softmax_general},
\begin{align}
    \left\| \frac{\partial V^{\pi_{\theta_t}}(\mu)}{\partial \theta_t } \right\|_2 \le \frac{1}{1 - \gamma} \cdot \frac{\sqrt{2}}{\Delta^*} \cdot \delta_t,
\end{align}
where $\delta_t = V^*(\mu) - V^{\pi_{\theta_t}}(\mu) > 0$. Let $\theta_{t+1} = \theta_t + \eta_t \cdot \frac{\partial V^{\pi_{\theta_t}}(\mu)}{\partial \theta_t }$, and $\pi_{\theta_{t+1}}(\cdot | s) = \softmax(\theta_{t+1}(s, \cdot))$, $\forall s \in \gS$ be the next policy after one step gradient update. Using similar calculations as in \cref{eq:lower_bound_softmax_special_intermediate_1},
\begin{align}
    \delta_t - \delta_{t+1} &= V^{\pi_{\theta_{t+1}}}(\mu) - V^{\pi_{\theta_t}}(\mu) - \Big\langle \frac{\partial V^{\pi_{\theta_t}}(\mu)}{\partial \theta_t }, \theta_{t+1} - \theta_t \Big\rangle + \Big\langle \frac{\partial V^{\pi_{\theta_t}}(\mu)}{\partial \theta_t }, \theta_{t+1} - \theta_t \Big\rangle \\
    &\le \frac{4}{(1- \gamma)^3} \cdot \left\| \theta_{t+1} - \theta_t \right\|_2^2 + \Big\langle \frac{\partial V^{\pi_{\theta_t}}(\mu)}{\partial \theta_t }, \theta_{t+1} - \theta_t \Big\rangle \qquad \left(
    \text{by \cref{lem:smoothness_softmax_general}} \right) \\
    &= \left( \frac{4 \eta_t^2}{(1- \gamma)^3} + \eta_t \right) \cdot \left\| \frac{\partial V^{\pi_{\theta_t}}(\mu)}{\partial \theta_t } \right\|_2^2 
    \qquad \left( \text{by } \theta_{t+1} = \theta_t + \eta_t \cdot \frac{\partial V^{\pi_{\theta_t}}(\mu)}{\partial \theta_t } \right) \\
    &\le \frac{10}{(1-\gamma)^5} \cdot \frac{1}{ (\Delta^*)^2} \cdot \delta_t^2. 
    \qquad \left( \text{by } \eta_t \in (0, 1] \text{ and by \cref{lem:reverse_lojasiewicz_softmax_general}} \right)
\end{align}
According to \cref{thm:final_rates_softmax_general}, we have $\delta_t > 0$, $\delta_t \to 0$ as $t \to \infty$. Using similar arguments as in \cref{eq:lower_bound_softmax_special_intermediate_2}, we can show that for all large enough $t \ge 1$, $\delta_t \le \frac{11}{10} \cdot \delta_{t+1}$. Divide both sides of $\delta_t - \delta_{t+1} \le \frac{10}{(1-\gamma)^5} \cdot \frac{1}{ (\Delta^*)^2} \cdot \delta_t^2$ by $\delta_t \cdot \delta_{t+1}$,
\begin{align}
    \frac{1}{\delta_{t+1}} - \frac{1}{\delta_t} \le \frac{10}{(1-\gamma)^5} \cdot \frac{1}{ (\Delta^*)^2} \cdot \frac{\delta_t}{\delta_{t+1}} \le \frac{10}{(1-\gamma)^5} \cdot \frac{1}{ (\Delta^*)^2} \cdot \frac{11}{10} = \frac{11 }{ (1- \gamma)^5 \cdot (\Delta^*)^2}.
\end{align}
Using similar calculations as in the proof of \cref{thm:lower_bound_softmax_special}, we have,
\begin{align}
    V^*(\mu) - V^{\pi_{\theta_t}}(\mu) = \delta_t \ge \frac{ (1- \gamma)^5 \cdot (\Delta^*)^2}{12 \cdot t},
\end{align}
for all large enough $t \ge 1$.
\end{proof}

\subsubsection{Proofs for the non-uniform \L{}ojasiewicz degree}
\label{sec:proofs_theoretical_understanding_entropy_lojasiewicz_degree}

\textbf{\cref{prop:lojasiewicz_degree_softmax}.}
Let $r\in [0,1]^K$ be arbitrary and consider $\theta\mapsto \expectation_{a \sim \pi_{\theta}}{ [ r(a) ] }$.
The non-uniform \L{}ojasiewicz degree of this map with constant $C(\theta) = \pi_\theta(a^*)$ is zero.
\begin{proof}
We prove by contradiction. Suppose the \L{}ojasiewicz degree of $\expectation_{a \sim \pi_{\theta}}{ [ r(a) ] }$ can be larger than $0$. Then there exists $\xi > 0$, such that,
\begin{align}
    \left\| \frac{d \pi_\theta^\top r}{d \theta} \right\|_2 \ge C(\theta) \cdot \left[ (\pi^* - \pi_\theta)^\top r \right]^{1-\xi}.
\end{align}
Consider the following example, $r = (0.6, 0.4, 0.2)^\top$, $\pi_\theta = (1 - 3 \epsilon, 2 \epsilon, \epsilon)^\top$ with small number $\epsilon > 0$.
\begin{align}
    (\pi^* - \pi_\theta)^\top r = r(a^*) - \pi_\theta^\top r = 0.6 - \left( 0.6 - 0.8 \epsilon \right) = 0.8 \cdot \epsilon.
\end{align}
According to the reversed \L{}ojasiewicz inequality of  \cref{lem:reverse_lojasiewicz_softmax_special},
\begin{align}
    \left\| \frac{d \pi_\theta^\top r}{d \theta} \right\|_2 \le \frac{\sqrt{2}}{\Delta} \cdot (\pi^* - \pi_\theta)^\top r = \frac{\sqrt{2}}{2} \cdot (\pi^* - \pi_\theta)^\top r \le \frac{1.5}{2} \cdot (\pi^* - \pi_\theta)^\top r = 0.6 \cdot \epsilon.
\end{align}
Also note that $\pi_\theta(a^*) = 1 - 3 \epsilon > 1 / 4$. Then for $\xi \in (0, 1]$, we have
\begin{align}
    \left\| \frac{d \pi_\theta^\top r}{d \theta} \right\|_2 \le 0.6 \cdot \epsilon = \frac{1}{4} \cdot 3 \cdot 0.8 \cdot \epsilon < \pi_\theta(a^*) \cdot 3 \cdot 0.8 \cdot \epsilon = C(\theta) \cdot 3 \cdot 0.8 \cdot \epsilon.
\end{align}
Next, since $\epsilon > 0$ can be very small,
\begin{align}
    \left\| \frac{d \pi_\theta^\top r}{d \theta} \right\|_2 < C(\theta) \cdot 3 \cdot 0.8 \cdot \epsilon &= C(\theta) \cdot 3 \cdot (0.8 \cdot \epsilon)^\xi \cdot (0.8 \cdot \epsilon)^{1-\xi} \\
    &< C(\theta) \cdot (0.8 \cdot \epsilon)^{1-\xi} = C(\theta) \cdot \left[ (\pi^* - \pi_\theta)^\top r \right]^{1-\xi},
\end{align}
where the second inequality is by $(0.8 \cdot \epsilon)^\xi < 1 / 3$ for small $\epsilon > 0$ since $\xi > 0$. This is a contradiction with the assumption. Therefore the \L{}ojasiewicz degree $\xi$ cannot be larger than $0$.
\end{proof}

\textbf{\cref{prop:lojasiewicz_degree_entropy}.}
Fix $\tau>0$.
With $C(\theta) = \sqrt{2 \tau} \cdot  \min_{a}{ \pi_\theta(a) }$, the \L{}ojasiewicz degree of $\theta \mapsto \expectation_{a \sim \pi_{\theta}}{ \left[ r(a) - \tau \log{\pi_\theta(a)} \right] }$ is at least $1/2$.
\begin{proof}
Denote $\delta_\theta = \expectation_{a \sim \pi_{\tau}^*}{ \left[ r(a) - \tau \log{\pi_{\tau}^*(a)} \right] } - \expectation_{a \sim \pi_{\theta}}{ \left[ r(a) - \tau \log{\pi_\theta(a)} \right] }$ as the soft sub-optimality. We have,
\begin{align}
\label{eq:entropy_lojsiewicz_degree_special_intermediate_1}
    \delta_\theta &= \expectation_{a \sim \pi_{\tau}^*}{ \left[ r(a) - \tau \log{\pi_{\tau}^*(a)} \right] } - \expectation_{a \sim \pi_{\theta}}{ \left[ r(a) - \tau \log{\pi_{\tau}^*(a)} \right] }  - \expectation_{a \sim \pi_{\theta}}{ \left[ \tau \log{\pi_{\tau}^*(a)} - \tau \log{\pi_\theta(a)} \right] } \\
    &= \tau \log \sum_{a}{ \exp\{ r(a) / \tau \} } - \tau \log \sum_{a}{ \exp\{ r(a) / \tau \} } + \tau \cdot \KL(\pi_\theta \| \pi_\tau^*) 
    \qquad \left( \text{since } \pi_\tau^* = \softmax(r/\tau) \right) \\
    &= \tau \cdot \KL(\pi_\theta \| \pi_\tau^*) \\
    &\le \frac{\tau}{2} \cdot \left\| \frac{r}{ \tau} - \theta - \frac{ (r /\tau - \theta)^\top \rvone}{K} \cdot \rvone \right\|_\infty^2 \qquad \left( \text{by \cref{lem:kl_logit_inequality}} \right) \\
    &= \frac{1}{2 \tau } \cdot \left\| r - \tau \theta - \frac{ (r - \tau \theta)^\top \rvone}{K} \cdot \rvone \right\|_\infty^2.
\end{align}
Next, the entropy regularized policy gradient w.r.t. $\theta$ is
\begin{align}
    \frac{ d \{ \pi_\theta^\top ( r - \tau \log{\pi_\theta}) \} }{d \theta} &= H(\pi_\theta) (r - \tau \log{\pi_\theta}) \\
    &= H(\pi_\theta) \left(r - \tau \theta + \tau \log{ \sum_{a}{\exp\{\theta(a)\}} } \cdot \rvone \right) \\
    &= H(\pi_\theta) \left(r - \tau \theta \right) \\
    &= H(\pi_\theta) \left( r - \tau \theta - \frac{ (r - \tau \theta)^\top \rvone}{K} \cdot \rvone \right),
\end{align}
where the last two equations are by $H(\pi_\theta) \rvone = \rvzero$ as shown in \cref{lem:golub_rank_one_perturb}.
Then we have,
\begin{align}
    \left\| \frac{ d \{ \pi_\theta^\top ( r - \tau \log{\pi_\theta}) \} }{d \theta} \right\|_2 &= \left\| H(\pi_\theta) \left( r - \tau \theta - \frac{ (r - \tau \theta)^\top \rvone}{K} \cdot \rvone \right) \right\|_2 \\
    &\ge \min_{a}{ \pi_\theta(a) } \cdot \left\| r - \tau \theta - \frac{ (r - \tau \theta)^\top \rvone}{K} \cdot \rvone \right\|_2 \qquad \left( \text{by \cref{lem:norm_decay_entropy_special}} \right) \\
    &\ge \min_{a}{ \pi_\theta(a) } \cdot \left\| r - \tau \theta - \frac{ (r - \tau \theta)^\top \rvone}{K} \cdot \rvone \right\|_\infty \\
    &\ge \min_{a}{ \pi_\theta(a) } \cdot \sqrt{2 \tau } \cdot  \sqrt{\delta_\theta} \qquad \left( \text{by \cref{eq:entropy_lojsiewicz_degree_special_intermediate_1}} \right) \\
    &= \sqrt{2 \tau } \cdot \min_{a}{ \pi_\theta(a) } \cdot \left( \expectation_{a \sim \pi_{\tau}^*}{ \left[ r(a) - \tau \log{\pi_{\tau}^*(a)} \right] } - \expectation_{a \sim \pi_{\theta}}{ \left[ r(a) - \tau \log{\pi_\theta(a)} \right] } \right)^\frac{1}{2}, 
\end{align}
which means the \L{}ojasiewicz degree of $\expectation_{a \sim \pi_{\theta}}{ \left[ r(a) - \tau \log{\pi_\theta(a)} \right] }$ is $1/2$ and $C(\theta) = \sqrt{2 \tau } \cdot \min_{a}{ \pi_\theta(a) }$.
\end{proof}

\section{Miscellaneous Extra Supporting Results}
\label{sec:supporting_lemmas}

\begin{lemma}[Ascent lemma for smooth function]
\label{lem:ascent_lemma_smooth_function}
Let $f:\R^d \to \R$ be a $\beta$-smooth function, $\theta\in \R^d$ 
and $\theta' = \theta + \frac{1}{\beta} \cdot \frac{\partial f(\theta)}{\partial \theta}$. We have,
\begin{align}
    f(\theta) - f(\theta') \le - \frac{1}{2 \beta} \cdot \left\| \frac{\partial f(\theta)}{\partial \theta} \right\|_2^2.
\end{align}
\end{lemma}
\begin{proof}
According to the definition of smoothness, we have,
\begin{align}
    \left| f(\theta') - f(\theta) - \Big\langle \frac{\partial f(\theta)}{\partial \theta}, \theta' - \theta \Big\rangle \right| \le \frac{\beta}{2} \cdot \| \theta' - \theta \|_2^2,
\end{align}
which implies,
\begin{align}
    f(\theta) - f(\theta') &\le - \Big\langle \frac{\partial f(\theta)}{\partial \theta}, \theta' - \theta \Big\rangle + \frac{\beta}{2} \cdot \| \theta' - \theta \|_2^2 \\
    &= - \frac{1}{\beta} \cdot \left\| \frac{\partial f(\theta)}{\partial \theta} \right\|_2^2 + \frac{\beta}{2} \cdot \frac{1}{\beta^2} \cdot \left\| \frac{\partial f(\theta)}{\partial \theta} \right\|_2^2 \qquad \left( \theta' = \theta + \frac{1}{\beta} \cdot \frac{\partial f(\theta)}{\partial \theta} \right) \\
    &= - \frac{1}{2 \beta} \cdot \left\| \frac{\partial f(\theta)}{\partial \theta} \right\|_2^2. \qedhere
\end{align}
\end{proof}

\begin{lemma}[First performance difference lemma \citep{kakade2002approximately}]
\label{lem:performance_difference_general}
For any policies $\pi$ and $\pi^\prime$,
\begin{align}
    V^{\pi^\prime}(\rho) - V^{\pi}(\rho) &= \frac{1}{1 - \gamma} \sum_{s}{ d_\rho^{\pi^\prime}(s) \sum_{a}{ \left( \pi^\prime(a | s) - \pi(a | s) \right) \cdot Q^{\pi}(s,a) } }\\
    &= \frac{1}{1 - \gamma} \sum_{s}{ d_{\rho}^{\pi^\prime}(s)  \sum_{a}{ \pi^\prime(a | s) \cdot A^{\pi}(s, a) } }.
\end{align}
\end{lemma}
\begin{proof}
According to the definition of value function,
\begin{align}
    V^{\pi^\prime}(s) - V^{\pi}(s) &= \sum_{a}{ \pi^\prime(a | s) \cdot Q^{\pi^\prime}(s,a) } - \sum_{a}{ \pi(a | s) \cdot Q^{\pi}(s,a) } \\
    &= \sum_{a}{ \pi^\prime(a | s) \cdot \left( Q^{\pi^\prime}(s,a) - Q^{\pi}(s,a) \right) } + \sum_{a}{ \left( \pi^\prime(a | s) - \pi(a | s) \right) \cdot Q^{\pi}(s,a) } \\
    &= \sum_{a}{ \left( \pi^\prime(a | s) - \pi(a | s) \right) \cdot Q^{\pi}(s,a) } + \gamma \sum_{a}{ \pi^\prime(a | s) \sum_{s^\prime}{  \gP( s^\prime | s, a) \cdot \left[ V^{\pi^\prime}(s^\prime) -  V^{\pi}(s^\prime)  \right] } } \\
    &= \frac{1}{1 - \gamma} \sum_{s^\prime}{ d_{s}^{\pi^\prime}(s^\prime) \sum_{a^\prime}{ \left( \pi^\prime(a^\prime | s^\prime) - \pi(a^\prime | s^\prime) \right) \cdot Q^{\pi}(s^\prime, a^\prime) }  } \\
    &= \frac{1}{1 - \gamma} \sum_{s^\prime}{ d_{s}^{\pi^\prime}(s^\prime) \sum_{a^\prime}{ \pi^\prime(a^\prime | s^\prime) \cdot \left( Q^{\pi}(s^\prime, a^\prime) - V^{\pi}(s^\prime) \right) }  } \\
    &= \frac{1}{1 - \gamma} \sum_{s^\prime}{ d_{s}^{\pi^\prime}(s^\prime)  \sum_{a^\prime}{ \pi^\prime(a^\prime | s^\prime) \cdot A^{\pi}(s^\prime, a^\prime) } }. \qedhere
\end{align}
\end{proof}

\begin{lemma}[Second performance difference lemma]
\label{lem:performance_difference_general_second}
For any policies $\pi$ and $\pi^\prime$,
\begin{align}
    V^{\pi^\prime}(\rho) - V^{\pi}(\rho) = \frac{1}{1 - \gamma} \sum_{s}{ d_\rho^{\pi}(s) \sum_{a}{ \left( \pi^\prime(a | s) - \pi(a | s) \right) \cdot Q^{\pi^\prime}(s,a) } }.
\end{align}
\end{lemma}
\begin{proof}
According to the definition of value function,
\begin{align}
    V^{\pi^\prime}(s) - V^{\pi}(s) &= \sum_{a}{ \pi^\prime(a | s) \cdot Q^{\pi^\prime}(s,a) } - \sum_{a}{ \pi(a | s) \cdot Q^{\pi}(s,a) } \\
    &= \sum_{a}{ \left( \pi^\prime(a | s) - \pi(a | s) \right) \cdot Q^{\pi^\prime}(s,a) } + \sum_{a}{ \pi(a | s) \cdot \left(  Q^{\pi^\prime}(s,a) - Q^{\pi}(s,a) \right) } \\
    &= \sum_{a}{ \left( \pi^\prime(a | s) - \pi(a | s) \right) \cdot Q^{\pi^\prime}(s,a) } + \gamma \sum_{a}{ \pi(a | s) \sum_{s^\prime}{  \gP( s^\prime | s, a) \cdot \left[ V^{\pi^\prime}(s^\prime) -  V^{\pi}(s^\prime)  \right] } } \\
    &= \frac{1}{1 - \gamma} \sum_{s^\prime}{ d_{s}^{\pi}(s^\prime) \sum_{a^\prime}{ \left( \pi^\prime(a^\prime | s^\prime) - \pi(a^\prime | s^\prime) \right) \cdot Q^{\pi^\prime}(s^\prime, a^\prime) }  }. \qedhere
\end{align}
\end{proof}

\begin{lemma}[Value sub-optimality lemma]
\label{lem:value_suboptimality}
For any policy $\pi$,
\begin{align}
    V^*(\rho) - V^{\pi}(\rho) = \frac{1}{1 - \gamma} \sum_{s}{ d_\rho^{\pi}(s) \sum_{a}{ \left( \pi^*(a | s) - \pi(a | s) \right) \cdot Q^*(s,a) } }.
\end{align}
\end{lemma}
\begin{proof}
According to the second performance difference lemma of \cref{lem:performance_difference_general_second}, the result immediately holds.
\end{proof}

\begin{lemma}[Spectrum of H matrix]
\label{lem:golub_rank_one_perturb}
Let $\pi \in \Delta(\gA)$. Denote $H(\pi) = \diagonalmatrix(\pi) - \pi \pi^\top$. Let
\begin{align}
    \pi(1) \le \pi(2) \le \cdots \le \pi(K).
\end{align}
Denote the eigenvalues of $H(\pi)$ as
\begin{align}
    \lambda_1 \le \lambda_2 \le \cdots \le \lambda_K.
\end{align}
Then we have,
\begin{align}
    \lambda_1 &= 0, \\
    \pi(i-1) \le \lambda_i &\le \pi(i), \ i = 2, 3, \dots, K.
\end{align}
\end{lemma}
\begin{proof}
According to \citet[Section 5]{golub1973some},
\begin{align}
    \pi(1) - \pi^\top \pi &\le \lambda_1 \le \pi(1), \\
    \pi(i-1) &\le \lambda_i \le \pi(i), \ i = 2, 3, \dots, K.
\end{align}
We show $\lambda_1 = 0$. Note
\begin{align}
    H(\pi) \rvone =  ( \diagonalmatrix(\pi) - \pi \pi^\top ) \rvone = \pi - \pi = 0 \cdot \rvone.
\end{align}
Thus $\rvone$ is an eigenvector of $H(\pi)$ which corresponds to eigenvalue $0$. Furthermore, for any vector $x \in \sR^K$,
\begin{align}
    x^\top H(\pi) x &= \expectation_{a \sim \pi}[x(a)^2] - \left( \expectation_{a \sim \pi}[x(a)] \right)^2 = \Var_{a \sim \pi}[x(a)] \ge 0,
\end{align}
which means all the eigenvalues of $H(\pi)$ are non-negative.
\end{proof}

\begin{lemma}
\label{lem:norm_decay_entropy_special}
Let $\pi \in \Delta(\gA)$. Denote $H(\pi) = \diagonalmatrix(\pi) - \pi \pi^\top$. For any vector $x \in \sR^K$,
\begin{align}
    \left\| \left( \identitymatrix - H(\pi) \right) \left(x - \frac{x^\top \rvone}{K} \cdot \rvone \right) \right\|_2 &\le \left( 1- \min_{a}{ \pi(a) } \right) \cdot \left\| x - \frac{x^\top \rvone}{K} \cdot \rvone \right\|_2, \\
    \left\| H(\pi) \left(x - \frac{x^\top \rvone}{K} \cdot \rvone \right) \right\|_2 &\ge \min_{a}{ \pi(a) } \cdot \left\| x - \frac{x^\top \rvone}{K} \cdot \rvone \right\|_2.
\end{align}
\end{lemma}
\begin{proof}
$x$ can be written as linear combination of eigenvectors of $H(\pi)$,
\begin{align}
    x &= a_1 \cdot \frac{\rvone}{\sqrt{K}} + a_2 v_2 + \cdots + a_K v_K \\
    &= \frac{x^\top \rvone}{K} \cdot \rvone + a_2 v_2 + \cdots + a_K v_K.
\end{align}
Since $H(\pi)$ is symmetric, $\left\{ \frac{\rvone}{\sqrt{K}}, v_2, \dots, v_K \right\}$ are orthonormal. The last equation is because the representation is unique, and
\begin{align}
    a_1 = x^\top \frac{\rvone}{\sqrt{K}} = \frac{x^\top \rvone}{\sqrt{K}}.
\end{align}
Denote
\begin{align}
    x^\prime = x - \frac{x^\top \rvone}{K} \cdot \rvone = a_2 v_2 + \cdots + a_K v_K.
\end{align}
We have
\begin{align}
    \| x^\prime \|_2^2 = a_2^2 + \cdots + a_K^2.
\end{align}
On the other hand,
\begin{align}
    ( \identitymatrix - H(\pi) ) x^\prime = a_2 (1 - \lambda_2 ) v_2 + \cdots + a_K ( 1-  \lambda_K ) v_K.
\end{align}
Therefore
\begin{align}
    \| ( \identitymatrix - H(\pi) ) x^\prime \|_2 &= \left( a_2^2 (1-\lambda_2)^2+ \cdots + a_K^2 (1- \lambda_K)^2 \right)^\frac{1}{2} \\
    &\le \left( ( a_2^2 + \cdots + a_K^2 ) \cdot (1 - \lambda_2)^2 \right)^\frac{1}{2} \\
    &= (1 - \lambda_2) \cdot \| x^\prime \|_2 \\
    &\le \left( 1 - \min_{a}{ \pi(a) } \right) \cdot \| x^\prime \|_2,
\end{align}
where the first inequality is by $0 \le \pi(1) \le \lambda_2 \le \cdots \le \lambda_K \le \pi(K) \le 1$, and the last inequality is according to $\lambda_2 \ge \pi(1) = \min_{a}{ \pi(a) }$, and both are shown in \cref{lem:golub_rank_one_perturb}. Similarly,
\begin{align}
    \|  H(\pi) x^\prime \|_2 &= \left( a_2^2 \lambda_2^2+ \cdots + a_K^2  \lambda_K^2 \right)^\frac{1}{2} \\
    &\ge \left( ( a_2^2 + \cdots + a_K^2 ) \cdot \lambda_2^2 \right)^\frac{1}{2} \\
    &= \lambda_2 \cdot \| x^\prime \|_2 \\
    &\ge \min_{a}{ \pi(a) } \cdot \| x^\prime \|_2. \qedhere
\end{align}
\end{proof}

\begin{lemma}
\label{lem:policy_logit_inequality_special}
Let $\pi_\theta = \softmax(\theta)$ and $\pi_{\theta^\prime} = \softmax(\theta^\prime)$. Then for any constant $c \in \sR$,
\begin{align}
    \left\| \pi_\theta - \pi_\theta^\prime \right\|_1 \le \left\| \theta^\prime - \theta - c \cdot \rvone \right\|_\infty.
\end{align}
\end{lemma}
\begin{proof}
This results improves the results of $\left\| \pi_\theta - \pi_{\theta^\prime} \right\|_\infty \le 2 \cdot \left\| \theta - \theta^\prime \right\|_\infty$ in \citet[Lemma 5]{xiao2019maximum}.
According to the $\ell_1$ norm strong convexity of negative entropy over probabilistic simplex, i.e., for any policies $ \pi$, $\pi^\prime$,
\begin{align}
    \pi^\top \log{\pi} \ge {\pi^\prime}^\top \log{{\pi^\prime}} + (\pi - \pi^\prime )^\top \log{ \pi^\prime } + \frac{1}{2} \cdot \left\| \pi^\prime - \pi \right\|_1^2,
\end{align}
we have (letting $\pi = \pi_\theta$, and $\pi^\prime = \pi_{\theta^\prime}$),
\begin{align}
    \KL(\pi_\theta \| \pi_{\theta^\prime}) &= \pi_\theta^\top \log{\pi_\theta} - { \pi_{\theta^\prime} }^\top \log{\pi_{\theta^\prime}} - (\pi_\theta - \pi_{\theta^\prime} )^\top \log{\pi_{\theta^\prime}} \ge \frac{1}{2} \cdot \left\| \pi_\theta - \pi_\theta^\prime \right\|_1^2,
\end{align}
which is the Pinsker's inequality. Then we have,
\begin{align}
    \left\| \pi_\theta - \pi_\theta^\prime \right\|_1 &\le \sqrt{2 \cdot \KL(\pi_\theta \| \pi_{\theta^\prime})} \\
    &\le \sqrt{2 \cdot \frac{1}{2} \cdot \left\| \theta^\prime - \theta - c \cdot \rvone \right\|_\infty^2 } \qquad \left( \text{by \cref{lem:kl_logit_inequality}} \right) \\
    &= \left\| \theta^\prime - \theta - c \cdot \rvone \right\|_\infty. \qedhere
\end{align}
\end{proof}

\begin{lemma}[Soft performance difference lemma]
\label{lem:soft_performance_difference_general}
For any policies $\pi$ and $\pi^\prime$,
\begin{align}
    \tilde{V}^{\pi^\prime}(\rho) - \tilde{V}^{\pi}(\rho) = \frac{1}{1 - \gamma} \sum_{s}{ d_\rho^{\pi}(s) \cdot \left[ \sum_{a}{ \left( \pi^\prime(a | s) - \pi(a | s) \right) \cdot \left[ \tilde{Q}^{\pi^\prime}(s,a) - \tau \log{\pi^\prime(a | s)} \right] } + \tau \cdot \KL(\pi(\cdot| s) \| \pi^\prime(\cdot| s) ) \right] }.
\end{align}
\end{lemma}
\begin{proof}
According to the definition of soft value function,
\begin{align}
\MoveEqLeft
    \tilde{V}^{\pi^\prime}(s) - \tilde{V}^{\pi}(s) = \sum_{a}{ \pi^\prime(a | s) \cdot \left[ \tilde{Q}^{\pi^\prime}(s,a) - \tau \log{ \pi^\prime(a | s) } \right] } - \sum_{a}{ \pi(a | s) \cdot \left[ \tilde{Q}^{\pi}(s,a) - \tau \log{ \pi(a | s) } \right]}\\
    &= \sum_{a}{ \left( \pi^\prime(a | s) - \pi(a | s) \right) \cdot \left[ \tilde{Q}^{\pi^\prime}(s,a) - \tau \log{ \pi^\prime(a | s) } \right] } + \sum_{a}{ \pi(a | s) \cdot \left[  \tilde{Q}^{\pi^\prime}(s,a) - \tau \log{ \pi^\prime(a | s) } - \tilde{Q}^{\pi}(s,a) + \tau \log{ \pi(a | s) } \right] } \\
    &= \sum_{a}{ \left( \pi^\prime(a | s) - \pi(a | s) \right) \cdot \left[ \tilde{Q}^{\pi^\prime}(s,a) - \tau \log{ \pi^\prime(a | s) } \right] } + \tau \KL(\pi(\cdot| s) \| \pi^\prime(\cdot| s) ) + \gamma \sum_{a}{ \pi(a | s) \sum_{s^\prime}{  \gP( s^\prime | s, a) \cdot \left[ \tilde{V}^{\pi^\prime}(s^\prime) -  \tilde{V}^{\pi}(s^\prime)  \right] } } \\
    &= \frac{1}{1 - \gamma} \sum_{s^\prime}{ d_{s}^{\pi}(s^\prime) \cdot \left[ \sum_{a^\prime}{ \left( \pi^\prime(a^\prime | s^\prime) - \pi(a^\prime | s^\prime) \right) \cdot \left[ \tilde{Q}^{\pi^\prime}(s^\prime,a^\prime) - \tau \log{ \pi^\prime(a^\prime | s^\prime) } \right] } + \tau \cdot \KL(\pi(\cdot| s^\prime) \| \pi^\prime(\cdot| s^\prime) ) \right]  }. \qedhere
\end{align}
\end{proof}

\begin{lemma}[Soft sub-optimality lemma]
\label{lem:soft_suboptimality}
For any policy $\pi$,
\begin{align}
    \tilde{V}^{\pi_\tau^*}(\rho) - \tilde{V}^{{\pi}}(\rho) = \frac{1}{1 - \gamma} \sum_{s}{ \left[  d_{\rho}^{\pi}(s) \cdot \tau \cdot { \KL(\pi( \cdot | s) \| \pi_\tau^*(\cdot | s) ) } \right] }.
\end{align}
\end{lemma}
\begin{proof}
According to \citet[Theorem 1]{nachum2017bridging}, $\forall (s,a)$,
\begin{align}
\label{eq:path_consistency_conditions_appendix_1}
    \tau \log{ \pi_\tau^*(a | s) } = \tilde{Q}^{\pi_\tau^*}(s, a) - \tilde{V}^{\pi_\tau^*}(s).
\end{align}
According to the soft performance difference lemma of \cref{lem:soft_performance_difference_general},
\begin{align}
\MoveEqLeft
    \tilde{V}^{\pi_\tau^*}(s) - \tilde{V}^{{\pi}}(s) = \frac{1}{1 - \gamma} \sum_{s^\prime}{ d_{s}^{\pi}(s^\prime) \cdot \left[ \sum_{a^\prime}{ \left( \pi_\tau^*(a^\prime | s^\prime) - \pi(a^\prime | s^\prime) \right) \cdot \left[ \tilde{Q}^{\pi_\tau^*}(s^\prime,a^\prime) - \tau \log{ \pi_\tau^*(a^\prime | s^\prime) } \right] } + \tau \cdot \KL(\pi(\cdot| s^\prime) \| {\pi_\tau^*}(\cdot| s^\prime) ) \right]  }  \\
    &= \frac{1}{1 - \gamma} \sum_{s^\prime}{ d_{s}^{\pi}(s^\prime) \cdot \left[ \sum_{a^\prime}{ \left( \pi_\tau^*(a^\prime | s^\prime) - \pi(a^\prime | s^\prime) \right) \cdot \tilde{V}^{\pi_\tau^*}(s^\prime) } + \tau \cdot \KL(\pi(\cdot| s^\prime) \| {\pi_\tau^*}(\cdot| s^\prime) ) \right]  } \qquad \left( \text{by \cref{eq:path_consistency_conditions_appendix_1}} \right) \\
    &= \frac{1}{1 - \gamma} \sum_{s^\prime}{ d_{s}^{\pi}(s^\prime) \cdot \left[ \left( 1 - 1 \right) \cdot \tilde{V}^{\pi_\tau^*}(s^\prime) + \tau \cdot \KL(\pi(\cdot| s^\prime) \| {\pi_\tau^*}(\cdot| s^\prime) ) \right]  } \\
    &= \frac{1}{1 - \gamma} \sum_{s^\prime}{ \left[ d_{s}^{\pi}(s^\prime) \cdot \tau \cdot \KL(\pi( \cdot | s^\prime) \| \pi_\tau^*(\cdot | s^\prime) ) \right] }. \qedhere
\end{align}
\end{proof}

\begin{lemma}[KL-Logit inequality]
\label{lem:kl_logit_inequality}
Let $\pi_\theta = \softmax(\theta)$ and $\pi_{\theta^\prime} = \softmax(\theta^\prime)$. Then for any constant $c \in \sR$,
\begin{align}
    \KL(\pi_\theta \| \pi_{\theta^\prime} ) \le \frac{1}{2} \cdot \left\| \theta^\prime - \theta - c \cdot \rvone \right\|_\infty^2.
\end{align}
In particular, let $c = \frac{(\theta^\prime - \theta )^\top \rvone}{K}$, we have
\begin{align}
    \KL(\pi_\theta \| \pi_{\theta^\prime} ) \le \frac{1}{2} \cdot \left\| \theta^\prime - \theta - \frac{(\theta^\prime - \theta )^\top \rvone}{K} \cdot \rvone \right\|_\infty^2.
\end{align}
\end{lemma}
\begin{proof}
According to the $\ell_1$ norm strong convexity of negative entropy over probabilistic simplex, i.e., for any policies $ \pi$, $\pi^\prime$,
\begin{align}
    {\pi^\prime}^\top \log{ {\pi^\prime} } \ge \pi^\top \log{\pi} + ( {\pi^\prime} - \pi )^\top \log{ \pi } + \frac{1}{2} \cdot \left\| \pi - {\pi^\prime} \right\|_1^2,
\end{align}
we have (letting $\pi = \pi_\theta$, and $\pi^\prime = \pi_{\theta^\prime}$),
\begin{align}
\MoveEqLeft
    \KL(\pi_\theta \| \pi_{\theta^\prime}) = \pi_\theta^\top \log{\pi_\theta} - { \pi_{\theta^\prime} }^\top \log{\pi_{\theta^\prime}} - (\pi_\theta - \pi_{\theta^\prime} )^\top \log{\pi_{\theta^\prime}} \\
    &\le (\pi_\theta - \pi_{\theta^\prime})^\top \log{ \pi_\theta }  - \frac{1}{2} \cdot \left\| \pi_\theta - \pi_{\theta^\prime} \right\|_1^2 - (\pi_\theta - \pi_{\theta^\prime} )^\top \log{\pi_{\theta^\prime}} \\
    &= (\pi_\theta - \pi_{\theta^\prime})^\top \left( \log{ \pi_\theta } - \log{ \pi_{\theta^\prime} } \right) - \frac{1}{2} \cdot \left\| \pi_\theta - \pi_{\theta^\prime} \right\|_1^2 \\
    &= (\pi_\theta - \pi_{\theta^\prime})^\top \left[ \theta - \theta^\prime - \left( \log \sum_{a} \exp\{ \theta(a) \} - \log \sum_{a} \exp\{ \theta^\prime(a) \} \right) \cdot \rvone \right] - \frac{1}{2} \cdot \left\| \pi_\theta - \pi_{\theta^\prime} \right\|_1^2 \\
    &= (\pi_\theta - \pi_{\theta^\prime})^\top \left( \theta - \theta^\prime \right) - \frac{1}{2} \cdot \left\| \pi_\theta - \pi_{\theta^\prime} \right\|_1^2 \\
    &= (\pi_\theta - \pi_{\theta^\prime})^\top \left( \theta - \theta^\prime - c \cdot \rvone \right) - \frac{1}{2} \cdot \left\| \pi_\theta - \pi_{\theta^\prime} \right\|_1^2 
    \qquad \left( \text{since } (\pi_\theta - \pi_{\theta^\prime})^\top c \cdot \rvone = 0 \text{ holds } \forall c \in \sR \right) \\
    &\le \left\| \theta - \theta^\prime - c \cdot \rvone \right\|_\infty \cdot \left\| \pi_\theta - \pi_{\theta^\prime} \right\|_1 - \frac{1}{2} \cdot \left\| \pi_\theta - \pi_{\theta^\prime} \right\|_1^2 
    \qquad \left( \text{by H{\" o}lder's inequality} \right) \\
    &\le \frac{1}{2} \cdot \left\| \theta - \theta^\prime - c \cdot \rvone \right\|_\infty^2,
\end{align}
where the last inequality is according to $ax - bx^2 \le \frac{a^2}{4b}$, $\forall a, b > 0$.
\end{proof}

\begin{lemma}[Reversed \L{}ojasiewicz]
\label{lem:reverse_lojasiewicz_softmax_general}
Denote $\Delta^*(s) = Q^*(s, a^*(s)) - \max_{a \not= a^*(s)}{ Q^*(s, a) } > 0$ as the optimal value gap of state $s$, where $a^*(s)$ is the action that the optimal policy selects under state $s$, and $\Delta^* = \min_{s \in \gS}{ \Delta^*(s) } > 0$ as the optimal value gap of the MDP. Then we have, 
\begin{align}
    \left\| \frac{\partial V^{\pi_\theta}(\mu)}{\partial \theta }\right\|_2 \le \frac{1}{1 - \gamma} \cdot \frac{\sqrt{2}}{\Delta^*} \cdot \left[ V^*(\mu) - V^{\pi_\theta}(\mu) \right].
\end{align}
\end{lemma}
\begin{proof}
Denote $\Delta^*(s,a) = Q^*(s,a^*(s)) - Q^*(s,a)$, and $\Delta^*(s) = \min_{a \not= a^*(s)}{ \Delta^*(s,a) }$. We have,
\begin{align}
    V^*(\mu) - V^{\pi_\theta}(\mu) &= \frac{1}{1 - \gamma} \sum_{s}{ d_\mu^{\pi_\theta}(s) \sum_{a}{ \left( \pi^*(a | s) - \pi_\theta(a | s) \right) \cdot Q^*(s,a) } } \qquad \left( \text{by \cref{lem:value_suboptimality}} \right) \\
    &= \frac{1}{1 - \gamma} \sum_{s}{ d_\mu^{\pi_\theta}(s) \cdot \left[ \sum_{a}{ \pi_\theta(a | s) \cdot Q^*(s, a^*(s)) } - \sum_{a}{ \pi_\theta(a | s) \cdot Q^*(s, a) } \right] } \\
    &= \frac{1}{1 - \gamma} \sum_{s}{ d_\mu^{\pi_\theta}(s) \cdot \left[ \sum_{a \not= a^*(s)}{ \pi_\theta(a | s) \cdot Q^*(s, a^*(s)) } - \sum_{a \not= a^*(s)}{ \pi_\theta(a | s) \cdot Q^*(s, a) } \right] } \\
    &= \frac{1}{1 - \gamma} \sum_{s}{ d_\mu^{\pi_\theta}(s) \cdot \left[ \sum_{a \not= a^*(s)}{ \pi_\theta(a | s) \cdot \Delta^*(s, a) } \right] } \\
    &\ge \frac{1}{1 - \gamma} \sum_{s}{ d_\mu^{\pi_\theta}(s) \cdot \left[ \sum_{a \not= a^*(s)}{ \pi_\theta(a | s)  } \right] \cdot \Delta^*(s) }.
\end{align}
Since $Q^{\pi_\theta}(s,a) \in [ 0, 1/(1 - \gamma) ]$, and $V^{\pi_\theta}(s) \in [ 0, 1/(1 - \gamma) ]$, we have $ | A^{\pi_\theta}(s,a) | \in [ 0, 1/(1 - \gamma) ]$. Also,
\begin{align}
    \left| A^{\pi_\theta}(s,a^*(s)) \right| &= \left| Q^{\pi_\theta}(s,a^*(s)) - \sum_{a}{ \pi_\theta(a | s) \cdot Q^{\pi_\theta}(s,a) } \right| \\
    &= \left| \sum_{a \not= a^*(s)}{ \pi_\theta(a | s) \cdot \left[ Q^{\pi_\theta}(s,a^*(s)) - Q^{\pi_\theta}(s,a) \right] } \right| \\
    &\le \sum_{a \not= a^*(s)}{ \pi_\theta(a | s) \cdot \left| Q^{\pi_\theta}(s,a^*(s)) - Q^{\pi_\theta}(s,a) \right| } \qquad \left( \text{by the triangle inequality} \right) \\
    &\le \frac{1}{1 - \gamma} \sum_{a \not= a^*(s)}{ \pi_\theta(a | s) }. 
    \qquad \left( \text{because } Q^{\pi_\theta}(s,a) \in [ 0, 1/(1 - \gamma) ] \right)
\end{align}
Therefore the $\ell_2$ norm of gradient can be upper bounded as
\begin{align}
    \left\| \frac{\partial V^{\pi_\theta}(\mu)}{\partial \theta }\right\|_2 &= \frac{1}{1-\gamma} \cdot \left[ \sum_{s}{ d_{\mu}^{\pi_\theta}(s)^2 \sum_{a} \pi_\theta(a|s)^2 \cdot A^{\pi_\theta}(s,a)^2 } \right]^{\frac{1}{2}} \\
    &= \frac{1}{1-\gamma} \cdot \left[ \sum_{s}{ d_{\mu}^{\pi_\theta}(s)^2 \cdot \left( \pi_\theta(a^*(s)|s)^2 \cdot A^{\pi_\theta}(s,a^*(s))^2 + \sum_{a \not= a^*(s)}{ \pi_\theta(a|s)^2 \cdot A^{\pi_\theta}(s,a)^2 } \right)  } \right]^{\frac{1}{2}} \\
    &\le \frac{1}{1-\gamma} \cdot \left[ \sum_{s}{ d_{\mu}^{\pi_\theta}(s)^2 \cdot \left( 1 \cdot \frac{1}{(1-\gamma)^2} \cdot \left[ \sum_{a \not= a^*(s)}{ \pi_\theta(a | s) } \right]^2 + \sum_{a \not= a^*(s)}{ \pi_\theta(a|s)^2 \cdot \frac{1}{(1-\gamma)^2} } \right)  } \right]^{\frac{1}{2}} \\
    &\le \frac{1}{(1-\gamma)^2} \cdot \left[ \sum_{s}{ d_{\mu}^{\pi_\theta}(s)^2 \cdot 2 \cdot \left[ \sum_{a \not= a^*(s)}{ \pi_\theta(a | s) } \right]^2 } \right]^{\frac{1}{2}} 
    \qquad \left( \text{by } \| x \|_2 \le \| x \|_1 \right) \\
    &\le \frac{1}{(1-\gamma)^2} \cdot \sqrt{2} \cdot \sum_{s}{ d_{\mu}^{\pi_\theta}(s) \cdot \left[ \sum_{a \not= a^*(s)}{ \pi_\theta(a | s) } \right] }. 
    \qquad \left( \text{by } \| x \|_2 \le \| x \|_1 \right)
\end{align}
Combining the results, we have
\begin{align}
\MoveEqLeft
    \left\| \frac{\partial V^{\pi_\theta}(\mu)}{\partial \theta }\right\|_2 \le \frac{1}{1 - \gamma} \cdot \sqrt{2} \cdot \frac{1}{1 - \gamma} \sum_{s}{ d_\mu^{\pi_\theta}(s) \cdot \left[ \sum_{a \not= a^*(s)}{ \pi_\theta(a | s)  } \right] } \\
    &=\frac{1}{1 - \gamma} \cdot \frac{\sqrt{2}}{\Delta^*} \cdot \frac{1}{1 - \gamma} \sum_{s}{ d_\mu^{\pi_\theta}(s) \cdot \left[ \sum_{a \not= a^*(s)}{ \pi_\theta(a | s)  } \right] \cdot \Delta^* } \\
    &\le \frac{1}{1 - \gamma} \cdot \frac{\sqrt{2}}{\Delta^*} \cdot \frac{1}{1 - \gamma} \sum_{s}{ d_\mu^{\pi_\theta}(s) \cdot \left[ \sum_{a \not= a^*(s)}{ \pi_\theta(a | s)  } \right] \cdot \Delta^*(s) } 
    \qquad \left( \text{by } \Delta^* \le \Delta^*(s) \text{ holds for all } s \right) \\
    &\le \frac{1}{1 - \gamma} \cdot \frac{\sqrt{2}}{\Delta^*} \cdot \left[ V^*(\mu) - V^{\pi_\theta}(\mu) \right]. \qedhere
\end{align}
\end{proof}

\section{Sub-optimality Guarantees for Other Entropy-Based RL Methods}
\label{sec:remark_sub_optimality_guarantees}

Some interesting insight worth mentioning in the proof of \cref{lem:lojasiewicz_entropy_general} is that the intermediate results provide sub-optimality guarantees for existing entropy regularized RL methods. In particular, \cref{eq:intermediate_justification_sac_1,eq:intermediate_justification_sac_2} provides policy improvement guarantee for Soft Actor-Critic \citep[SAC]{haarnoja2018soft}, and  \cref{eq:intermediate_justification_pcl_1,eq:intermediate_justification_pcl_2} provide sub-optimality guarantees for Patch Consistency Learning \citep[PCL]{nachum2017bridging}.
\begin{remark}[Soft policy improvement inequality]
In \citet[Eq. (4) and Lemma 2]{haarnoja2018soft}, the policy is updated by
\begin{align}
    \pi_{\theta_{t+1}} = \argmin_{\pi_\theta}{ \KL\left( \pi_\theta(\cdot | s) \middle\| \frac{\exp\left\{ Q^{\pi_{\theta_t}}(s, \cdot) \right\}}{ \sum_{a}{ \exp\left\{ Q^{\pi_{\theta_t}}(s, a) \right\} } } \right) },
\end{align}
which is exactly the KL divergence in \cref{eq:intermediate_justification_sac_2}, with $\bar{\pi}_{\theta}(\cdot | s)$ defined in \cref{eq:intermediate_justification_sac_1}. The soft policy improvement inequality of \cref{eq:intermediate_justification_sac_2} guarantees that if the soft policy improvement is small, then the sub-optimality  is small.
\end{remark}
\begin{remark}[Path inconsistency inequality]
In \citet[Theorems 1 and 3]{nachum2017bridging}, it is shown that
\begin{itemize}
    \item (i) soft optimal policy $\pi_\tau^*$ satisfies the consistency conditions \cref{eq:path_consistency_conditions_1,eq:path_consistency_conditions_2};
    \item (ii) for any policy $\pi$ that satisfies the consistency conditions, i.e., if $\ \forall s, a$,
    \begin{align}
        \pi(a | s) = \exp\left\{ ( \tilde{Q}^{\pi}(s, a) - \tilde{V}^{\pi}(s) ) / \tau \right\}, \quad \text{and } \quad  \tilde{V}^{\pi}(s) = \tau \log{ \sum_{a} \exp\left\{ \tilde{Q}^{\pi}(s, a) / \tau \right\}},
    \end{align}
then $\pi = \pi_\tau^*$, and $\tilde{V}^{\pi} = \tilde{V}^{\pi_\tau^*}$.
\end{itemize}
However, \citet{nachum2017bridging} does not show if the consistency is violated during learning, how the violation is related to the sub-optimality. To see why \cref{lem:lojasiewicz_entropy_general} provides insight, define the following ``path inconsistency",
\begin{align*}
   r(s,a) + \gamma \sum_{s^\prime}{ \gP(s^\prime | s, a) \tilde{V}^{\pi}(s^\prime) }  - \tilde{V}^{\pi}(s) - \tau \log{\pi(a | s)} =  \tilde{Q}^{\pi}(s, a) - \tilde{V}^{\pi}(s) - \tau \log{\pi(a | s)},
\label{eq:path_inconsistency} \numberthis
\end{align*}
which captures the violation of consistency conditions during learning. Note that for softmax policy $\pi_\theta(\cdot | s) = \softmax(\theta(s, \cdot))$, the r.h.s. of \cref{eq:path_inconsistency} can be written in vector form as
\begin{align}
    \tilde{Q}^{\pi_\theta}(s, \cdot) - \tilde{V}^{\pi_\theta}(s) \cdot \rvone - \tau \log{\pi_\theta(\cdot | s)} &= \tilde{Q}^{\pi_\theta}(s, \cdot) - \tilde{V}^{\pi_\theta}(s) \cdot \rvone - \tau \theta(s, \cdot) + \tau \log \sum_{a} \exp\{ \theta(s,a) \} \cdot \rvone.
\end{align}
Denote $c_\theta(s) = \frac{\tilde{V}^{\pi_\theta}(s)}{\tau} - \log \sum_{a} \exp\{ \theta(s,a) \}$, and using \cref{lem:kl_logit_inequality} in the proof of \cref{lem:lojasiewicz_entropy_general}, in particular, \cref{eq:intermediate_justification_pcl_1},
\begin{align*}
    \KL( \pi_\theta(\cdot | s) \| \bar{\pi}_{\theta}(\cdot | s) ) \le \frac{1}{2} \cdot \left\| \frac{ \tilde{Q}^{\pi_\theta}(s,\cdot) }{\tau} - \theta(s, \cdot) - c_\theta(s) \cdot \rvone \right\|_\infty^2 = \frac{1}{2 \tau^2} \cdot \left\| \tilde{Q}^{\pi_\theta}(s, \cdot) - \tilde{V}^{\pi_\theta}(s) \cdot \rvone - \tau \log{\pi_\theta(\cdot | s)}  \right\|_\infty^2.
\end{align*}
Using the above results in \cref{eq:intermediate_justification_pcl_2},
\begin{align}
\label{eq:intermediate_justification_pcl_3}
\MoveEqLeft
    \left[ \tilde{V}^{\pi_\tau^*}(\rho) - \tilde{V}^{{\pi_\theta}}(\rho) \right]^\frac{1}{2} 
    \le \frac{1}{\sqrt{1 - \gamma}} \cdot \frac{1}{\sqrt{2 \tau }} \cdot \left\| \frac{d_{\rho}^{\pi_\tau^*} }{ d_{\mu}^{\pi_\theta}} \right\|_\infty^\frac{1}{2} \sum_{s}{ \sqrt{ d_{\mu}^{\pi_\theta}(s) } \cdot \left\| \tilde{Q}^{\pi_\theta}(s, \cdot) - \tilde{V}^{\pi_\theta}(s) \cdot \rvone - \tau \log{\pi_\theta(\cdot | s)} \right\|_\infty } \\
    &= \frac{1}{\sqrt{1 - \gamma}} \cdot \frac{1}{\sqrt{2 \tau }} \cdot \left\| \frac{d_{\rho}^{\pi_\tau^*} }{ d_{\mu}^{\pi_\theta}} \right\|_\infty^\frac{1}{2} \sum_{s}{ \sqrt{ d_{\mu}^{\pi_\theta}(s) } \cdot \max_{a} \left| r(s,a) + \gamma \sum_{s^\prime}{ \gP(s^\prime | s, a) \tilde{V}^{\pi_\theta}(s^\prime)  } - \tau \log{\pi_\theta(a | s)} - \tilde{V}^{\pi_\theta}(s)  \right| },
\end{align}
where (square of) $\left| r(s,a) + \gamma \sum_{s^\prime}{ \gP(s^\prime | s, a) \tilde{V}^{\pi_\theta}(s^\prime)  } - \tau \log{\pi_\theta(a | s)} - \tilde{V}^{\pi_\theta}(s) \right|$ is exactly the (one-step) path inconsistency objective used in PCL \citep[Eq. (14)]{nachum2017bridging}. Therefore, minimizing path inconsistency guarantees small sub-optimality. The path inconsistency inequality of \cref{eq:intermediate_justification_pcl_3} implies path consistency of \citet{nachum2017bridging}.
\end{remark}

\section{Simulation Results}
\label{sec:simulation_results}

To verify the convergence rates in the main paper, we conducted experiments on one-state MDPs, which have $K$ actions, with randomly generated reward $r \in [0,1]^K$, and randomly initialized policy $\pi_{\theta_1}$.

\subsection{Softmax Policy Gradient}

$K = 20$, $r \in [0,1]^K$ is randomly generated, and $\pi_{\theta_1}$ is randomly initialized. Softmax policy gradient, i.e., \cref{update_rule:softmax_special} is used with learning rate $\eta = 2/5$ and $T = 3 \times 10^5$. As shown in \cref{fig:fig2}(a), the sub-optimality $\delta_t = \left( \pi^* - \pi_{\theta_t} \right)^\top r$ approaches $0$. Subfigures (b) and (c) show $\log{\delta_t}$ as a function of $\log{t}$. As $\log{t}$ increases, the slope is approaching $-1$, indicating that $\log{\delta_t} = - \log{t} + C$, which is equivalent to $\delta_t = {C^\prime}/t$. Subfigure (d) shows $\pi_{\theta_t}(a^*)$ as a function of $t$.

\begin{figure*}[ht]
\centering
\includegraphics[width=1\linewidth]{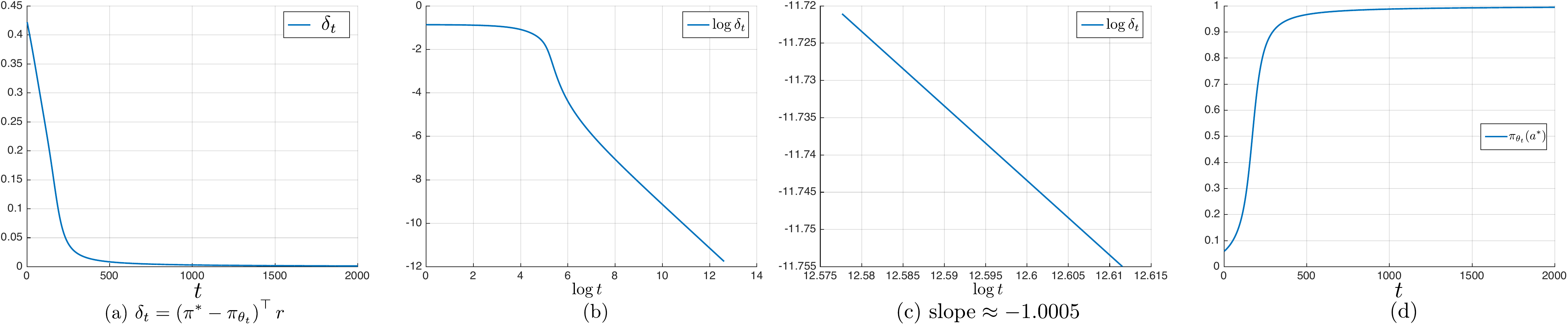}
\vskip -0.1in
\caption{Softmax policy gradient, \cref{update_rule:softmax_special}.}
\label{fig:fig2}
\vskip -0.15in
\end{figure*}

\subsection{Entropy Regularized Softmax Policy Gradient}

$K = 20$, $r \in [0,1]^K$ and $\pi_{\theta_1}$ are the same as above. Entropy regularized softmax policy gradient, i.e., \cref{update_rule:entropy_special} is used with temperature $\tau = 0.2$, learning rate $\eta = 2/5$ and $T = 5 \times 10^4$. As shown in \cref{fig:fig3}(a), the soft sub-optimality $\tilde{\delta}_t = { \pi_\tau^* }^\top \left( r - \tau \log{ \pi_\tau^* } \right)  - { \pi_{\theta_t} }^\top \left( r - \tau \log{ \pi_{\theta_t} } \right)$ approaches $0$. Subfigure (b) shows $\log{\tilde{\delta}_t}$ as a function of $t$. As $t$ increases, the curve approaches a straight line, indicating that $\log{\tilde{\delta}_t} = - C_1 \cdot t + C_2$, which is equivalent to $\tilde{\delta}_t = {C_2^\prime}/\exp\{ C_1^\prime \cdot t\}$. Subfigure (c) shows $\zeta_t$ as defined in \cref{lem:contraction_entropy_special} as a function of $t$, which verifies \cref{lem:matching_entropy_special}. Subfigure (d) shows $\min_{a}{\pi_{\theta_t}(a)}$ as a function of $t$. As $t$ increases, $\min_{a}{\pi_{\theta_t}(a)}$ approaches constant values, which verifies \cref{lem:lower_bound_min_prob_entropy_special}.

\begin{figure*}[ht]
\centering
\includegraphics[width=1\linewidth]{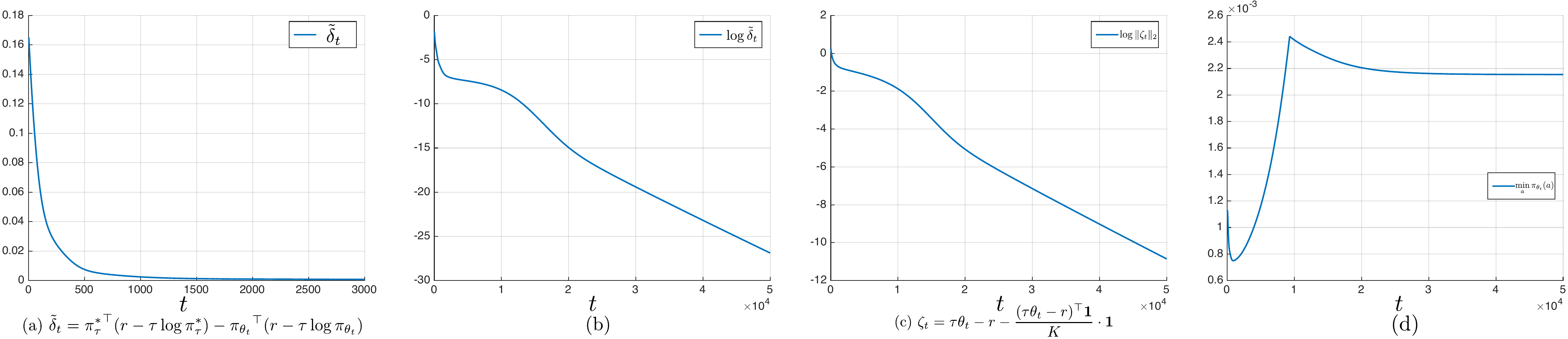}
\vskip -0.1in
\caption{Entropy regularized softmax policy gradient, \cref{update_rule:entropy_special}.}
\label{fig:fig3}
\vskip -0.15in
\end{figure*}

\subsection{``Bad" Initializations for Softmax Policy Gradient (PG)}

As illustrated in \cref{fig:fig1}, ``bad" initializations lead to attraction toward sub-optimal corners and slowly escaping for softmax policy gradient. \cref{fig:fig4} shows one example with $K=5$. Softmax policy gradient takes about $ 8 \times 10^6$ iterations around a  sub-optimal corner. While with entropy regularization ($\tau = 0.2$), the convergence is significantly faster.

\begin{figure*}[ht]
\centering
\includegraphics[width=1\linewidth]{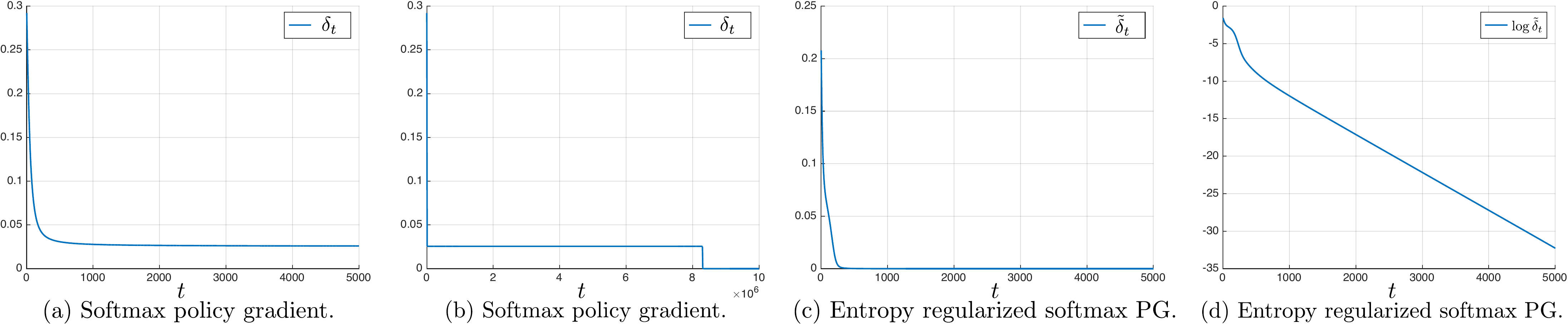}
\vskip -0.1in
\caption{Bad initialization for softmax policy gradient.}
\label{fig:fig4}
\vskip -0.15in
\end{figure*}

\subsection{Decaying Entropy Regularization}

We run entropy regularized policy gradient with decaying temperature $\tau_t = \frac{\alpha \cdot \Delta}{ \log{t} }$ for $t \ge 2$, i.e., \cref{update_rule:decaying_entropy_special}. \cref{fig:fig5} shows one example with $K=10$ and different $\alpha$ values. The actual rate is $O( \frac{1}{t^{-\text{slope}}} )$, and the partial rate in \cref{thm:rates_decaying_entropy_special} is $O(\frac{1}{t^{1/\alpha}})$.

\begin{figure*}[ht]
\centering
\includegraphics[width=1\linewidth]{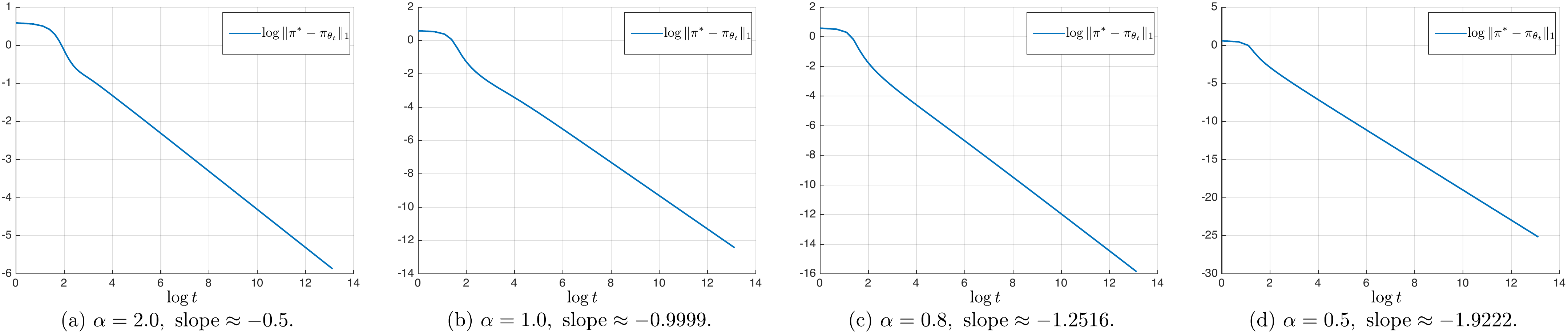}
\vskip -0.1in
\caption{Decaying entropy regularization, \cref{update_rule:decaying_entropy_special}.}
\label{fig:fig5}
\vskip -0.2in
\end{figure*}

\end{document}